\documentclass[mnsc,nonblindrev]{informs3}

\RRHFirstLine{}
\RRHSecondLine{}
\LRHFirstLine{}
\LRHSecondLine{}

\OneAndAHalfSpacedXI %

\usepackage{endnotes}

\usepackage{enumitem}
\usepackage{amsmath}
\usepackage{bm} 
\usepackage{framed}
\usepackage[outdir=./]{epstopdf}
\usepackage{makecell}

\usepackage{graphicx}
\usepackage{subfig}
\usepackage{xcolor,cancel}
\usepackage{tabularx}
\usepackage{multirow}
\usepackage{threeparttable}
\usepackage{booktabs}  %

\usepackage{algorithm}
\usepackage{algorithmicx}
\usepackage{algpseudocode}

\usepackage{url}
\usepackage{tabularx}

\newcolumntype{C}[1]{>{\centering\arraybackslash}p{#1}}

\newcommand{\E}{\mathbb{E}}
\newcommand{\Var}{\mathrm{Var}}
\newcommand{\Cov}{\mathrm{Cov}}

\usepackage{natbib}
 \bibpunct[, ]{(}{)}{,}{a}{}{,}%
 \def\bibfont{\small}%

\TheoremsNumberedThrough     %
\ECRepeatTheorems

\EquationsNumberedThrough    %

\begin{document}

\RUNAUTHOR{Bekci, Jasin, and Uichanco}

\RUNTITLE{Borrowing Strength}

\TITLE{Borrowing Strength Across Stores: Transfer Learning for Choice-Based Revenue Management\thanks{Authors are listed in alphabetical order.}}

\ARTICLEAUTHORS{%
\AUTHOR{Recep Bekci}
\AFF{Department of Management Science and Engineering, University of Waterloo, Waterloo, Ontario, Canada\\
\EMAIL{recep.bekci@uwaterloo.ca}}
\AUTHOR{Stefanus Jasin}
\AFF{Stephen M. Ross School of Business, University of Michigan, Ann Arbor, Michigan, USA\\
\EMAIL{sjasin@umich.edu}}
\AUTHOR{Joline Uichanco}
\AFF{Department of Technology Management and Innovation, Tandon School of Engineering,\\
New York University, Brooklyn, New York, USA\\
\EMAIL{ju441@nyu.edu}}
} %

\ABSTRACT{Large retail chains must tailor assortment and pricing decisions to local customer preferences, yet even retailers with extensive sales records often have limited data for estimating demand at the store or local market level, particularly in product categories with short selling windows. Local estimation\textemdash fitting a demand model for each store on its own\textemdash can therefore be unreliable, while pooling data across stores improves estimation precision at the cost of ignoring local demand heterogeneity. We study how retailers can balance these competing effects through transfer learning under a multinomial logit (MNL) demand model. We propose a shrinkage estimator that combines store-specific demand estimates with information pooled from other stores. Unlike the linear or quadratic settings underlying much of the existing theoretical literature on information transfer in operations, the nonlinear MNL likelihood does not admit a straightforward bias–variance decomposition. By exploiting its self-concordant-like properties, we establish finite-sample performance guarantees that characterize when information transfer improves demand estimation and how the appropriate degree of shrinkage depends on local sample size, feature dimensionality, and cross-store heterogeneity. We further derive regret bounds for downstream assortment and pricing decisions. Numerical experiments demonstrate the operational benefits of shrinkage across different data environments. An empirical study using data from several hundred stores of a European apparel retailer shows that shrinkage improves out-of-sample product prioritization and revenue capture relative to both local estimation and full pooling. Our findings demonstrate how retailers can leverage information across their store networks to improve revenue management decisions while preserving local demand differences.
}%

\maketitle

\section{Introduction}
\label{sec:introduction}

Retailers often make revenue management decisions across many local markets where customers differ substantially in how they substitute across products. Since local covariates such as demographics, climate, and competitive intensity shape these substitution patterns, a retail chain carrying the same catalog across hundreds of stores may adapt to local tastes by varying its assortments and prices or by featuring different products within each store's stocked assortment. Yet the data available at any single store are often sparsest 
precisely where accurate estimation matters most, as in style-sensitive categories with short selling windows or new product launches. Retailers therefore face a tradeoff between store-level estimation, which captures local preferences but can be noisy when data are scarce, and pooling data across stores, which improves precision but may mask differences that matter for local decisions. 

This tradeoff is especially sharp when the retailer estimates these substitution patterns through a structural choice model, such as the multinomial logit (MNL) model, a workhorse in assortment and pricing optimization. In this model, demand for a product depends on its own features (e.g., price, brand, style) and on the features of the other products in the offered set, so a store's choice model is governed by a whole vector of coefficients that must be estimated jointly. With so many coefficients to recover, a store-level maximum likelihood estimator (MLE) can exhibit high variance when local data are scarce. Pooling the data, by estimating a single coefficient vector from the combined sales of all stores, reduces this variance by borrowing strength across the chain, but a fully pooled estimator is biased whenever stores differ in their underlying preferences. Both sources of estimation error can lead to assortment and pricing decisions that are poorly suited to the local market. This creates a natural opportunity for transfer learning, in which information from other stores is used to improve store-specific demand estimation without assuming that all stores have identical preferences.

This paper studies that tension by asking:
\emph{when local data are limited, how much data should a retailer borrow from other stores when estimating a store-specific choice model?} We formulate a multi-store MNL setting in which each store has its own coefficient vector, but all such vectors lie within a distance $\Delta$ of a common center. This captures the feature that local markets share a common demand structure but are not identical, with $\Delta$ measuring how much they differ. We study three estimators for each store: a \emph{local estimator} that uses only store-specific data, a \emph{pooled estimator} that fits a single coefficient vector to the combined data of all stores, and a \emph{shrinkage estimator} that combines the local estimate with a pooled estimate constructed from the other stores. The shrinkage weight determines how much information is transferred from the rest of the chain to each individual store. Our finite-sample analysis characterizes how the value of such information transfer depends on local sample size, feature dimensionality, and cross-store heterogeneity.

Related work on data pooling and transfer learning shares our intuition: when each estimation unit has little data, borrowing across units can reduce variance. How tractable that analysis becomes depends on the object being compared. For a mean-squared-error criterion, the effect of pooling admits an explicit quadratic risk decomposition \citep{li2023mostly}; for prescriptive problems with linear objectives, it can be handled through a tractable bias correction \citep{gupta2021small}; and in stochastic optimization more broadly, the benefit of pooling is known to be structure-dependent rather than universal \citep{gupta2022data}. Hierarchical choice models have also long exploited information sharing across markets, but our focus is on establishing finite-sample performance guarantees for such transfer under structural choice demand. Our setting falls outside this tractable regime in two respects. The statistical object is the store-level expected log-likelihood loss of the MNL model, which is nonlinear, with curvature that varies across the parameter space, so the risk comparison admits no exact bias-variance identity. And the decision objects, assortment and pricing regret, depend on the estimated coefficient vector through a nonlinear optimization rather than a fixed quadratic or linear map. Classical quadratic shrinkage arguments therefore do not extend directly, and establishing finite-sample guarantees for transfer learning in this choice-model setting requires new analytical machinery. 

We develop this machinery by adapting self-concordant-like (SCL) analysis to the MNL setting. We specialize the SCL property of the MNL negative log-likelihood to our shared-feature multi-store formulation, showing how local curvature controls the third derivative. This property gives a structured way to handle the likelihood's nonlinearity. It lets us track how the curvature of the likelihood changes across nearby parameter values, derive local quadratic approximations with explicit remainder bounds, and obtain nonasymptotic concentration results for store-level estimators without relying on asymptotic normality. %

The analysis finds that shrinkage towards the pooled estimate reduces the variance of a store's estimate but adds bias when the store differs from the rest of the chain. We derive a shrinkage weight based on a local quadratic approximation and introduce a damped version that controls the error arising from the nonlinearity of the MNL likelihood. Our finite-sample analysis identifies conditions under which the resulting estimator improves on local estimation, and shows that the degree of shrinkage decreases as local sample size or cross-store heterogeneity increases. Hence, our analysis yields an interpretable finite-sample characterization of when and how much information retailers should transfer across stores when estimating structural choice models.

The paper ties this statistical analysis directly to revenue management decisions. What ultimately matters is how parameter error propagates through the downstream optimization to the decision. In assortment, a small change in an estimated coefficient can change the assortment and cause a discrete jump in revenue. In pricing, the effect is smoother but still costly, as a biased estimate of price sensitivity can leave a store systematically overpriced or underpriced. We derive regret bounds that characterize how local estimation error, pooled estimation error, and cross-store heterogeneity affect the performance of assortment and pricing decisions. A simulation study confirms these insights. Across levels of data scarcity and cross-store heterogeneity, shrinkage generally stays close to whichever of local estimation and full pooling attains lower regret, and can outperform both in intermediate regimes.

We then ask whether this advantage holds on observational data in a case study of a large multi-category apparel retailer. We rank products by each estimator's predicted revenue contribution and compare the share of realized revenue captured by the top-ranked products on held-out data. Adaptive shrinkage outperforms both local and pooled estimation across product categories and store-weeks with different amounts of local training data. Averaged across category-store-weeks, adaptive shrinkage improves revenue capture by 2.6 percentage points relative to the local estimator and by 3.1 percentage points relative to the pooled estimator. It outperforms each benchmark in more than 80\% of category-store-weeks. The improvement over the pooled estimator is statistically reliable in every
category, and the improvement over the local estimator, while smaller on
average, is reliable overall and reaches as much as 11 percentage points in
some category and volume combinations. These results demonstrate the practical value of transferring information across stores for product prioritization while preserving meaningful differences in local demand.

This paper makes four main contributions. First, we formulate a transfer-learning problem for multi-store MNL estimation under data scarcity and bounded cross-store heterogeneity, and introduce a family of shrinkage estimators indexed by a weight that controls how much each store borrows from the chain, with local estimation and full pooling as its two extremes.
Second, we develop the analytical apparatus required for this nonlinear setting. By adapting self-concordant-like analysis to the MNL log-likelihood, we obtain finite-sample control of the likelihood geometry that does not rely on asymptotic normality. This matters because local samples are small relative to the parameter dimension, the regime in which asymptotic approximations can misrepresent both the estimation risk and the resulting decisions. Third, using this apparatus, we develop a shrinkage weight with finite-sample risk guarantees that characterize when cross-store information transfer is beneficial as a function of sample size, dimension, and heterogeneity.
Fourth, we extend the analysis from estimation to decisions, establishing operational regret bounds for the assortment and pricing problems that quantify how estimation error propagates into revenue. Simulation and an empirical study with retailer data then confirm that adaptive shrinkage improves these downstream decisions in practice.

The rest of the paper is organized as follows. Section~\ref{sec:literature} reviews the related literature on data scarcity, pooling, and structural demand estimation. Section~\ref{sec:multistore_mnl} introduces the multi-store MNL model and its two key features, limited local data and bounded cross-store heterogeneity, and Section~\ref{sec:estimators} defines the local, pooled, and shrinkage estimators. Section~\ref{sec:theory} establishes the finite-sample performance guarantees, and Section~\ref{sec:operational} translates them into operational regret bounds for assortment and pricing. Sections~\ref{sec:simulation} and~\ref{sec:empirical} report the simulation and empirical studies, and Section~\ref{sec:conclusion} concludes.

\section{Literature Review}
\label{sec:literature}

This paper sits at the intersection of structural revenue management
and small-data learning across related markets. One relevant literature
studies choice-based assortment and pricing, where structural demand
models are used to support revenue-management decisions. A second
literature studies how to borrow information across related units when
local data are scarce, through hierarchical pooling, aggregation,
clustering, or transfer. A final methodological strand develops
finite-sample tools for nonlinear likelihood models. The discussion
below is necessarily selective rather than exhaustive and focuses on
the papers most closely related to our contribution.

\paragraph{Choice-based assortment and pricing.} A broad stream of work uses structural demand models to support assortment and pricing decisions in retail, including demand estimation and store-level assortment optimization with observed sales and scanner data \citep{fisher2014demand,rooderkerk2013optimizing}. More recent papers study high-dimensional and personalized choice models, dynamic learning of joint assortment and pricing decisions, and large-scale field implementations in online retail \citep{jiang2025high,chen2022statistical,miao2021dynamic,feldman2022customer}. Another branch enriches the underlying choice model through default effects, consideration sets, and sequential or patience-based search \citep{najafi2025pricing,aouad2021assortment,gao2021assortment}. Our paper contributes to this broad stream by focusing on a different problem: how to estimate store-specific MNL models when a retailer operates many related stores with limited local data, and how that estimation problem affects downstream assortment and pricing decisions.

\paragraph{Data scarcity.}
The literature on data scarcity is broad, so we focus on the strands most closely related to our setting. In demand estimation, a long line of work uses hierarchical or partial-pooling models to stabilize local estimates across related markets, stores, or customer segments \citep{ramaswamy1993empirical,montgomery1997creating,arora1998hierarchical,bradlow2000hierarchical}. More recent work shows that hierarchical choice models remain strong practical benchmarks and that conventional out-of-sample fit measures can be poorly aligned with profit-based model selection \citep{smith2023optimal}. Our paper is related to this literature in its use of cross-unit borrowing, but differs in emphasis: we study finite-sample shrinkage in a multi-store MNL model and characterize explicitly how the local-versus-pooled tradeoff depends on local sample size and cross-store heterogeneity.

A second line of work studies small-data decision making through pooling, shrinkage, and clustering. In operations, \cite{gupta2021small} formalize the small-data, large-scale regime, \cite{gupta2022data} study data pooling in stochastic optimization, and \cite{gupta2024debiasing} develop debiasing methods for policy evaluation in this setting. Related retail analytics papers use aggregation or clustering to improve prediction and pricing when unit-level data are sparse \citep{cohen2022data,miao2022context,li2023mostly}. These papers are close in spirit to ours, but they mainly focus on linear-objective stochastic optimization, reduced-form demand prediction, or clustering-based decision rules rather than store-specific structural choice estimation.

A third line uses transfer, meta-, or cross-learning to move information across related products, tasks, or experiments \citep{bastani2021predicting,bastani2022meta,lei2025pooling,feng2023transfer}. Our paper is complementary to this literature. Rather than transferring a predictor or learning a shared prior across tasks, we study adaptive partial pooling for store-specific MNL estimation, make heterogeneity explicit through the radius \(\Delta\), exploit the nonlinear geometry of the MNL likelihood, and evaluate the resulting shrinkage weights through downstream assortment and pricing regret.

\paragraph{Methodological positioning.}
Methodologically, the closest lineage is the self-concordant analysis
of binary logistic regression \citep{bach2010self}, its later
generalization to broader nonlinear loss classes
\citep{sun2019generalized}, and finite-sample analyses of
\(M\)-estimators based on this geometry \citep{ostrovskii2021finite}.
Our paper builds on this line in three directions. First, we specialize
the self-concordant-like geometry of the MNL negative log-likelihood to
our shared-feature multi-store setting. Second, we use this geometry to
derive store-specific local quadratic expansions and nonasymptotic
concentration bounds with explicit remainder control under data
scarcity. Third, we embed these results in a leave-one-out shrinkage
analysis and translate the resulting finite-sample guarantees into
operational regret bounds for assortment and pricing.

\section{Multi-Store MNL Estimation Under Data Scarcity}
\label{sec:multistore_mnl}

We now formalize the multi-store demand estimation problem. We consider a retail network in which a common product category is offered across many stores whose customers may have different preferences. Estimation is challenging due to two reasons. First, each store provides only limited local choice data. Second, customers substitute across products differently from store to store, so pooling data across the network can bias the demand model for any individual store. This section introduces the multi-store MNL model and the heterogeneity structure that underpins the analysis.

\subsection{Store Network and Local Choice Models}

Consider a network of $K\ge2$ stores offering a common catalog of $N$ products. We denote by $\mathcal N$ the set of products. Product $i\in\mathcal N$ is associated with a feature vector $X_i\in\mathbb{R}^d$, and the outside option is indexed by $0$ with $X_0=0$. Each product's feature vector $X_i$ is the same across all stores; what varies from store to store is how customers value those features.

We encode these store-specific valuations in a parameter vector $\beta_k^\star\in\mathbb{R}^d$, so that customer choices at store $k\in\{1,\dots,K\}$ follow a multinomial logit model with parameter $\beta_k^\star$. The probability that a store-$k$ customer chooses product $i$ is
\begin{equation}
P_i(\beta_k^\star)
=
\frac{\exp\!\left(X_i^\top\beta_k^\star\right)}
{1 + \sum_{j \in \mathcal N} \exp\!\left(X_j^\top \beta_k^\star\right)},
\qquad i \in \mathcal N.
\label{eq:sec3_mnl_prob}
\end{equation}
This specifies a complete demand system at store $k$, determined entirely by its parameter $\beta_k^\star$. 

Consider choice data of $n_k$ store-$k$ customers, indexed by $t \in \{1, \ldots, n_k\}$. Let $y_{t,k} \in \mathcal N \cup \{0\}$ denote the product chosen by customer $t$ (with $y_{t,k} = 0$ if customer $t$ selects the outside option), and collect these into vector $Y_k=(y_{1,k},\dots,y_{n_k,k})$. These choices are independently drawn from the store-$k$ model, 
\[
\mathbb{P}(y_{t,k}=i)=P_i(\beta_k^\star), \qquad i\in \mathcal N \cup \{0\}.
\]
Let $c_{ik}:=\sum_{t=1}^{n_k}\mathbf{1}\{y_{t,k}=i\}$ denote the number of times product $i$ is chosen at store $k$. The empirical negative log-likelihood for store $k$ is
\begin{equation}
L_k(\beta;Y_k)
=
-\sum_{i\in \mathcal N} c_{ik}\,\beta^\top X_i
+
n_k \log\!\left(1+\sum_{j \in \mathcal N} \exp(\beta^\top X_j)\right).
\label{eq:sec3_store_loss}
\end{equation}
Its population counterpart is
\begin{equation}
L_k^\star(\beta):=\mathbb{E}\!\left[L_k(\beta;Y_k)\right],
\label{eq:sec3_pop_loss}
\end{equation}
where the expectation is over $Y_k$ drawn from the true store-$k$ model $P(\beta_k^\star)$. We refer to
\(L_k(\beta;Y_k)\) as the store-$k$ \textit{empirical loss} and \(L_k^\star(\beta)\) as the store-$k$ \textit{population loss}.  A superscript $\star$ denotes a population quantity defined at the true parameter.

Under correct specification, \(\beta_k^\star\) minimizes \(L_k^\star(\beta)\). To see this, note that
\[
\frac{1}{n_k} L_k^\star(\beta)
=
- \sum_{i \in \mathcal N \cup \{0\}}
P_i(\beta_k^\star)\,\log P_i(\beta).
\]
The right-hand side is the cross-entropy between the true model \(P(\beta_k^\star)\) and the candidate distribution
\(P(\beta)\). It can be decomposed as
\[
- \sum_{i \in \mathcal N \cup \{0\}}
P_i(\beta_k^\star)\,\log P_i(\beta_k^\star)
\;+\;
D_{\mathrm{KL}}\!\left(
P(\beta_k^\star)\,\|\,P(\beta)
\right).
\]
The first term is constant in \(\beta\), while the KL divergence is
nonnegative and vanishes if and only if the two choice distributions
coincide. Provided the choice model is identifiable (formalized in Assumption~\ref{ass:fisher}), the divergence vanishes only at \(\beta = \beta_k^\star\), so \(\beta_k^\star\) is the unique
minimizer of \(L_k^\star\). For this reason, we evaluate any estimator $\widehat\beta$ through the population loss $L_k^\star(\widehat\beta)$.

Since $\beta_k^\star$ determines how assortment and pricing decisions translate 
into revenue at store $k$, estimating it accurately is central to store-level revenue management. The difficulty is that, in many product categories, the local sample size $n_k$ is too small for local estimation to be reliable.

\subsection{Data Scarcity and Cross-Store Heterogeneity}

The data-scarce regime arises naturally in retail. Demand is fragmented across stores, so even large chains
generate only a modest number of informative observations for any given store--category pair. The effective sample size is smaller still when the objective is to recover substitution patterns rather than aggregate sales, because the analyst must infer relative preferences across products instead of simply tracking overall volume. This scarcity is most acute in categories with short selling windows or frequent assortment turnover, such as seasonal products, fashion items, and new launches. These are precisely the categories where structural demand estimation is most valuable, and where store-level estimation is most fragile.

Data scarcity motivates pooling, but store-level demands are structurally heterogeneous. Stores differ in customer demographics, local competition, climate, and format. These differences shift the underlying taste parameters. Hence, we model stores as related but not identical, assuming a common center \(\beta_c^\star \in \mathbb{R}^d\)
and a radius \(\Delta \ge 0\) such that
\begin{equation}
\|\beta_k^\star - \beta_c^\star\|_2 \le \Delta,
\qquad k = 1, \dots, K,
\label{eq:sec3_heterogeneity}
\end{equation}
as formalized in Assumption \ref{ass:hetero}. The radius \(\Delta\) measures the degree of cross-store
heterogeneity.

This condition imposes neither identical coefficients, discrete clusters, nor a hierarchical prior. Rather, it
formalizes the idea that stores share a common demand structure as part of the same network, while still allowing
for systematic local variation.

\subsection{Assumptions}
\label{subsec:assumptions}

Our analysis relies on a set of regularity conditions on the data-generating model, which we state next. Throughout, let $X^\beta$ denote the random feature vector of the alternative chosen by a customer whose choices follow the MNL model \eqref{eq:sec3_mnl_prob} with parameter $\beta$, so that $\mathbb{P}(X^\beta = X_i) = P_i(\beta)$ for $i \in \mathcal N$. For each store $k$, write
$\Sigma_X^k := \Cov(X^{\beta_k^\star})$.

\begin{assumption}[Bounded features]
\label{ass:bounded}
There exists \(B>0\) such that
$
\max_{i \in \mathcal N} \|X_i\|_2 \le B.
$
\end{assumption}

\begin{assumption}[Compact parameter region]
\label{ass:parameter}
The parameter space \(\Theta\subset\mathbb R^d\) is compact and convex,
with finite diameter
\[
D_\Theta
:=
\sup_{\beta,\beta'\in\Theta}\|\beta-\beta'\|_2.
\]
There exists \(\rho_\Theta>0\) such that
\[
\beta_k^\star\in\operatorname{int}(\Theta),
\qquad
\operatorname{dist}(\beta_k^\star,\partial\Theta)
\ge \rho_\Theta,
\qquad k=1,\ldots,K.
\]
The quantities $D_\Theta$ and $\rho_\Theta$ are fixed independently of
$d$, $K$, and the sample sizes.
\end{assumption}

\begin{assumption}[Nondegenerate local information]
\label{ass:fisher}
There exist constants \(0<m\le M<\infty\), independent of \(n_k\),
such that
$
m I_d \preceq \Sigma_X^k \preceq M I_d
$
for all $k$.
\end{assumption}

\begin{assumption}[Sampling independence]
\label{ass:indep}
Observations are independent across stores and i.i.d.\ within each store.
\end{assumption}

\begin{assumption}[Bounded cross-store heterogeneity]
\label{ass:hetero}
There exist a common center \(\beta_c^\star\in\mathbb{R}^d\) and a
radius \(\Delta\ge0\) such that
$
\|\beta_k^\star-\beta_c^\star\|_2 \le \Delta
$
for all $k$.
\end{assumption}

\begin{assumption}[Balanced store sample sizes]
\label{ass:ratio}
There exists a constant \(C_{\max}\ge 1\) such that
$
\max_{1\le k\le K} n_k
\le
C_{\max}\,\min_{1\le k\le K} n_k.
$
\end{assumption}

\begin{assumption}[Minimum local sample size]
\label{ass:sample}
There exists a constant \(C_0\ge 1\), depending only on
\((B,m,M,C_{\max},D_\Theta,\rho_\Theta)\), such that
$
n_{\min}:=\min_{1\le k\le K} n_k \ge C_0 d.
$
\end{assumption}

These assumptions have simple interpretations in the present setting.
Assumption~\ref{ass:bounded} fixes the scale of the product features.
This is mainly a normalization condition. It rules out products whose
feature vectors are arbitrarily large and ensures that the MNL
likelihood changes smoothly as the parameter changes.

Assumption~\ref{ass:parameter} ensures that every empirical estimator
is well defined, including on samples for which an unrestricted MNL
likelihood would not have a finite minimizer. The fixed interior margin
keeps the population parameters away from the boundary. Convexity also
ensures that every shrinkage combination remains in the same parameter
region. A convenient implementation is the Euclidean ball
$\Theta=\{\beta:\|\beta\|_2\le B_\beta\}$, with
$\|\beta_k^\star\|_2\le B_\beta-\rho_\Theta$.

Assumption~\ref{ass:fisher} requires each store's choice data to contain meaningful variation
in the product-feature space. A direct calculation (see
Corollary~\ref{cor:app_population_hessian} in the Appendix) shows that
$\nabla^2 L_k^\star(\beta_k^\star) = n_k \Sigma_X^k$. Thus \(\Sigma_X^k\) is the
per-observation information matrix of the store-\(k\) choice model. For
any unit vector \(v\), $v^\top \Sigma_X^k v = 
\Var(v^\top X^{\beta_k^\star})$, so the lower bound \(m I_d\preceq \Sigma_X^k\) rules out feature directions along which the alternatives chosen at store \(k\) provide
almost no variation. Without such variation, the likelihood would be
nearly flat in that direction and the corresponding component of
\(\beta_k^\star\) would be difficult to estimate. The upper bound
\(\Sigma_X^k\preceq M I_d\) prevents the opposite degeneracy, where the
curvature scale becomes unbounded. Together, these bounds ensure that
each store-level population loss has stable local curvature, so the
norm used in the finite-sample analysis is well behaved. 

Assumptions~\ref{ass:indep} and~\ref{ass:hetero} formalize a multi-store setting in which samples are independent across stores and store-specific parameters lie within a bounded distance of a common center. Assumption~\ref{ass:ratio} rules out extreme imbalance in store-level
sample sizes. Assumption~\ref{ass:sample} keeps the analysis away from an
ultra-sparse regime in which local estimation is essentially
ill-posed. For the results below, we take $C_0$ sufficiently large. Lemma~\ref{lem:app_constants} in Appendix~\ref{app:barDelta} shows that the required threshold depends only on $(B,m,M,C_{\max},D_\Theta,\rho_\Theta)$.

Unless otherwise stated, numerical constants may depend on the
primitive regularity quantities
\((B,m,M,C_{\max},D_\Theta,\rho_\Theta)\), but not on
\(d\), \(K\), or the sample sizes \(\{n_k\}_{k=1}^K\).  

\section{Local, Pooled, and Shrinkage Estimators}
\label{sec:estimators}

We next introduce three estimators of the store-$k$ parameter $\beta_k^\star$, differing in how much they borrow from the rest of the network. The \emph{local} estimator relies solely on local data $Y_k$. At the opposite extreme, the \emph{pooled} estimator borrows as
aggressively as possible from the entire chain, effectively
imposing a common demand model across stores. The \emph{shrinkage}
estimator lies between these extremes. It anchors on the local estimator and then shrinks it toward a leave-one-out pooled estimator constructed from other stores' data $(Y_\ell)_{\ell\neq k}$. %

\subsection{Local Estimation}

The most direct way to estimate $\beta^\star_k$ is to use only store $k$'s own data $Y_k$. This leads to the store-$k$ local maximum likelihood estimator (MLE)
\begin{equation}
\widehat\beta_k
\;:=\;
\argmin_{\beta\in\Theta} L_k(\beta;Y_k),
\qquad k=1,\dots,K,
\label{eq:sec4_local_mle}
\end{equation}
where \(L_k(\beta;Y_k)\) is the empirical loss defined in \eqref{eq:sec3_store_loss}.

The appeal of the local estimator $\widehat\beta_k$ is that it uses only store $k$'s own data, so it targets the store's own model \(P(\beta_k^\star)\) and preserves local heterogeneity. When store \(k\) has a rich purchase history, it is the natural benchmark. However, in a data-scarce regime, \(\widehat\beta_k\) becomes unstable. Its variance grows as \(n_k\) shrinks, and this instability propagates into downstream decisions. Specifically, small perturbations in the estimated choice model can change which products are selected for the assortment or which prices are recommended, as Section~\ref{sec:simulation} illustrates. The local MLE is thus statistically
well-founded, but operationally fragile when local data are thin.

\subsection{Pooled Estimation}

At the opposite extreme, one may ignore cross-store
heterogeneity and estimate a single choice model for all stores using
data from the entire network $Y = (Y_k)_{k=1}^K$. Define the \textit{pooled empirical loss}
\begin{equation}
L_0(\beta;Y)
\;:=\;
\sum_{k=1}^K L_k(\beta;Y_k),
\label{eq:sec4_pooled_loss}
\end{equation}
and the corresponding pooled estimator
\begin{equation}
\widehat\beta_0
\;:=\;
\argmin_{\beta\in\Theta} L_0(\beta;Y).
\label{eq:sec4_pooled_mle}
\end{equation}
The {pooled population loss} and its minimizer are
\begin{equation}
L_0^\star(\beta)
\;:=\;
\sum_{k=1}^K L_k^\star(\beta),
\qquad
\beta_0^\star
\;:=\;
\argmin_{\beta\in\Theta} L_0^\star(\beta).
\label{eq:sec4_pooled_pop}
\end{equation}
The pooled population target \(\beta_0^\star\) need not coincide with the
common center \(\beta_c^\star\).

The pooled estimator $\widehat\beta_0$ is statistically stable, so it is appealing when local data are too scarce for reliable local estimation. When stores are heterogeneous, however, its target \(\beta_0^\star\) differs from each $\beta_k^\star$, so it is systematically biased for every store. It is therefore not ideal for estimating demand to guide pricing or assortment decisions across heterogeneous markets. Nevertheless, the pooled estimator provides a useful benchmark, since it represents the maximal variance reduction available from borrowing across the chain.

\subsection{Shrinkage Toward Leave-One-Out Pooled Targets}

We introduce a shrinkage estimator $\widehat\beta_k^{\mathrm{shr}}$ that balances the local estimator's variance against the pooled estimator's bias by shrinking the local estimator $\widehat\beta_k$ toward a pooled target. %
The target is not the full-sample estimator \(\widehat\beta_0\) but a pooled estimator constructed from $(Y_\ell)_{\ell \neq k}$, the data of all stores except $k$. Specifically, the store-$k$ leave-one-out pooled target is
\begin{equation}
\widehat\beta_0^{(-k)}
\;:=\;
\argmin_{\beta\in\Theta} L_0^{(-k)}(\beta;Y),
\label{eq:sec4_loo_mle}
\end{equation}
where $L_0^{(-k)}(\beta;Y)$ is the store-\(k\) \textit{leave-one-out pooled empirical loss}
\begin{equation}
L_0^{(-k)}(\beta;Y)
\;:=\;
\sum_{\ell\neq k} L_\ell(\beta;Y_\ell).
\label{eq:sec4_loo_loss}
\end{equation}
The leave-one-out population loss and its minimizer are
\begin{equation}
L_0^{(-k),\star}(\beta)
\;:=\;
\sum_{\ell\neq k} L_\ell^\star(\beta),
\qquad
\beta_0^{(-k),\star}
\;:=\;
\argmin_{\beta\in\Theta} L_0^{(-k),\star}(\beta).
\label{eq:sec4_loo_pop}
\end{equation}

All displayed minimizers exist by compactness of $\Theta$. They are
unique because Hessian stability and Assumptions~\ref{ass:parameter}--
\ref{ass:fisher} imply
$\Sigma_X(\beta)\succeq m e^{-2BD_\Theta}I_d$ for every
$\beta\in\Theta$, so each local or pooled objective is strictly convex
on $\Theta$.

We introduce a shrinkage weight $w_k \in [0,1]$ that controls the degree of shrinkage. The store-\(k\) shrinkage estimator is
\begin{equation}
\widehat\beta_k^{\mathrm{shr}}(w_k)
\;:=\;
(1-w_k)\,\widehat\beta_k
\;+\;
w_k\,\widehat\beta_0^{(-k)}.
\label{eq:sec4_shrinkage}
\end{equation}
When \(w_k=0\), the estimator reduces to the local MLE. When \(w_k=1\), it fully adopts the leave-one-out pooled estimator. When the weight is clear from context, we write \(\widehat\beta_k^{\mathrm{shr}}\) for brevity. 
Because \(\Theta\) is convex, \(\widehat\beta_k^{\mathrm{shr}}(w_k)\in\Theta\)
for every \(w_k\in[0,1]\).

Why shrink toward the leave-one-out estimator $\widehat\beta_0^{(-k)}$ rather than the full pooled estimator $\widehat\beta_0$? The two targets are close when the network is large and no single store dominates the pooled sample, so the choice matters little for the estimate itself. 
However, it matters for the finite-sample analysis. Since the shrinkage estimator combines $\widehat\beta_k$ with a pooled target, its error mixes the variance of the local estimator with the bias of the pooled target, and separating these contributions requires the two estimators to be independent. The leave-one-out target ensures this. The local estimator \(\widehat\beta_k\) depends only on  \(Y_k\), while \(\widehat\beta_0^{(-k)}\) depends only on \( (Y_\ell)_{\ell \neq k} \), so under independence across stores they are independent. Had the target been the full pooled estimator $\widehat\beta_0$, which reuses store $k$'s data, the two would be dependent and their contributions would be entangled.

A practical question is how to set $w_k$. If \(w_k\) is too small, the estimator inherits the noise of the local MLE. If \(w_k\) is too large, it can over-pool and blur meaningful local structure. Thus, the choice of
\(w_k\) determines how the estimator balances local variance reduction
against heterogeneity-induced bias. Section~\ref{sec:theory} quantifies this tradeoff, deriving the optimal shrinkage weight as a function of local sample size, feature dimension, and cross-store heterogeneity.

\section{Finite-Sample Guarantees for Adaptive Shrinkage}
\label{sec:theory}

Because $\beta_k^\star$ uniquely minimizes $L_k^\star$, the excess risk $L_k^\star(\widehat\beta) - L_k^\star(\beta_k^\star)$ of an estimator $\widehat\beta$ is nonnegative and vanishes only when $\widehat\beta = \beta_k^\star$. An estimator is computed from the sample, so its excess risk is a random variable. Taking the expectation over the data gives the expected excess risk, which measures the typical cost of an estimation method.

We compare the shrinkage and local estimators through their expected excess risks. For a shrinkage weight $w_k \in [0,1]$, define the expected {excess risk differential} as
\begin{equation}
\mathcal{R}_{\mathrm{shr},k}(w_k)
\;:=\;
\mathbb{E}\!\left[
L_k^\star(\widehat\beta_k^{\mathrm{shr}}(w_k)) - L_k^\star(\beta_k^\star)\right]
 \;-\; \mathbb{E}\!\left[
L_k^\star(\widehat\beta_k)
- L_k^\star(\beta_k^\star)\right],
\label{eq:excess_loss}
\end{equation}
where the expectations are taken over the data $Y$ generated by the true store models. For a weight vector \( w=(w_1,\dots,w_K) \in [0,1]^K\), summing across stores gives $\mathcal{R}_{\mathrm{shr}}(w) = \sum_{k=1}^K \mathcal{R}_{\mathrm{shr},k}(w_k)$. 

The total differential $\mathcal{R}_\mathrm{shr}(w)$ measures how much shrinkage raises or lowers the excess risk relative to separate estimation. A negative value indicates that shrinkage improves on local estimation in expectation, aggregated across stores, whereas a positive value indicates that it performs worse. Our goal in this section is to choose a shrinkage weight for which we can establish a finite-sample upper bound on $\mathcal{R}_{\mathrm{shr}}$.

The challenge is that the expected excess risk does not decompose exactly into variance and bias. In much of the small-data pooling literature, performance is measured by a mean-squared error, which is quadratic in the parameter, so the effect of pooling admits an exact bias--variance decomposition \citep{li2023mostly}. The shrinkage weight balancing the two can then be read off that decomposition, and the risk bounded in the same terms. The MNL population loss $L_k^\star$ is not quadratic, so no such exact decomposition exists. 

We overcome this challenge by expanding $L_k^\star$ around $\beta_k^\star$, thus writing the expected excess risk as a quadratic leading term plus a nonlinear remainder. The quadratic term encodes the bias--variance tradeoff, and minimizing it yields an quadratic-optimal weight. This weight can be too aggressive: it can move the estimate far enough toward the pooled target that the remainder overtakes the quadratic. We therefore propose a damped weight that shrinks less aggressively. To derive a finite-sample guarantee, we establish that $L_k^\star$ is \emph{self-concordant-like} (SCL), a curvature condition under which the remainder is bounded explicitly at any sample size. The rest of this section develops these steps in turn, culminating in a finite-sample bound on $\mathcal{R}_{\mathrm{shr}}$ for the damped weight.

\subsection{Local Quadratic Approximation of Population Loss}
\label{sec:quadratic_appx}

Fix a store $k$. For any parameter $\beta$, expanding $L_k^\star$ about $\beta_k^\star$ writes the excess risk $L_k^\star(\beta) - L_k^\star(\beta_k^\star)$ as a quadratic form in the \emph{parameter error} $\beta - \beta_k^\star$, plus a remainder.  We discuss the two in turn.

The quadratic term is set by how sharply the loss $L_k^\star$ curves at $\beta_k^\star$. Specifically, let $H_k^\star$ denote the Hessian at $\beta_k^\star$ where
\[
H_k^\star 
:= 
\nabla^2 L_k^\star(\beta_k^\star) 
= 
n_k \Sigma_X^k
\]
and $\Sigma_X^k$, defined in Section~\ref{subsec:assumptions}, is the covariance matrix of the random feature vector $X$ under choice model $P(\beta_k^\star)$. Here, the second equality follows from Corollary~\ref{cor:app_population_hessian} in the Appendix. Since $\Sigma_X^k$ is the per-observation information matrix, $H_k^\star$ is the Fisher information carried by all $n_k$ store-$k$ observations. We measure the parameter error in the geometry induced by $H_k^\star$, writing
\[
\|v\|_k := \sqrt{v^\top H_k^\star v},
\qquad
\langle u,v\rangle_k := u^\top H_k^\star v,
\qquad
u,v\in\mathbb{R}^d,
\]
for the resulting local norm and inner product. Rather than treating all directions equally, as the Euclidean norm does, $\|\cdot\|_k$ weights each direction according to the curvature of the loss at $\beta_k^\star$. %

Controlling the remainder relies on a structural property of the MNL likelihood. As we show in the Appendix (Theorem~\ref{thm:app_mnl_scl}), under bounded features (Assumption~\ref{ass:bounded}) the population loss $L_k^\star$ is \emph{self-concordant-like}, meaning its third derivative is controlled by its curvature, so the Hessian varies slowly across parameters in the neighborhood of $\beta_k^\star$. This is what makes a finite-sample theory possible. A classical Taylor expansion would render the remainder of the quadratic approximation negligible only as $n_k \rightarrow \infty$. The SCL property instead bounds it at any sample size.

The next lemma makes the expansion precise, in terms of the normalized local
displacement
\[
\eta_k(\beta)
:=
\frac{2B\|\beta-\beta_k^\star\|_k}{\sqrt{n_k m}}, 
\]
where $B$ is the feature bound of Assumption~\ref{ass:bounded} and $m$ is the curvature lower bound of Assumption~\ref{ass:fisher}. 

\begin{lemma}
\label{lem:theory_local_expansion}
For each store \(k\) and any \(\beta\in\mathbb{R}^d\),
\begin{equation}
L_k^\star(\beta)-L_k^\star(\beta_k^\star)
=
\frac{1}{2}\,
\|\beta-\beta_k^\star\|_k^2
+
r_k(\beta),
\label{eq:theory_local_expansion}
\end{equation}
where the remainder satisfies
\begin{equation}
|r_k(\beta)|
\le
\frac{1}{2}\,
\|\beta-\beta_k^\star\|_k^2\,
\omega_{\mathrm{SCL}}\!\left(\eta_k(\beta)\right),
\label{eq:theory_local_remainder_bound}
\end{equation}
with
\[
\omega_{\mathrm{SCL}}(t)
:=
\begin{cases}
\dfrac{2e^t-(t^2+2t+2)}{t^2}, & t>0,\\[6pt]
0, & t=0.
\end{cases}
\]
In particular,
$
\omega_{\mathrm{SCL}}(t)=O(t)
\text{ as } t\downarrow0.
$
Thus, whenever \(\eta_k(\beta)\to0\), the remainder is negligible
relative to the leading quadratic term:
\[
|r_k(\beta)|
=
o\!\left(
\|\beta-\beta_k^\star\|_k^2
\right).
\]
\end{lemma}

\vspace{2mm}
The expansion of excess risk holds for every \(\beta \in \mathbb{R}^d\). Evaluating
\eqref{eq:theory_local_expansion} at \(\beta = \widehat\beta_k^{\mathrm{shr}}(u)\)
and \(\beta = \widehat\beta_k\), subtracting, then taking the expectation gives the decomposition of the excess risk differential
\begin{align}
\mathcal{R}_{\mathrm{shr},k}(u) &\;\;=\;\;
\frac{1}{2} \;
\E \left[
\|\widehat\beta_k^{\mathrm{shr}}(u) - \beta_k^\star\|_k^2
-
\|\widehat\beta_k-\beta_k^\star\|_k^2
\right] \;\;+\;\;
\E \left[ r_k(\widehat\beta_k^{\mathrm{shr}}(u))
-
r_k(\widehat\beta_k) \right].
\label{eq:theory_risk_decomposition}
\end{align}
This decomposes $\mathcal{R}_{\mathrm{shr},k}(u)$ into two parts. The first is the leading term, half the expected difference in the squared parameter error. The second is the remainder.

It remains to express the leading term of the decomposition in a form that reveals the shrinkage tradeoff. Define the parameter error vectors
\begin{equation}
e_k := \widehat\beta_k - \beta_k^\star,
\qquad
g_k := \widehat\beta_0^{(-k)} - \beta_k^\star,
\label{eq:error_vectors}
\end{equation}
where \(e_k\) is the local parameter error and \(g_k\) is the leave-one-out pooled parameter error.  By definition of the shrinkage estimator \eqref{eq:sec4_shrinkage}, the shrinkage parameter error for shrinkage weight $u$ is
\begin{equation}
\widehat\beta_k^{\mathrm{shr}}(u)-\beta_k^\star
=
(1-u)e_k+u g_k.
\label{eq:shrinkage_gap}
\end{equation}
Hence, the leading term becomes
\[
\frac{1}{2}
\E \left[
\|(1-u)e_k+u g_k\|_k^2
-
\|e_k\|_k^2
\right].
\]
The following lemma collapses this term into a one-dimensional quadratic  in the shrinkage weight~$u$, whose coefficients are the benefit and the cost of shrinking.

\begin{lemma}
\label{lem:theory_quadratic_identity}
For each store \(k\) and any scalar shrinkage weight \(u\in\mathbb{R}\),
\begin{equation}
\frac{1}{2} \; \mathbb{E}\!\left[
\|(1-u)e_k+u g_k\|_k^2-\|e_k\|_k^2
\right]
=
\frac{1}{2} \left( \alpha_k u^2-2\zeta_k u \right),
\label{eq:theory_quadratic_identity}
\end{equation}
where
\begin{equation}
\alpha_k
:=
V_k^{\mathrm{loc}}
+
V_k^{\mathrm{pool}}
-
2C_k,
\qquad
\zeta_k
:=
V_k^{\mathrm{loc}}
-
C_k,
\label{eq:theory_alpha_zeta}
\end{equation}
with
\[
V_k^{\mathrm{loc}}
:=
\mathbb{E}\|e_k\|_k^2,
\qquad
V_k^{\mathrm{pool}}
:=
\mathbb{E}\|g_k\|_k^2,
\qquad
C_k
:=
\mathbb{E}\langle e_k,g_k\rangle_k .
\]
\end{lemma}

\vspace{3mm}
Near $u=0$, the leading term of the decomposition is dominated by $-\zeta_k u$, so a small move toward the pooled estimate changes the squared parameter error at rate $-\zeta_k$. When $\zeta_k > 0$, this rate is negative, so shrinking lowers the squared error. Intuitively, the local estimator is noisy, and even a small move toward the more stable pooled target reduces this variance. Hence, coefficient $\zeta_k$ can be interpreted as the marginal gain from that variance reduction.

The marginal cost of shrinkage is $\alpha_k u$, which can exceed this marginal gain for large $u$. This cost arises because shrinking displaces the estimator from the true parameter $\beta_k^\star$ due to a heterogeneity bias. The magnitude of $\alpha_k$ is
\[
\alpha_k
=
V_k^{\mathrm{loc}}
+
V_k^{\mathrm{pool}}
-
2C_k
=
\mathbb{E}\|g_k-e_k\|_k^2
=
\mathbb{E}\|
\widehat\beta_0^{(-k)}-\widehat\beta_k
\|_k^2,
\]
the expected squared distance between the local and pooled estimators, so the farther apart they are, the sooner the heterogeneity bias overtakes the  variance reduction and shrinkage stops paying off. The ratio \(\zeta_k/\alpha_k\), which we call the gain-to-cost ratio, weighs the variance reduction against the heterogeneity bias.

\subsection{The Quadratic-Optimal Weight}\label{sec:quadratic_benchmark}

From Lemma~\ref{lem:theory_quadratic_identity} we have ${\mathcal{R}}_{\mathrm{shr},k}(u) = Q_k(u) + \E[r_k(\widehat\beta_k^{\mathrm{shr}}(u)) - r_k(\widehat \beta_k)]$, where
\[
Q_k(u)
:=
\frac{1}{2} \left( \alpha_k u^2
-
2\zeta_k u \right),
\qquad
u\in[0,1]
\]
is a quadratic function in the shrinkage weight. Setting the remainder aside for the moment, we ask which weight minimizes the leading quadratic term. 

\begin{lemma}
\label{lem:theory_quadratic_benchmark}
Suppose \(\alpha_k>0\). Then the exact minimizer of \(Q_k(u)\) over
\(u\in[0,1]\) is
\begin{equation}
w_{k,\mathrm{quad}}^{\mathrm{opt}}
:=
\Pi_{[0,1]}\!\left(\frac{\zeta_k}{\alpha_k}\right),
\label{eq:theory_quadratic_benchmark}
\end{equation}
where \(\Pi_{[0,1]}\) denotes projection onto \([0,1]\).
\end{lemma}

The quadratic-optimal weight is simply the gain-to-cost ratio projected onto $[0,1]$. If \(0<\zeta_k<\alpha_k\), it shrinks more when the variance reduction is large relative to the heterogeneity bias. If \(\zeta_k\le 0\), shrinking does not improve $Q_k(\cdot)$ even locally, so it keeps the local estimator. If \(\zeta_k\ge\alpha_k\), it clips to $u = 1$ and adopts the leave-one-out pooled estimator outright.

The remainder we set aside is what makes this quadratic-optimal weight unreliable. By Lemma~\ref{lem:theory_local_expansion}, the remainder grows with the shrinkage estimator's displacement from \(\beta_k^\star\), which by \eqref{eq:shrinkage_gap} is equal to \((1-u)e_k+u g_k\). To see what can make that displacement large, split the pooled error into
\begin{align} \label{eqn:g_composition}
g_k
=
\underbrace{\bigl(\widehat\beta_0^{(-k)}-\beta_0^{(-k),\star}\bigr)}_{=:\varepsilon_k^{\mathrm{pool}}}
+
\underbrace{\bigl(\beta_0^{(-k),\star}-\beta_k^\star\bigr)}_{=:\delta_k},
\end{align}
where \(\varepsilon_k^{\mathrm{pool}}\) is the estimation error of the leave-one-out pooled estimator around its own population target, and \(\delta_k\) is the structural mismatch between that target and store \(k\)'s true parameter. Substituting, the displacement is
\begin{align} \label{eqn:displacement}
\widehat{\beta}_k^{\mathrm{shr}}(u) - \beta_k^\star \; = \; (1-u)e_k+u g_k
\; = \;
\underbrace{(1-u)e_k+u\varepsilon_k^{\mathrm{pool}}}_{\text{estimation error}}
+
\underbrace{u\delta_k}_{\text{deterministic shift}}.
\end{align}
The two estimation-error terms are controlled by the concentration bounds of Appendix~\ref{app:mle}, and both shrink as their samples grow. The shift \(u\delta_k\) does not: $\delta_k$ is a fixed population quantity, so only the shrinkage weight $u$ can control it. A weight near $1$ can carry the estimate outside the neighborhood of \(\beta_k^\star\) where the quadratic approximation is accurate, and there the remainder can offset the reduction in $Q_k$ that the quadratic-optimal weight was chosen to achieve.

\subsection{A Damped Shrinkage Weight}

The quadratic-optimal weight is unreliable precisely where it clips to $u=1$, since the deterministic shift is then at full size and, unlike the estimation errors, does not shrink as samples grow. We therefore replace the clipping with a smooth map that never reaches full pooling, leaving the weight essentially unchanged where it is small. Let
\[
\tau_k
:=
\left(\frac{\zeta_k}{\alpha_k}\right)_+,
\]
denote the positive part of the gain-to-cost ratio, and set
\begin{equation}
w_k^\star
:=
\frac{\tau_k}{1+\tau_k}
=
\begin{cases}
0, & \zeta_k \le 0,\\[6pt]
\dfrac{\zeta_k}{\alpha_k+\zeta_k}, & \zeta_k > 0.
\end{cases}
\label{eq:theory_analytic_weight}
\end{equation} 
The damped weight $w_k^\star$ has three properties. It is monotone in \(\tau_k\), so stores with a larger gain-to-cost ratio  borrow more. For small
\(\tau_k\), it agrees with the quadratic-optimal weight to first order,
\[
w_k^\star
=
\tau_k+O(\tau_k^2)
\qquad\text{as } \tau_k\downarrow0,
\]
so nothing changes in the regime where the quadratic approximation is accurate. And it maps into $[0,1)$, so the deterministic shift $w_k^\star\delta_k$ never reaches its full size. The particular map $\tau\mapsto\tau/(1+\tau)$ is a design choice; any map into $[0,1)$ that is first-order equivalent to the identity near zero would serve the same purpose.

The next lemma confirms that this weight keeps the shrinkage benefit while limiting how far the estimate moves toward the pooled target.

\begin{lemma}
\label{lem:theory_weight_properties}
Let \(\zeta_k^+ := \max\{\zeta_k,0\}\), and define $V_{\mathrm{loc}} := \E\|e_k\|_k^2$. Then \(w_k^\star\in[0,1]\), and the following properties hold.
\begin{enumerate}[label=(\roman*),leftmargin=*]
    \item \textbf{Leading term bound:}
    \begin{equation}
    Q_k(w_k^\star) 
    \;\;\le\;\;
    -
    \frac{1}{2}\frac{(\zeta_k^+)^2}{\alpha_k+\zeta_k^+}.
    \label{eq:theory_quadratic_certificate}
    \end{equation}

    \item \textbf{Systematic bias bound:}
    \begin{equation}
     \| \; w_k^\star \, \E (g_k) \; \|_k
    \;\;\le\;\;
    \sqrt{V_k^{\mathrm{loc}}}.
    \label{eq:theory_mean_displacement_control}
    \end{equation}
\end{enumerate}
\end{lemma}

The two properties capture the opposing requirements the weight must meet: retain the reduction in the excess risk from shrinking, and limit the systematic error that borrowing introduces.

The first property shows that the damped weight reduces the leading term of $\mathcal{R}_{\mathrm{shr},k}$. In the interior case \(0<\zeta_k<\alpha_k\), the quadratic-optimal weight attains the minimum of $Q_k$, namely \(-\tfrac12\zeta_k^2/\alpha_k\). The damped weight does not reach it, but property~(i) certifies
\[
Q_k(w_k^\star)
\le
-\frac12 \frac{\zeta_k^2}{\alpha_k+\zeta_k}
=
\frac{1}{1+\tau_k}\cdot \left( -\frac12 \frac{\zeta_k^2}{\alpha_k} \right),
\]
so the leading term is at most a fraction \(1/(1+\tau_k)\) of the optimal value. Damping therefore costs little when \(\tau_k\) is small and more when \(\tau_k\) is large.

The second property controls the remainder. By Lemma~\ref{lem:theory_local_expansion}, the magnitude $|r_k(\beta)|$ is at most $\tfrac12 \|\beta - \beta_k^\star\|_k^2\omega_{\mathrm{SCL}}(\eta_k(\beta))$, and $\eta_k(\beta)$ is itself proportional to the parameter displacement $\|\beta - \beta_k^\star\|_k$. At $\beta = \widehat\beta_{k}^\mathrm{shr} (w_k^\star)$, equation \eqref{eqn:displacement} indicates that the parameter displacement is bounded by  a local estimation error term, $(1-w_k^\star) \|e_k\|_k$, plus a systematic bias term, $\|w_k^\star \E (g_k)\|_k$. The local estimation error is controlled at scale $\sqrt{V_k^{\mathrm{loc}}}$ by the concentration bounds of Appendix~\ref{app:mle}. The systematic bias is deterministic, so it has to be bounded separately, which is what property~(ii) does. Since $\sqrt{V_k^{\mathrm{loc}}}$ is the expected size of the local MLE error $\|e_k\|_k$, the bound says that shrinkage at $w_k^\star$ never introduces a systematic bias larger than the noise it is meant to remove. With both displacement terms at the same scale, Lemma~\ref{lem:app_remainder_control} turns this into the remainder bound that is used to bound $\mathcal{R}_{\mathrm{shr},k}(w_k^\star)$ in Theorem~\ref{thm:theory_main} below.

\subsection{Main Finite-Sample Bound}

We can now state the finite-sample guarantee for the damped weight. It holds in a moderate heterogeneity range $\Delta \in [0,\bar\Delta]$ where the heterogeneity bias introduced by shrinkage stays controlled while the variance reduction remains large enough to yield a net benefit. The threshold $\bar\Delta$ is given explicitly in Appendix~\ref{app:barDelta}. 

\begin{theorem}
\label{thm:theory_main} Let
\(w^\star=(w_1^\star,\dots,w_K^\star)\) be the vector of damped weights defined in \eqref{eq:theory_analytic_weight}. 
Suppose Assumptions~\ref{ass:bounded}--\ref{ass:sample} hold. Then there
exist constants $c_1,c_2,c_3,c_4>0$ and a threshold
\(\bar\Delta>0\), depending only on
\((B,m,M,C_{\max},D_\Theta,\rho_\Theta)\), such that 
whenever \(\Delta\in[0,\bar\Delta]\), every store $k$ satisfies
\begin{equation}
\mathcal{R}_{\mathrm{shr},k}(w_k^\star) 
\le
-
\frac{c_1 d^2}{c_2 d+c_3 n_k\Delta^2}
+
c_4\,\frac{d^{3/2}}{\sqrt{n_k}},
\label{eq:theory_per_store_bound}
\end{equation}
and summing over stores gives
\begin{equation}
\mathcal{R}_{\mathrm{shr}}(w^\star)
\le
\sum_{k=1}^K
n_k \cdot \left[ - \left(\frac{d}{n_k} \right) \cdot \frac{c_1}{c_2 +c_3 (n_k/d) \Delta^2} + c_4 \left(\frac{d}{n_k}\right)^{3/2} \right].
\label{eq:theory_total_bound}
\end{equation}
\end{theorem}

\vspace{3mm}

A negative value of  $\mathcal{R}_{\mathrm{shr}}(w^\star)$ certifies that shrinkage lowers the expected excess risk relative to local estimation.  The aggregate bound in \eqref{eq:theory_total_bound} is governed by two
quantities: the normalized local sample size \(\rho_k := n_k/d\) and the cross-store heterogeneity level \(\Delta\). By
Assumption~\ref{ass:sample}, we have
\(
\rho_k\ge C_0 \ge 1.
\) Rewriting the bracketed term in \eqref{eq:theory_total_bound} gives
\[
-\frac{1}{\rho_k}\,
\frac{c_1}{c_2+c_3\rho_k \Delta^2}
+
c_4 \left(\frac{1}{\rho_k}\right)^{3/2}.
\]
This makes the comparative statics transparent. When
\(\Delta^2 \ll 1/\rho_k\), the denominator of the first term is essentially
constant, so the first term is of order \(-1/\rho_k\),
while the positive nonlinear remainder is of smaller order
\(1/\rho_k^{3/2}\). Thus, shrinkage is most attractive when local data are
thin and store heterogeneity is low. By contrast, when \(\Delta^2\) grows
relative to \(1/\rho_k\), the negative term weakens like
\(-1/(\rho_k^{2}\Delta^2)\), reflecting the rising cost of a shrinkage estimator that is further displaced from \(\beta_k^\star\).

The same argument also yields the order of magnitude of the damped weight.

\begin{corollary}
\label{cor:theory_weight_scale}
Let $\Delta_k := \|\delta_k\|_2$ denote the size of the structural mismatch $\delta_k := \beta_0^{(-k),\star}-\beta_k^\star$ between the pooled target and the store $k$ parameter. Under the assumptions of Theorem~\ref{thm:theory_main}, there exist constants \(0<c_-\le c_+<\infty\), depending only on
\((B,m,M,C_{\max},D_\Theta,\rho_\Theta,\bar\Delta)\), such that whenever \(\Delta\in[0,\bar\Delta]\), every store \(k\) satisfies
\begin{equation}
c_-\,\frac{d}{d+n_k\Delta_k^2}
\;\le\;
w_k^\star
\;\le\;
c_+\,\frac{d}{d+n_k\Delta_k^2}.
\label{eq:theory_weight_scale}
\end{equation}
\end{corollary}

\vspace{2mm}
Corollary~\ref{cor:theory_weight_scale} shows that the damped weight takes the intuitive shrinkage form
\[
w_k^\star
\asymp
\frac{(d/n_k)}{(d/n_k) + \Delta_k^2},
\]
a ratio of variance to variance plus squared bias. The numerator $(d/n_k)$ is the local estimation variance, since by Theorem~\ref{thm:app_mle_concentration} the local MLE has squared Euclidean error of order $d/n_k$. The term \(\Delta_k^2\) in the denominator is the squared bias of the pooled target. The damped weight therefore borrows more when store \(k\)'s local data are thin or its
pooled target lies close to $\beta_k^\star$.

The two heterogeneity quantities in our results have distinct roles. The global heterogeneity \(\Delta\) determines when the finite-sample guarantee holds, since
Theorem~\ref{thm:theory_main} applies uniformly across stores whenever
$\Delta \le \bar\Delta$.  
The store-specific mismatch $\Delta_k$ determines how
much each store borrows, since by Corollary~\ref{cor:theory_weight_scale}
the weight $w_k^\star$ scales with $\Delta_k$.
The weight therefore adapts to each store's distance from its pooled
target, which may be smaller than the uniform bound $\Delta$.

Figure~\ref{fig:theory_phase_diagram} visualizes these comparative
statics. For each level of global heterogeneity $\Delta$ and local
sample ratio $n_k/d$, the color shows the reduction in expected excess
risk that Theorem~\ref{thm:theory_main} guarantees for shrinkage relative
to local estimation. Positive values
(blue) mark where shrinkage is guaranteed to improve on local
estimation, and negative values (red) mark where the bound no longer
guarantees an improvement. The guaranteed gain is largest when
heterogeneity is small, and fades as $\Delta$ grows and the local sample
becomes too thin for the remainder to be controlled.

\begin{figure}[t]
    \centering
    \includegraphics[width=0.82\textwidth]{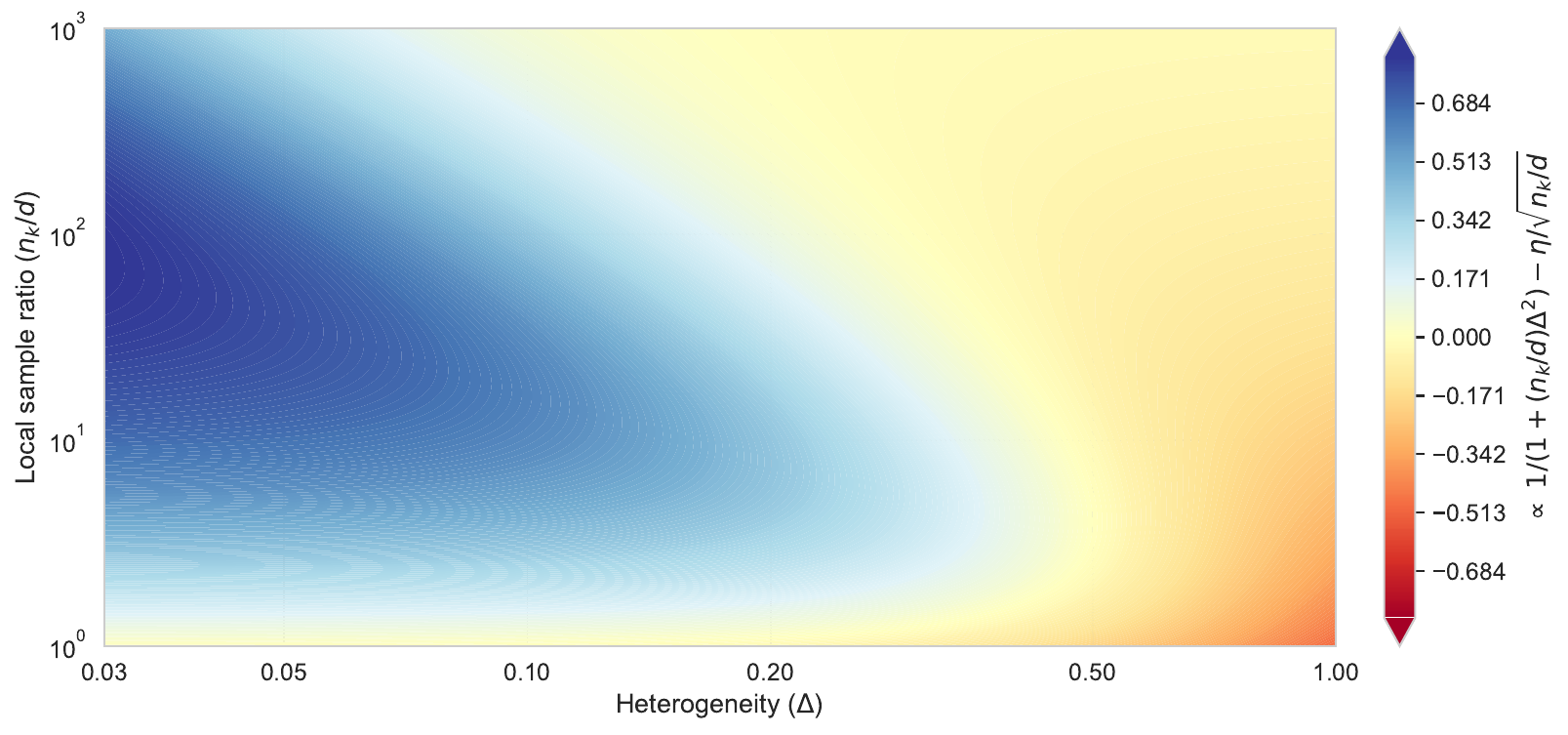}
    \caption{Reduction in expected excess risk that
Theorem~\ref{thm:theory_main} guarantees for shrinkage relative to local
estimation, over cross-store heterogeneity \(\Delta\) and the local
sample ratio \(n_k/d\). Blue marks where shrinkage is guaranteed to
improve on local estimation, and red where the bound gives no such
guarantee.}
    \label{fig:theory_phase_diagram}
\end{figure}

\subsection{Order-Based Plug-In Weights}

The damped weight \(w_k^\star\) is built from \(\alpha_k\) and \(\zeta_k\), but it is not directly computable. As defined in \eqref{eq:theory_alpha_zeta}, these quantities are expectations over the sampling distribution of the estimators, describing how the local and pooled errors vary across samples. Estimating them requires knowing the sampling variability of the local and pooled estimators.

Corollary~\ref{cor:theory_weight_scale} offers a feasible route. The shrinkage weight scale it gives depends on the store only through the mismatch
\(\Delta_k\), a fixed distance between $\beta_0^{(-k),\star}$ and $\beta_k^\star$. Thus, it requires only point estimates, which the fitted estimators \(\widehat\beta_0^{(-k)}\) and \(\widehat\beta_k\) supply
directly. 

Because the corollary determines only the order of $w_k^\star$, we keep its functional form and absorb the constants into a single parameter
\(\lambda\), giving the \textit{order-based} target
\begin{equation}
w_k^\circ(\lambda)
:=
\frac{d}{d+\lambda n_k\Delta_k^2},
\qquad \lambda>0.
\label{eq:theory_practical_target}
\end{equation}
Here, \(\lambda\) is a finite-sample calibration parameter that
can be normalized, selected by validation, or treated as a sensitivity
parameter. The corresponding plug-in weight is
\begin{equation}
\widehat w_k(\lambda)
:=
\frac{d}{d+\lambda n_k\,\widehat\Delta_k^2},
\label{eq:theory_plugin_weight}
\end{equation}
where \(\widehat\Delta_k\) estimates the structural mismatch \(\Delta_k\). A natural choice is
\(\widehat\Delta_k = \|\widehat\beta_0^{(-k)}-\widehat\beta_k\|_2\),
possibly with smoothing or finite-sample calibration. This estimator is
consistent, \(\widehat\Delta_k\xrightarrow{p}\Delta_k\), since \(\widehat\beta_k\)
and \(\widehat\beta_0^{(-k)}\) converge to their respective targets. %

The next result describes the asymptotic behavior of the order-based plug-in weight, and in particular when it recovers the right amount of pooling.

\begin{lemma}
\label{lem:theory_plugin_behavior}
Fix \(\lambda>0\) and a store \(k\). All convergences are taken as \(d/n_k\to0\), and \(X=o_p(a)\) means
\(X/a\xrightarrow{p}0\). 

\begin{enumerate}[label=(\roman*),leftmargin=*]
    \item If \(\Delta_k>0\), then \(w_k^\circ(\lambda)\to0\). If in addition
    \(\widehat\Delta_k\xrightarrow{p}\Delta_k\), then \(\widehat w_k(\lambda)\xrightarrow{p}0\) and 
    \(\widehat w_k(\lambda)/w_k^\circ(\lambda)\xrightarrow{p}1\).

    \item If \(\Delta_k=0\), then
    \(w_k^\circ(\lambda)=1\). If in addition $\widehat{\Delta}_k = o_p(\sqrt{d/n_k})$, then \(\widehat w_k(\lambda)\xrightarrow{p}1\). Consistency, \(\widehat\Delta_k\xrightarrow{p}\Delta_k = 0\), does not suffice.
\end{enumerate}
\end{lemma}

\vspace{2mm}
Lemma~\ref{lem:theory_plugin_behavior} characterizes the asymptotic behavior of the order-based plug-in weight. When \(\Delta_k>0\), store \(k\) is structurally distinct from the rest of the chain. As the local sample grows relative to the dimension, the order-based target weight \(w_k^\circ(\lambda)\) vanishes. This reflects a bias-variance tradeoff. The local variance shrinks with \(n_k\), while the store heterogeneity bias persists, so borrowing pays off less. Our choice of \(\widehat\Delta_k = \|\widehat\beta_0^{(-k)} - \widehat\beta_k\|_2\) is a consistent estimator of \(\Delta_k\), so the lemma states that the plug-in weight vanishes at the same rate as the target weight, \(\widehat w_k(\lambda)/w_k^\circ(\lambda)\xrightarrow{p}1\). 

The case \(\Delta_k=0\) is more subtle. Here there is no heterogeneity bias, so the order-based target places all weight on the pooled estimator, \(w_k^\circ(\lambda)=1\).  Whether the plug-in weight \(\widehat w_k(\lambda)\) converges to that target is less immediate. From \eqref{eq:theory_plugin_weight}, \(\widehat\Delta_k\) enters the
denominator through the growing factor \(n_k/d\), so even a small estimation
error can make the plug-in weight under-shrink relative to the target.  Consistency of \(\widehat\Delta_k\) is therefore not enough, and the sharper
rate \(\widehat\Delta_k=o_p(\sqrt{d/n_k})\) is needed instead.  This rate is difficult to ensure near \(\Delta_k=0\). There the two population
parameters \(\beta_k^\star\) and \(\beta_0^{(-k),\star}\) coincide, so
\(\widehat\Delta_k\) reduces to the difference between two noisy estimators and
is of the same order \(\sqrt{d/n_k}\) as the local estimation error. The sharper rate requires \(\widehat\Delta_k\) to be smaller than this order. To overcome this, we propose a validation-based shrinkage weight, selected directly from the held-out likelihood, that avoids estimating \(\Delta_k\) altogether. We evaluate both the plug-in weight and the cross-validated weight in the simulation and empirical analyses.

Overall, Theorem~\ref{thm:theory_main}, Corollary~\ref{cor:theory_weight_scale}, and Lemma~\ref{lem:theory_plugin_behavior} show that the appropriate level of shrinkage is governed by the interaction among dimension, local sample size, and heterogeneity. The order-based plug-in weight is an explicit, implementable family of weights, while the cross-validated weight
provides an empirical alternative when held-out data are available. The two play complementary roles, with the plug-in weight following directly from the theory and the cross-validated weight letting the data tune the amount of borrowing.

\section{Implications for Assortment and Pricing Decisions}
\label{sec:operational}

We now translate the statistical guarantees for shrinkage into operational
regret bounds. In retail execution, the cost of estimation error is not
only statistical. Noisy estimates of the choice model parameters can change the inputs to downstream assortment and pricing problems. For example, they
may distort predicted substitution patterns, alter product rankings, or
shift the estimated price sensitivity, leading the retailer to select a
different assortment or choose a different markup.

We state the operational regret bounds we derive in this section for a generic shrinkage weight \(w_k\in[0,1]\), so they can be evaluated at an analytically chosen weight. The propositions below take the amount of shrinkage as given and show how it affects assortment and pricing regret.

\subsection{Assortment Decisions}
\label{subsec:operational_assortment}

We first consider the unconstrained assortment problem. At store \(k\), the retailer
selects a subset of products \(S\subseteq\mathcal N\) to offer. Under
the MNL demand model, the expected revenue from assortment \(S\) under
parameter \(\beta\) is
\begin{equation}
R(S;\beta)
=
\sum_{j\in S} p_j\, P_j(S;\beta),  %
\label{eq:assortment_revenue}
\end{equation}
where \(p_j\) is the price of product \(j\), and \(P_j(S;\beta)\) is its MNL
choice probability when the retailer offers assortment \(S\). This is the choice
probability \(P_j(\beta)\) of Section~\ref{sec:multistore_mnl}, now written with the offered set made explicit, and it reduces to \(P_j(\beta)\) when \(S\) is the full set of alternatives. %

Given the true store-level parameter \(\beta_k^\star\), the oracle
assortment is
$
S_k^\star
\in
\arg\max_{S\subseteq\mathcal N} R(S;\beta_k^\star).
$
In practice, the retailer does not observe \(\beta_k^\star\) and instead, for a
given shrinkage weight \(w_k\), selects the assortment that maximizes estimated
revenue under \(\widehat\beta_k^{\mathrm{shr}}(w_k)\):
\begin{equation}
\widehat S_k(w_k)
\in
\arg\max_{S\subseteq\mathcal N}
R(S;\, \widehat\beta_k^{\mathrm{shr}}(w_k)).
\label{eq:assortment_plugin}
\end{equation}
The resulting assortment regret is
\begin{equation}
\mathrm{Regret}_k^{\mathrm{assort}}(w_k)
=
R(S_k^\star;\beta_k^\star)
-
R(\widehat S_k(w_k);\beta_k^\star).
\label{eq:assortment_regret_def}
\end{equation}
Because \(\widehat S_k(w_k)\) is computed from the random sample,
\(\mathrm{Regret}_k^{\mathrm{assort}}(w_k)\) is a random variable, which
Proposition~\ref{prop:assortment_regret_hp} bounds with high probability over
the random sample.

\begin{proposition}
\label{prop:assortment_regret_hp}
Fix store \(k\) and  shrinkage weight \(w_k\in[0,1]\) where 
\(
\widehat\beta_k^{\mathrm{shr}}(w_k)
=
(1-w_k)\widehat\beta_k+w_k\widehat\beta_0^{(-k)}
\). 
Suppose Assumptions~\ref{ass:bounded}--\ref{ass:sample} hold and
$\Delta\in[0,\Delta_0]$, where $\Delta_0$ is defined in
Appendix~\ref{sec:preliminaries}. Let
$
N_{-k}:=\sum_{\ell\neq k}n_\ell
$
denote the leave-one-out pooled sample size. Then, for any
\(\delta\in(0,1)\), with probability at least \(1-\delta\),
\begin{align}
\mathrm{Regret}_k^{\mathrm{assort}}(w_k)
\le
4p_{\max}B
\Big[
(1-w_k)\,r_{\mathrm{loc}}(\delta)
+
w_k\bigl(
r_{\mathrm{pool}}(\delta)
+ \Delta_k
\bigr)
\Big],
\label{eq:assortment_regret_hp_clean}
\end{align}
where \(\Delta_k=\|\beta_0^{(-k),\star}-\beta_k^\star\|_2\), \(C_{\mathrm{mle}}>0\) depends only on the primitive regularity
quantities, and
\[
r_{\mathrm{loc}}(\delta)
=
\frac{C_{\mathrm{mle}}}{\sqrt m}
\sqrt{\frac{d+\log(8/\delta)}{n_k}},
\qquad
r_{\mathrm{pool}}(\delta)
=
\frac{C_{\mathrm{mle}}}{\sqrt m}
\sqrt{\frac{d+\log(8/\delta)}{N_{-k}}}.
\]
\end{proposition}

\vspace{2mm}
Proposition~\ref{prop:assortment_regret_hp} gives a direct revenue-level interpretation of the shrinkage tradeoff. The bound combines three quantities, the estimation error of the local MLE \(r_{\mathrm{loc}}(\delta)\), the estimation error of the leave-one-out pooled target \(r_{\mathrm{pool}}(\delta)\), and the structural mismatch \(\Delta_k\). The larger \(w_k\), the more the bound depends on the pooled error and the mismatch, and the less on the local error.

The local error \(r_{\mathrm{loc}}(\delta)=O(\sqrt{d/n_k})\) is large when store
\(k\)'s sample is small relative to the dimension of the choice model. Relying
on thin local data is then costly, because the parameter error can change which
products the assortment optimization selects, especially when several
assortments have similar predicted revenues.

The pooled error \(r_{\mathrm{pool}}(\delta)=O(\sqrt{d/N_{-k}})\) is typically
smaller than the local error, since the pool draws on \(N_{-k}\gg n_k\) samples.
Borrowing toward the pooled target lowers the estimation error but adds the
deterministic mismatch \(\Delta_k\), the distance between the pooled target and
store \(k\)'s true parameter. For \(\Delta\le\Delta_0\), this mismatch satisfies
\[
\Delta_k
\le
2\Delta
+
\frac{2BM}{m}e^{2BD_\Theta+12B\Delta}\Delta^2.
\]
The first term is the direct cross-store difference implied by the 
condition \(\|\beta_\ell^\star-\beta_c^\star\|_2\le\Delta\). The second-order term
comes from the nonlinear MNL population likelihood. The leave-one-out pooled
minimizer is not a Euclidean average of the store-specific parameters, and the
self-concordant-like curvature control bounds the additional displacement. The
mismatch is therefore of order \(\Delta+\Delta^2\).

The assortment bound mirrors the statistical risk comparison of
Section~\ref{sec:theory}, trading local estimation noise against heterogeneity bias as \(w_k\) varies. %

\subsection{Pricing Decisions}
\label{subsec:operational_pricing}

We next consider pricing under the same multi-product MNL structure. To separate pricing from assortment selection, fix an offered set \(S\subseteq\mathcal N\) and let price be the decision. For the pricing analysis, we partition the store-\(k\)
parameter as
\[
\beta_k^\star=(\eta_k^\star,\alpha_k^\star),
\]
separating the intercept and nonprice coefficients \(\eta_k^\star\) from the
price-sensitivity coefficient \(\alpha_k^\star>0\). For product $i \in S$, we write the feature vector as \(X_i=(Z_i,-p_i)\), where $p_i$ is the price and \(Z_i\) collects the nonprice attributes. Hence, $X_i^\top\beta_k^\star =Z_i^\top\eta_k^\star-\alpha_k^\star p_i$, so we write the choice probability as a function of the price vector
\(p=(p_i)_{i\in S}\),
\[
P_i(p;\beta_k^\star)
=
\frac{\exp(Z_i^\top\eta_k^\star-\alpha_k^\star p_i)}
{1+\sum_{j\in S}\exp(Z_j^\top\eta_k^\star-\alpha_k^\star p_j)},
\qquad i\in S.
\]
This is the choice probability of Section~\ref{sec:multistore_mnl}
with the prices exposed as the decision. Let \(c_i\) denote the unit cost of product \(i\). The expected profit at store
\(k\) is
\begin{equation}
\Pi_k(p;\beta_k^\star)=\sum_{i\in S}(p_i-c_i)\,P_i(p;\beta_k^\star).
\label{eq:pricing_profit_vector}
\end{equation}

The oracle pricing problem is
\begin{equation}
p_k^\star\in\argmax_{p\in\mathcal P}\Pi_k(p;\beta_k^\star),
\qquad
\mathcal P:=\{p\in\mathbb R^{|S|}: c_i\le p_i\le\bar p,\ i\in S\}.
\label{eq:pricing_oracle_vector}
\end{equation}
For the unconstrained interior version of this problem, the linear-in-price MNL
model with a common price-sensitivity coefficient implies the familiar
equal-markup structure \citep{li2011pricing}. There exists a scalar
\(m_k^\star\ge0\) such that \(p_{ik}^\star=c_i+m_k^\star\) for all \(i\in S\).
With price bounds, this structure holds for the products whose prices are
interior, while boundary constraints may bind for some items. In either case,
the pricing decision is governed by the same estimated store-level choice model
as the assortment decision.

In practice, the retailer does not observe \(\beta_k^\star\) and, for a given shrinkage weight \(w_k\), maximizes estimated profit under
\(\widehat\beta_k^{\mathrm{shr}}(w_k)\),
\begin{equation}
\widehat p_k(w_k)\in\argmax_{p\in\mathcal P}\Pi_k(p;\widehat\beta_k^{\mathrm{shr}}(w_k)).
\label{eq:pricing_plugin_vector}
\end{equation}
The resulting pricing regret is
\begin{equation}
\mathrm{Regret}_{k}^{\mathrm{price}}(w_k)
=
\Pi_k(p_k^\star;\beta_k^\star)-\Pi_k(\widehat p_k(w_k);\beta_k^\star).
\label{eq:pricing_regret_vector}
\end{equation}
As in the assortment case, this regret is a random variable through its
dependence on the sample. The following proposition bounds the regret with high probability over the random sample. 

\begin{proposition}
\label{prop:pricing_regret_hp}
Fix a store \(k\), assortment
\(S\subseteq\mathcal N\), and shrinkage weight \(w_k\in[0,1]\), where $\widehat\beta_k^{\mathrm{shr}}(w_k) := (1-w_k)\widehat\beta_k+w_k\widehat\beta_0^{(-k)}$. Suppose Assumptions~\ref{ass:bounded}--\ref{ass:sample} hold and
$\Delta\in[0,\Delta_0]$, where $\Delta_0$ is defined in
Appendix~\ref{sec:preliminaries}. Let $N_{-k}:=\sum_{\ell\neq k}n_\ell$ denote the leave-one-out pooled sample size. Then, for any 
\(\delta\in(0,1)\), 
\begin{equation}
\mathrm{Regret}_{k}^{\mathrm{price}}(w_k)
\le
C_p
\Big[
(1-w_k)\,r_{\mathrm{loc}}(\delta)
+
w_k \bigl(
r_{\mathrm{pool}}(\delta)
+
\Delta_k
\bigr)
\Big],
\label{eq:pricing_regret_clean_cont}
\end{equation}
where $\Delta_k = \|\beta_0^{(-k),\star} - \beta_k^\star\|_2$, 
\(C_p>0\) depends only on $S$ and the primitive regularity quantities, and
\[
r_{\mathrm{loc}}(\delta)
=
\sqrt{\frac{d+\log(1/\delta)}{n_k}},
\qquad
r_{\mathrm{pool}}(\delta)
=
\sqrt{\frac{d+\log(1/\delta)}{N_{-k}}}.
\]
\end{proposition}

The pricing bound has the same three parts as the assortment bound, the local
error, the pooled error, and the mismatch \(\Delta_k\). What differs is the
downstream decision. In assortment, estimation error can change which products
are selected. In pricing, it distorts the estimated nonprice valuations and
price sensitivity that set the profit-maximizing prices.

Propositions~\ref{prop:assortment_regret_hp} and \ref{prop:pricing_regret_hp}
are both stated for a generic shrinkage weight so, in principle, operational objectives can motivate
task-specific weights. For example, a weight chosen to minimize an
assortment-regret bound need not coincide with one chosen to minimize a
pricing-regret bound. Our focus is not to solve that task-specific
tuning problem. The propositions instead show how the same local-versus-pooled tradeoff
propagates to two canonical revenue-management decisions.

\section{Simulation Study}
\label{sec:simulation}

The theory predicts that the operational value of shrinkage, measured by
assortment and pricing regret, is governed by the interaction among local
sample size and cross-store heterogeneity. Shrinkage is most
valuable when local data are limited yet stores remain similar enough that
borrowing from the rest of the network is informative. The main simulation study therefore varies two quantities, cross-store
heterogeneity \(\Delta\) and local sample size \(n\), each in a separate
experiment. Robustness checks varying the number of stores \(K\), the catalog size \(N\),
and the parameter dimension \(d\) are reported in
Appendix~\ref{app:computational_details}.

\subsection{Experimental Design}

Each experiment varies one of \(\Delta\) and \(n\) and holds every other
parameter fixed, including the number of stores \(K=10\), the catalog size
\(N=15\), and the parameter dimension \(d=7\), with five nonprice features, one price-sensitivity coefficient, and one intercept. For each estimator it
measures assortment and pricing regret, comparing the local MLE, full pooling,
and the shrinkage estimators between them, including the order-based plug-in and
cross-validated weights. The model instance (i.e., product features and store MNL parameters) and
choice data are drawn at random, and we report expected regret, averaging over 200
independent Monte Carlo replications at each value of \(\Delta\) or \(n\). 

We next describe how we generate each model instance. To generate the product features, we independently draw nonprice feature vectors \(Z_i\sim\mathcal N(0,I_{d-2})\) and prices \(p_i\sim\mathrm{Unif}[1,10]\), thus constructing the feature vector \(X_i=(1,Z_i^\top,-p_i)^\top\). %
The store MNL parameters are generated as a random draw from a ball of radius $\Delta$ centered on $\beta_c^\star \in\mathbb R^d$, i.e., 
\(\beta_k^\star=\beta_c^\star+\nu_k\) with \(\nu_k\in\mathbb R^d\) denoting the random perturbation with \(\|\nu_k\|_2=\Delta\). Here, the common center $\beta_c^\star$ itself is randomly generated: we draw the center's \(d-2\) nonprice
coefficients independently from \(\mathcal N(0,1)\), set its intercept
to \(5\) and its price-sensitivity coefficient to \(1\), and hold the
resulting $\beta_c^\star$ fixed across stores. 

The heterogeneity experiment fixes \(n=50\) and
varies \(\Delta\) over 13 values from \(0\) to \(2.5\), spaced more finely at
small \(\Delta\). The sample-size experiment fixes \(\Delta=1\) and varies \(n\)
over 8 values from \(30\) to \(500\). The exact grids and the full data-generating process  are given in Appendix~\ref{app:computational_details}. %

We compare five estimators of the MNL parameters. The three
shrinkage estimators share the form
\(\widehat\beta_k^{\mathrm{shr}}(w_k)=(1-w_k)\widehat\beta_k+w_k\widehat\beta_0^{(-k)}\)
with \(w_k\in[0,1]\), where \(\widehat\beta_0^{(-k)}\) is the leave-one-out pooled
target in \eqref{eq:sec4_loo_mle}. They differ only in how the shrinkage weight \(w_k\) is chosen. 
\begin{itemize}
  \item \textbf{Store MLE}, \(\widehat\beta_k\) in \eqref{eq:sec4_local_mle}: uses
  only store-\(k\) data.
  \item \textbf{Pooled MLE}, \(\widehat\beta_0\) in \eqref{eq:sec4_pooled_mle}: 
  uses the data pooled across all stores.
  \item \textbf{Shrinkage (Oracle)}: an infeasible, realization-specific
analogue of the damped weight \(w_k^\star\) in
\eqref{eq:theory_analytic_weight}. For each store and model instance, it uses
the true parameter to compute the realized local and leave-one-out
pooled estimation errors, and applies the same damping rule to their
Euclidean quadratic terms. Appendix~\ref{app:oracle_shrinkage} gives the
calculation and explains its relation to the damped weight. This
benchmark is labeled Shrinkage (\(w^\star\)) in the simulation figures.
  \item \textbf{Shrinkage (Plug-in)},
\(\widehat w_k(\widehat\lambda)\): the order-based rule in
\eqref{eq:theory_plugin_weight} that uses the empirical mismatch
\(\widehat\Delta_k\). We calibrate \(\lambda\)
using a pooled-Hessian estimate of local estimation variance, as
described in Appendix~\ref{app:plugin_calibration}.
  \item \textbf{Shrinkage (CV)}: the weight selected separately for each
store by five-fold cross-validation within its \(n_k\) observations.
Each fold uses 80\% of the local observations for fitting and 20\% for
validation. On the grid
\(\mathcal W=\{0,0.01,\ldots,1\}\), we select
\[
\widehat w_k^{\mathrm{CV}}
\in
\argmin_{u\in\mathcal W}
\sum_{f=1}^{5}L_{k,f}^{\mathrm{val}}\!\left(
(1-u)\widehat\beta_k^{\mathrm{tr},f}
+
u\widehat\beta_0^{(-k)}
\right),
\]
where \(\widehat\beta_k^{\mathrm{tr},f}\) is fitted on the four folds
excluding fold \(f\), and \(L_{k,f}^{\mathrm{val}}\) is the negative
log-likelihood on the held-out fold. The leave-one-out pooled target
\(\widehat\beta_0^{(-k)}\) uses all training observations from the other
stores and is held fixed across folds. The final shrinkage estimate
combines \(\widehat\beta_k\), fitted on all \(n_k\) observations, with
\(\widehat\beta_0^{(-k)}\) using the selected weight.
\end{itemize}
For the store and pooled MLE, likelihood is minimized over the same
parameter region \(\Theta\) of Assumption~\ref{ass:parameter}. Convexity of
\(\Theta\) keeps each shrinkage estimator in \(\Theta\).

\begin{figure}[t]
\begin{center}
\begin{tabular}{cc}
\includegraphics[width=0.47\textwidth]{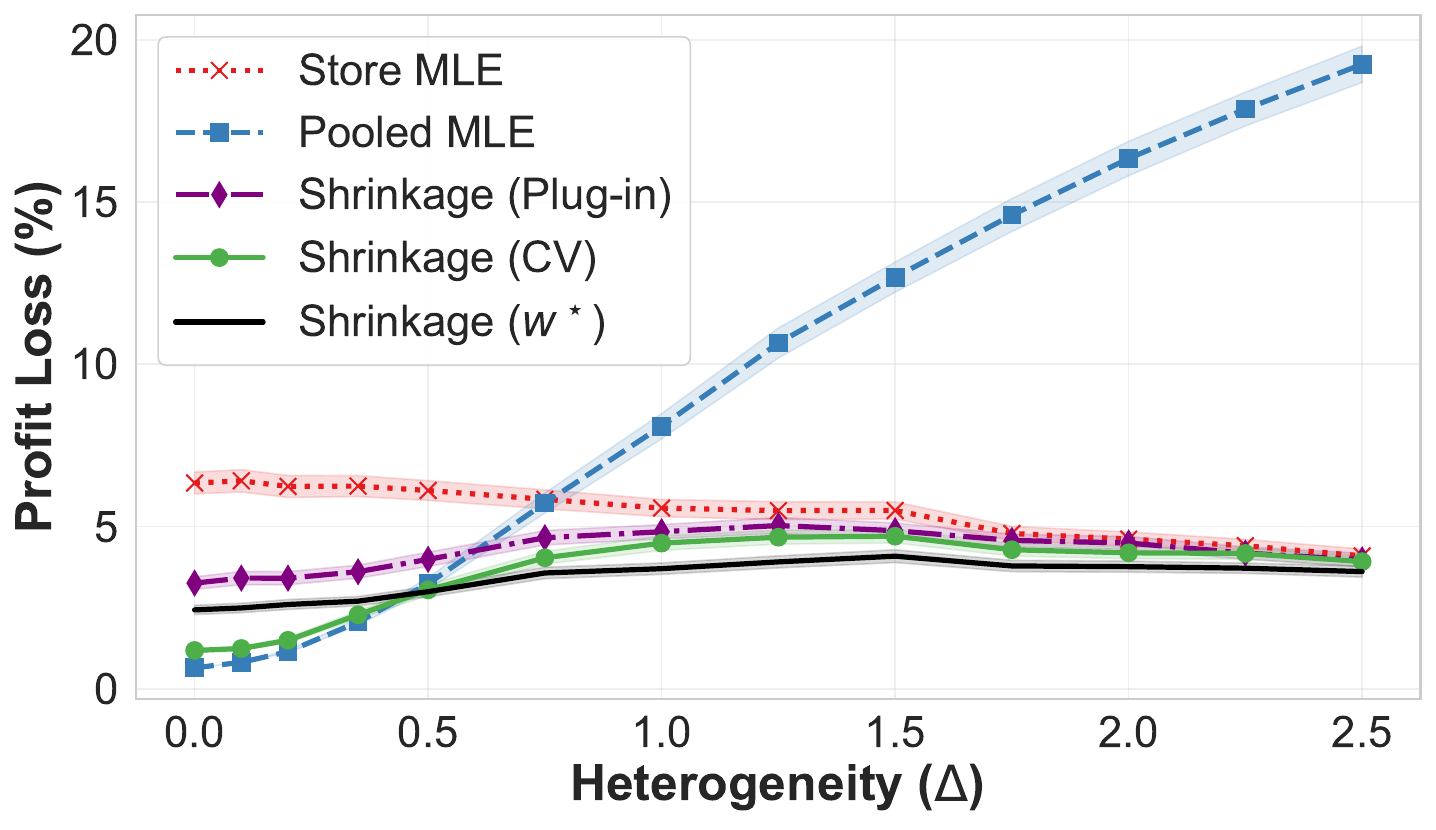}&
\includegraphics[width=0.47\textwidth]{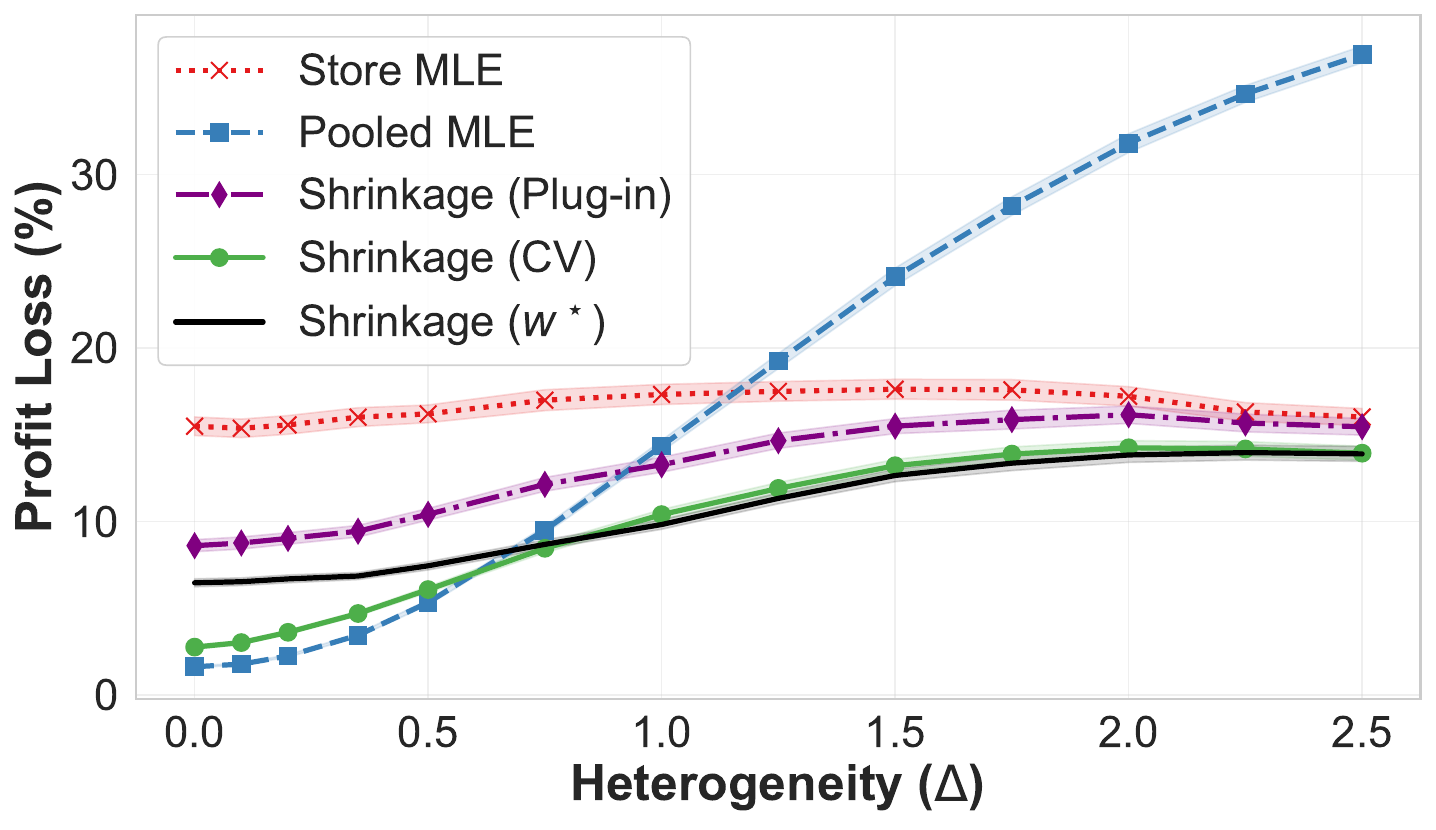}
\end{tabular}
\caption{Percentage Regret for assortment (left) and pricing (right) decisions as cross-store heterogeneity $\Delta$ varies.} \label{fig:exp1_regret}
\end{center}
\end{figure}

\subsection{Varying Cross-Store Heterogeneity}

Figure~\ref{fig:exp1_regret} reports assortment and pricing regret as functions
of the heterogeneity radius \(\Delta\), and the pattern matches the theory.  When \(\Delta\) is small, heterogeneity bias is small for every method, so the differences in
regret are driven by estimation variance. The fully pooled estimator therefore
substantially outperforms the local estimator, because it draws on the far
larger pooled sample and so carries the least variance. As \(\Delta\) grows, the heterogeneity bias takes over. The pooled estimator fits
a demand model increasingly misaligned with the store-specific truth, so its
variance advantage is outweighed and its regret rises steadily.

The local estimator behaves differently. It carries no heterogeneity bias, so its
regret does not climb with \(\Delta\) the way the pooled curve does. 
The two curves therefore cross at an intermediate value of \(\Delta\), the point past which borrowing from the pooled data stops paying off.

The shrinkage estimators interpolate between the two extremes. The oracle
benchmark shows the value of a well-chosen weight, staying close to the lower
envelope of the local and pooled regret curves and improving on both in between. Note that it does
not beat every method at every \(\Delta\), since its damped weight is based
on realized parameter errors rather than chosen to minimize assortment or
pricing regret. The order-based plug-in weight is derived from the
population weight \(w_k^\star\), and its regret follows the same pattern,
albeit at a slightly higher regret level.

The cross-validated weight differs from the oracle benchmark in the
information it uses. The oracle benchmark applies the damped rule to
realized quadratic parameter errors using the known truth, whereas the CV
weight works with the likelihood directly, using held-out validation data.
Maximizing the validation likelihood pushes the weighted
estimator toward the true store-\(k\) parameter. Since that estimator blends the
high-variance local estimate with the low-variance but heterogeneity-biased pooled
estimate, moving closer means trading the variance pooling removes against the bias
it adds. When \(\Delta\) is small that bias is small, so the criterion favors heavy
borrowing and the CV weight is large, and its regret tracks the pooled estimator.
As \(\Delta\) grows, the bias from pooling degrades predictive fit, the CV weight
falls toward the local estimator, and its regret follows the local curve. This
adjustment keeps it near the better of the two extremes across \(\Delta\).

\begin{figure}[t]
\begin{center}
\begin{tabular}{cc}
\includegraphics[width=0.47\textwidth]{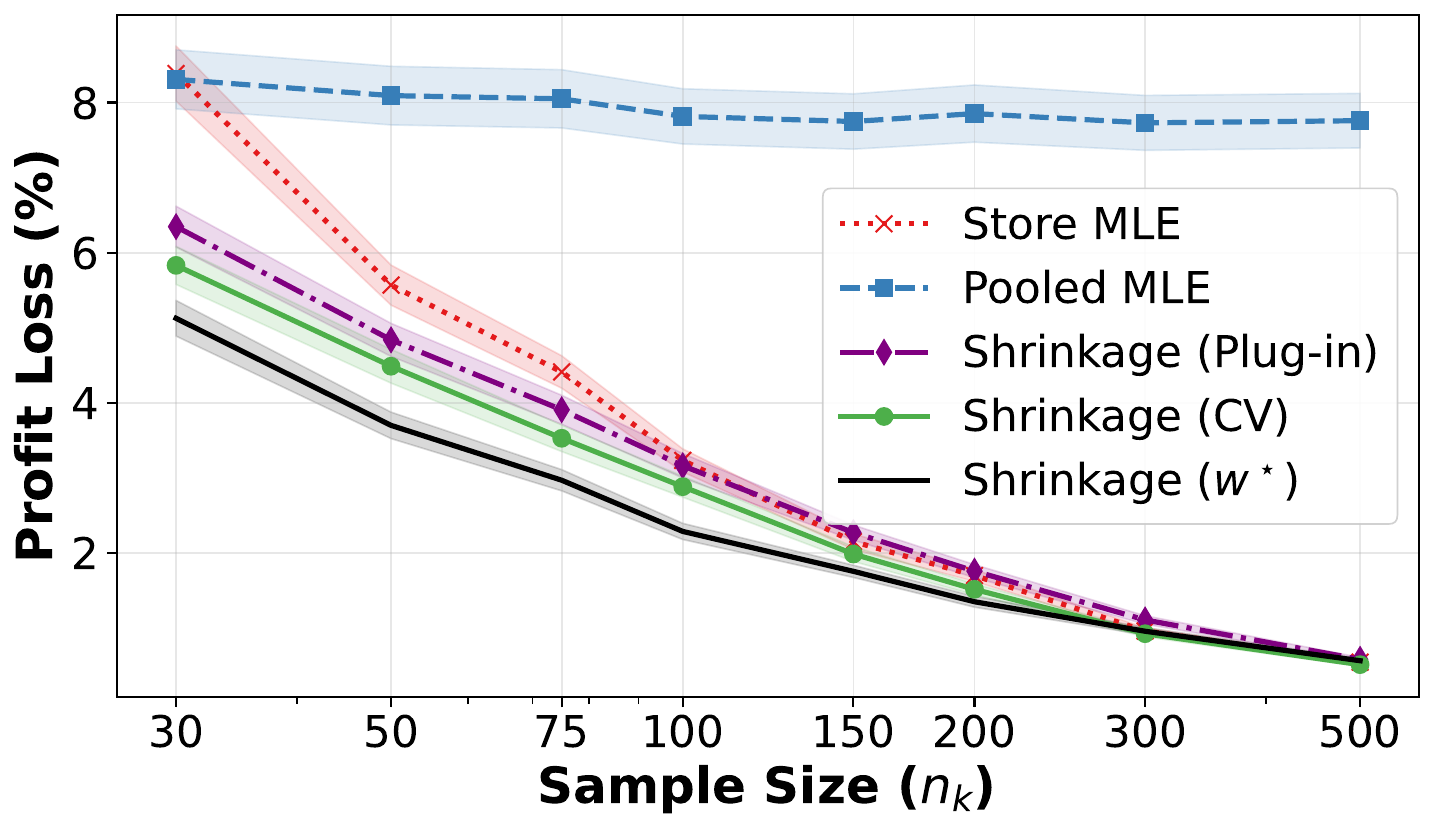}&
\includegraphics[width=0.47\textwidth]{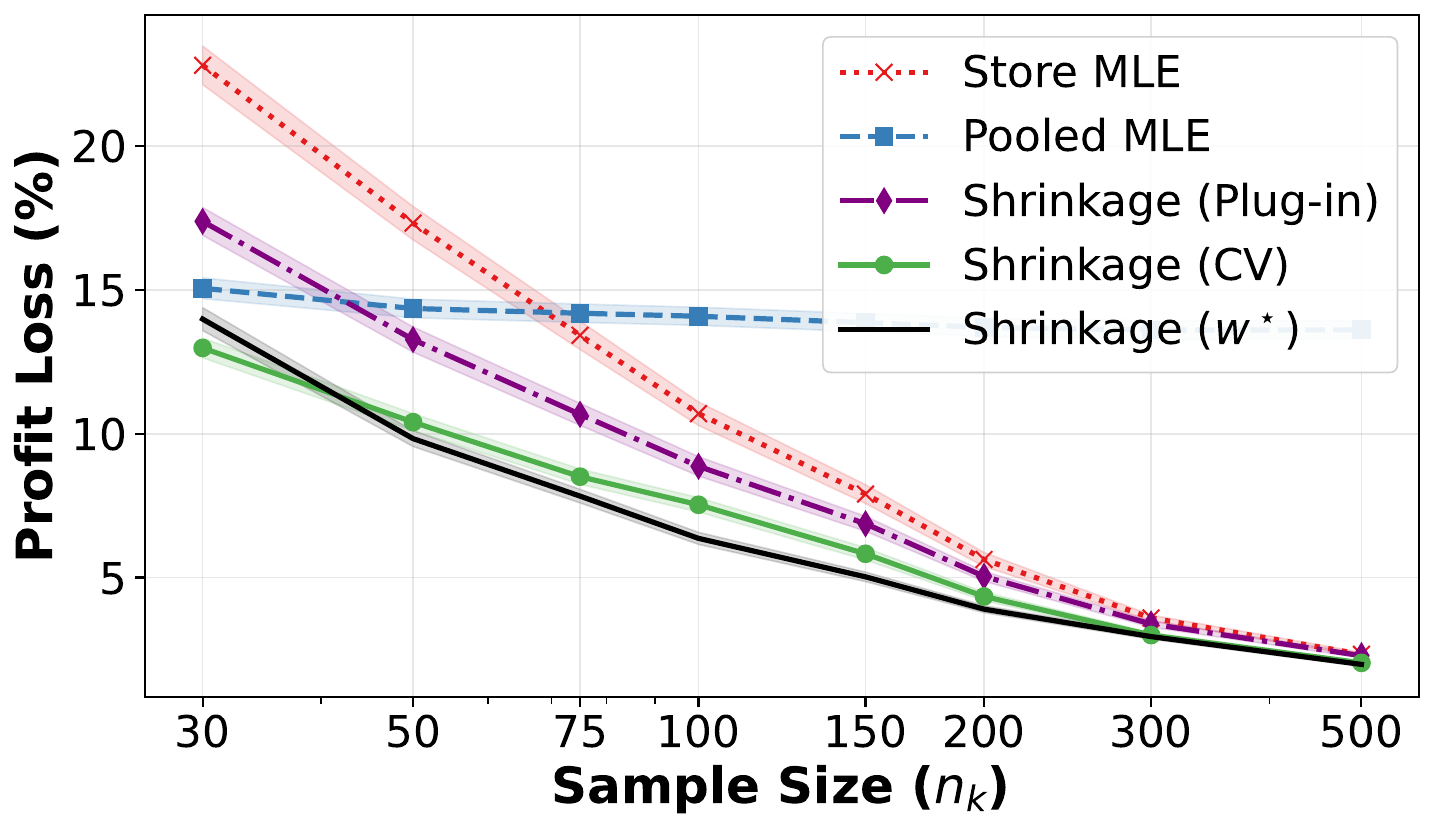}
\end{tabular}
\caption{Percentage regret for assortment (left) and pricing (right)
decisions as the local sample size \(n_k\) varies.}\label{fig:exp2_regret}
\end{center}
\end{figure}

\subsection{Varying Local Sample Size}

Figure~\ref{fig:exp2_regret} reports regret for the assortment and pricing decisions as the local sample size \(n_k\) varies, fixing $\Delta = 1$. At this value of $\Delta$, the stores are far from homogeneous, so full pooling carries a
nontrivial bias even when local samples are thin.

The local estimator improves sharply as \(n_k\) increases. When $n_k$ is small, the local estimator is noisy and can induce high
operational regret. As $n_k$ grows, its variance falls, and the resulting assortment and pricing decisions become more reliable. This is
the cleanest comparative static in the figure and directly reflects the
finite-sample scaling in theory. 

The pooled estimator behaves differently. A larger \(n_k\)
gives the pooled model more data, but it
does not remove the structural mismatch created by treating
heterogeneous stores as homogeneous. Because this mismatch, rather than variance,
is the main source of its error at \(\Delta=1\), the pooled curve stays nearly
flat. In the assortment panel it lies above the local curve for all but the thinnest local samples, so pooling does worse than the store-specific fit despite drawing on more data.

The shrinkage estimators combine the local and pooled estimates. In
Figure~\ref{fig:exp2_regret}, the oracle benchmark lies below both
estimators' regret curves throughout, while the plug-in and cross-validated
estimators stay close to it. This gain over either estimator
comes from their complementary errors. The local estimator is unbiased but noisy and the pooled estimator is
stable but biased, and shrinkage keeps most of the variance reduction
while taking on only part of the bias. As \(n_k\) grows and local estimation sharpens, the oracle weight slides from heavy borrowing back toward the local estimator. 

\section{Empirical Study}
\label{sec:empirical}

This section uses observational data provided by a large European multi-store
apparel retailer. Without a ground-truth demand model, we cannot measure pricing or assortment regret
against a known optimum, so we instead judge each method by how well it performs on a prediction task.

The task is to predict the rank of products by revenue contribution. Such a ranking is a routine input to operational decisions such as ABC classification,
where the highest-contribution items receive more inventory attention, better shelf space,
and more managerial effort.  

We split the weeks into training and test periods. Fix a store $k$. Each method's store-level estimate \(\widehat\beta_k\) is fit on the training weeks. In each test week $t$, among the set \(\mathcal A_{kt}\) of products store~$k$ actually stocked, the estimator ranks the products by predicted expected revenue
\(p_i P_i(\mathcal A_{kt}; \widehat{\beta}_k)\), then selects the set \({S}_{kt}(\widehat{\beta}_k)\) of the top
\(m\%\) of the products. We measure the revenue capture,
\begin{align}
        \text{Capture}_{kt}(\widehat{\beta}_k; \mathcal A_{kt}) = \frac{\sum_{i \in {S}_{kt}(\widehat{\beta}_k)} r_{ikt}}{\sum_{i \in \mathcal A_{kt}} r_{ikt}}, \label{eqn:rev_capture}
\end{align}
where \(r_{ikt}\) is the observed revenue of product \(i\) at store $k$ in evaluation week $t$. The metric scores the selection against what
actually sold, so it does not rely on a ground-truth demand model. This metric is computed for each store and each test week.

\subsection{Data and Empirical Setting}

The dataset covers several
hundred stores of the retailer over 53 weeks. Products in the dataset fall into five
women's apparel categories (blazers, dresses, leggings, trousers, and T-shirts). For each product, store, and week, the data logs whether the product was stocked, the
number of units sold that week, and the realized revenue. 

We retain stocked products with zero sales, which is important because a
zero-sale stocked product is informative about demand. In this
section, a \emph{product} means a distinct item identifier in the
retailer's system, while \emph{units sold} refers to the physical number
of units purchased. Thus, product variety and sales volume are separate
objects: a category can contain many distinct products even if only a
small number of units are sold in a given store-week. %

\begin{table}[t]
\centering
\caption{Dataset dimensions by product category.}
\label{tab:empirical_data_dimensions}
\begin{tabular}{lrrrr}
\toprule
Category & \# Stores & \# Products & Total Observations & Total Units Sold \\
\midrule
Blazers  & 365 &   790 &   622{,}642 &   173{,}471 \\
Dresses  & 371 & 4{,}182 & 1{,}669{,}206 &   781{,}011 \\
Leggings & 361 &   255 &   193{,}958 &    79{,}181 \\
Trousers & 372 & 3{,}394 & 1{,}552{,}620 &   614{,}147 \\
T-shirts & 375 & 3{,}984 & 2{,}783{,}502 & 1{,}894{,}169 \\
\bottomrule
\end{tabular}
\end{table}

Table~\ref{tab:empirical_data_dimensions} reports the scale of the
data by category. Across the five categories, the data contain 6.82 million
product-store-week observations and 3.54 million units sold over the
sample period. The categories differ substantially in both product
variety and sales volume. Leggings have a relatively small product
variety, while dresses, trousers, and T-shirts involve thousands of
distinct products. This variation lets us examine whether the value of
shrinkage persists across categories with different degrees of assortment
complexity.

\begin{table}[t]
\centering
\caption{Average store- and store-week-level scale by category.}
\label{tab:empirical_data_characteristics}
\begin{tabular}{lccc}
\toprule
Category
& \makecell{Avg.\ \# products\\ per store}
& \makecell{Avg.\ assortment size\\ per store-week}
& \makecell{Avg.\ units sold\\ per store-week} \\ 
\midrule
Blazers  & 178.7 &   9.7 &  4.7 \\
Dresses  & 493.9 &  39.2 & 24.9 \\
Leggings &  48.8 &   5.5 &  3.8 \\
Trousers & 447.5 &  31.4 & 18.6 \\
T-shirts & 754.5 & 101.9 & 69.6 \\
\bottomrule
\end{tabular}
\end{table}

Table~\ref{tab:empirical_data_characteristics} reports the
store-level and store-week-level scale by category. The local estimation problem is thin not because
the retailer is small, but because sales are fragmented across many
stores, weeks, and products. For a given category, a typical store-week contains many available products but only a limited number of realized purchases from which to infer substitution patterns. This is precisely the setting in which
separate store-level estimation can become unstable, while full pooling
can wash out meaningful local demand differences.

\subsection{Estimators, Rolling Design, and Decision-Aligned Cross-Validation}

Fix the category. For each store, we compare three estimators (store MLE, pooled MLE, cross-validated shrinkage) on the ranking task. Each estimate the store's \textit{conditional} MNL choice model for products in the category. Specifically, we condition on a
category purchase taking place at the store and model which product is chosen. This
captures substitution among the available products rather than total category
conversion. Hence, we avoid imposing a separate outside-option model for total category
traffic. 

The evaluation for each store uses a rolling out-of-sample design. A window of 8
training weeks, 2 validation weeks, and 2 test weeks is advanced through the panel. All three estimators are fit on the 8 training weeks,
so no method sees more data than another, and performance is measured on the 2 test
weeks. Because training always precedes the weeks a ranking is scored on, every
ranking uses only information available beforehand. For each validation or test week, the three methods rank a common candidate set \(\mathcal
A^c_{kt}\), the category-$c$ products stocked in store $k$ that week $t$ that were also observed in the window's
training weeks. Predicted probabilities are renormalized over \(\mathcal A^c_{kt}\),
so the methods differ only in how they rank the same products, not in which products
they consider.

Cross-validated shrinkage selects its weight on the validation weeks, using the
revenue-capture metric \eqref{eqn:rev_capture} rather than a likelihood
criterion. For each store $k$ and category $c$, the selected weight maximizes total validation
revenue capture,
\[
\widehat u^c_{k}
\in
\argmax_{u\in\mathcal W}
\sum_{t\in\mathcal T^{\mathrm{val}}}
\mathrm{Capture}_{kt}(\widehat{\beta}_{k}^{\mathrm{shr,c}}(u); \mathcal A^c_{kt}),
\]
where \(\mathrm{Capture}_{t}(\cdot)\) is the revenue capture metric \eqref{eqn:rev_capture} computed on validation week \(t\), and  \(\widehat{\beta}^{\mathrm{shr,c}}_{k}(u)\) is the  store $k$ shrinkage estimator for category $c$ with weight $u$. The selected weight is then applied in
the test weeks.

\begin{figure}[t]
\begin{center}
\includegraphics[width=\textwidth]{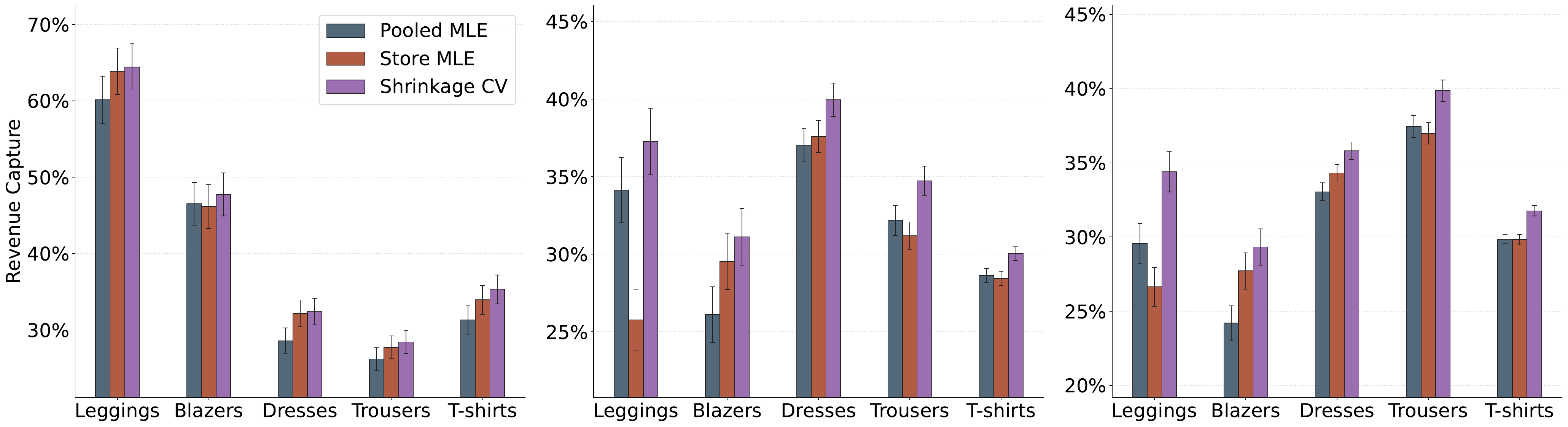} 
\caption{Out-of-sample revenue capture across the five apparel categories, evaluated at a \(m=20\%\) budget constraint. Store-level predictive performance is disaggregated by historical sales volume into Lower Volume stores (left), Medium Volume stores (center), and Higher Volume stores (right). Error bars represent standard errors.}
\label{fig:empirical_capture}
\end{center}
\end{figure}

\subsection{Revenue Capture Results}

We compare the estimators as the amount of local training data varies. Within each category, we rank
the test-period store-weeks by their recent training volume (i.e., the units sold in the 8
weeks before the test week), then we split them into terciles: a data-scarce third, a middle third, and a data-rich third. Figure~\ref{fig:empirical_capture}
reports mean revenue capture at \(m=20\%\) for each category and training volume tercile, data-scarce
on the left and data-rich on the right. 
Within a category and store-week the estimators rank the candidate products. We compare
their revenue capture within a category-tercile cell.

The results show a consistent advantage for adaptive partial pooling.
Shrinkage CV achieves the highest revenue capture in every reported
category-tercile cell. Averaging equally across the fifteen
category-tercile cells, Shrinkage CV captures \(36.84\%\) of realized
revenue, compared with \(34.13\%\) for Store MLE and \(33.67\%\) for
Pooled MLE. This corresponds to an average lift of \(2.71\) percentage
points over Store MLE and \(3.18\) percentage points over Pooled MLE.

The training volume level patterns are also informative. In the low-volume tercile, Shrinkage CV captures \(41.68\%\) of realized revenue on
average across categories, compared with \(40.79\%\) for Store MLE and
\(38.55\%\) for Pooled MLE. The gain over Store MLE is modest in this
group, \(0.89\) percentage points on average, because the store-specific
benchmark is already strong in several low-volume cells. However,
Shrinkage CV still improves over Pooled MLE by \(3.12\) percentage
points, showing that full pooling loses important local structure even
when local information is limited.

The gains over Store MLE become larger in the medium- and high-volume
terciles. In the medium-volume tercile, Shrinkage CV captures
\(34.62\%\) of realized revenue, compared with \(30.51\%\) for Store MLE
and \(31.62\%\) for Pooled MLE. In the high-volume tercile, the
corresponding values are \(34.23\%\), \(31.09\%\), and \(30.83\%\).
Thus, the average lift over Store MLE is \(4.11\) percentage points in
the medium-volume group and \(3.14\) percentage points in the
high-volume group. The average lift over Pooled MLE is similarly
stable: \(3.00\) percentage points in the medium-volume group and
\(3.41\) percentage points in the high-volume group.

The category-level comparisons show the same bias--variance tradeoff.
Pooled MLE is sometimes competitive, but it is too rigid when local
demand patterns differ across stores. Store MLE preserves local
heterogeneity, but it can be unstable when the local ranking problem is
thin relative to the number of products. Shrinkage CV provides a more
reliable middle ground. Across the fifteen category-tercile cells, its
lift over Store MLE ranges from \(0.26\) percentage points to
\(11.49\) percentage points, while its lift over Pooled MLE ranges from
\(1.21\) percentage points to \(5.13\) percentage points. The largest
gain over Store MLE occurs for leggings in the medium-volume tercile,
where revenue capture increases from \(25.78\%\) under Store MLE to
\(37.27\%\) under Shrinkage CV.

The empirical message is therefore stronger than a simple comparison of
average performance. Shrinkage CV is not merely mimicking either
benchmark. It improves on Pooled MLE in cells where full pooling washes
out local structure, and it improves on Store MLE in cells where
separate estimation is too noisy for reliable product prioritization.
The result is a consistent improvement in the featured-assortment
ranking problem across categories and across the local-information
distribution.

\subsection{Statistical Reliability of the Gains}

Figure~\ref{fig:empirical_capture} showed where Shrinkage CV captures more revenue
than the benchmarks, but not whether that gain is more than sampling noise. We now examine this. 

Fix the category~$c$. For each store $k$ and test week $t$, define the lift over Store MLE as
\[
\Delta^{\mathrm{store},c}_{kt}
=
\mathrm{Capture}_{kt}(\widehat{\beta}^{\mathrm{shr}, c}_k; \mathcal A^c_{kt})
-
\mathrm{Capture}_{kt}(\widehat{\beta}^c_k; \mathcal A^c_{kt}),
\]
and the lift over Pooled MLE as
\[
\Delta^{\mathrm{pool},c}_{kt}
=
\mathrm{Capture}_{kt}(\widehat{\beta}^{\mathrm{shr}, c}_k; \mathcal A^c_{kt})
-
\mathrm{Capture}_{kt}(\widehat{\beta}^{c}_0; \mathcal A^c_{kt}).
\]
Because each store-week is scored under all three estimators, we measure the gain
store-week by store-week and compute the confidence intervals and win rates on
these within-store-week differences. Differencing within a store-week removes its
idiosyncratic difficulty, which gives a sharper comparison than contrasting the
methods' separate averages.

Confidence intervals come from a clustered bootstrap over stores. For each row of
Table~\ref{tab:paired_lift}, we resample the stores in that data slice and keep all of a
store's store-weeks together, which accounts for within-store correlation across weeks. The win rate is the fraction of store-weeks in which Shrinkage CV captures more revenue than the corresponding benchmark. 

\begin{table}[h!]
\centering
\resizebox{\textwidth}{!}{
\begin{tabular}{l|ccc|ccc}
\toprule
 & \multicolumn{3}{c|}{\textbf{$\Delta^{\mathrm{store}}$ vs. Store MLE}} 
 & \multicolumn{3}{c}{\textbf{$\Delta^{\mathrm{pool}}$ vs. Pooled MLE}} \\
\textbf{Group} 
& \textbf{Mean lift (pp)} & \textbf{95\% CI} & \textbf{Win rate} 
& \textbf{Mean lift (pp)} & \textbf{95\% CI} & \textbf{Win rate} \\
\midrule
Overall 
& 2.64 & [2.17, 3.10] & 86.6\% 
& 3.11 & [2.55, 3.68] & 83.6\% \\
\midrule
Category: Blazers  
& 1.59 & [-0.04, 3.20] & 91.0\% 
& 4.21 & [2.85, 5.58] & 93.2\% \\
Category: Dresses  
& 1.44 & [1.00, 1.95] & 85.4\% 
& 3.13 & [1.98, 4.23] & 80.7\% \\
Category: Leggings 
& 7.53 & [5.69, 9.66] & 93.9\% 
& 4.15 & [2.46, 5.81] & 94.3\% \\
Category: Trousers 
& 2.48 & [1.66, 3.27] & 85.3\% 
& 2.41 & [1.37, 3.39] & 82.2\% \\
Category: T-shirts 
& 1.65 & [1.25, 2.06] & 80.9\% 
& 2.30 & [1.34, 3.31] & 73.7\% \\
\midrule
Bucket: Low-volume stores    
& 0.85 & [-0.24, 1.90] & 94.0\% 
& 3.13 & [1.60, 4.64] & 90.4\% \\
Bucket: Medium-volume stores 
& 3.63 & [2.76, 4.50] & 86.4\% 
& 2.86 & [2.02, 3.73] & 83.5\% \\
Bucket: High-volume stores   
& 2.98 & [2.48, 3.46] & 81.8\% 
& 3.31 & [2.67, 3.96] & 79.1\% \\
\bottomrule
\end{tabular}
}
\caption{Mean lift from Shrinkage CV at \(m=20\%\). Lifts are reported in percentage points. Confidence
intervals are clustered-bootstrap intervals resampling stores, and win
rates report the fraction of store-weeks in which Shrinkage CV
outperforms the corresponding benchmark.}
\label{tab:paired_lift}
\end{table}

Table~\ref{tab:paired_lift} shows that the gain from Shrinkage CV is
consistent. Overall, Shrinkage CV improves revenue capture by \(2.64\)
percentage points relative to Store MLE, with a clustered-bootstrap
\(95\%\) confidence interval of \([2.17, 3.10]\), and wins in \(86.6\%\)
of store-weeks. Relative to Pooled MLE, the average lift is \(3.11\)
percentage points, with confidence interval \([2.55, 3.68]\), and
Shrinkage CV wins in \(83.6\%\) of store-weeks. Both overall intervals
sit well above zero, so the advantage seen in
Figure~\ref{fig:empirical_capture} is more than sampling noise.

The category-level results reinforce the same conclusion. The lift over Store MLE is
positive in all five categories, largest for leggings at \(7.53\)
percentage points, though for blazers the confidence interval still
includes zero. The lift over Pooled MLE is positive and clear of zero in
every category, ranging from \(2.30\) percentage points for T-shirts to
\(4.21\) for blazers. %

The volume terciles show where the gain comes from. In the low-volume
tercile, the lift over Store MLE averages only \(0.85\) percentage points
and its confidence interval includes zero, yet Shrinkage CV still wins in
\(94.0\%\) of store-weeks.  Over the same cells the lift over Pooled MLE is
much larger, \(3.13\) percentage points with confidence interval
\([1.60, 4.64]\). Where local data are scarcest, the fitted ranking stays
close to Store MLE rather than pooling, and the clear gain is the one over
the pooled benchmark. In the medium- and high-volume terciles the lift over
Store MLE is larger and distinguishable from zero, \(3.63\) and \(2.98\)
percentage points, while the lift over Pooled MLE stays stable across
terciles, from \(2.86\) to \(3.31\) percentage points.

These findings caution against reading the empirical pattern as a simple monotone one in which the lowest-volume store-weeks gain the most.   The gain over Store MLE is in fact smallest there.  What governs the gain is how much local information a cell carries relative to the difficulty of its ranking problem, not local sample size alone. Low-volume store-weeks have fewer observations,
but they often rank fewer products, so Store MLE already ranks them well and leaves little room to improve. Higher-volume store-weeks carry more
observations but often broader assortments and more complex substitution,
so the information available for ranking can stay thin even outside the
lowest-volume cells.

The empirical conclusion is therefore clear. Shrinkage CV produces reliable out-of-sample improvements over both local estimation and pooled estimation. The
paired lifts show these improvements to be consistent across store-weeks,
categories, and volume terciles.

\section{Conclusion}
\label{sec:conclusion}

Large retail chains often have extensive sales data across their networks but limited information for estimating demand at individual stores. This creates a fundamental tradeoff: estimating demand separately for each store can be unreliable, while pooling data across stores can obscure meaningful differences in customer preferences. We study how transfer learning can balance these competing effects under a multinomial logit (MNL) demand model, using shrinkage estimation to selectively borrow information across stores.

Our main methodological contribution is to extend the finite-sample analysis of information transfer to nonlinear structural choice models. By exploiting the self-concordant-like properties of the MNL likelihood, we establish performance guarantees that characterize when information transfer improves estimation accuracy and how the appropriate degree of shrinkage depends on local sample size, model dimensionality, and cross-store heterogeneity. We further derive regret bounds for assortment and pricing decisions. Numerical experiments and an empirical study using data from a large apparel retailer demonstrate the operational benefits of shrinkage relative to separate estimation and full pooling.

Our findings highlight the value of information transfer for revenue management in heterogeneous markets. Rather than relying exclusively on local data or imposing a common demand model across stores, retailers can improve their decisions by selectively borrowing information from related markets. More broadly, our analysis provides a foundation for extending transfer learning to richer nonlinear choice models and more complex revenue management settings.

\bibliographystyle{ormsv080} %
\bibliography{references} 

\ECSwitch

\section{Auxiliary Geometric Results for the MNL Likelihood}
\label{app:geometry}

This appendix collects the auxiliary geometric results used in the main text.  We keep the notation of Sections~\ref{sec:multistore_mnl}--\ref{sec:theory}.  Throughout, the constants in the bounds below may change from line to line, but depend only on the primitive quantities identified in the main text.

\subsection{Preliminaries}
\label{sec:preliminaries}

Subscripted norms follow a single convention: for a symmetric positive definite matrix
$A\succ0$,
\[
\|v\|_{A}:=\sqrt{v^\top A v},
\qquad
\langle u,v\rangle_{A}:=u^\top A v,
\]
so a matrix subscript names the matrix inducing the norm. The exceptions are the Euclidean
norm $\|\cdot\|_2$, the operator norm $\|\cdot\|_{\mathrm{op}}$, and the sub-Gaussian norm
$\|\cdot\|_{\psi_2}$, where the subscript names the norm itself. In this notation, the local norm and inner product of Section~\ref{sec:quadratic_appx} is $\|\cdot\|_k = \| \cdot \|_{H_k^\star}$ and $\langle \cdot, \cdot \rangle_k = \langle \cdot, \cdot \rangle_{H_k^\star}$, with $H_k^\star \succ 0$ by Assumption~\ref{ass:fisher}. 

For a parameter $\beta$, define the first two moments of $X^\beta$ as the following functions of $\beta$:
\begin{align} 
\mu_X(\beta) &:= \mathbb{E}\bigl[X^\beta\bigr] = \sum_{i=0}^N P_i(\beta)
X_i, \label{eqn:x_mu}\\
\Sigma_X(\beta) &:= \Cov\bigl(X^\beta\bigr) = \sum_{i=0}^N P_i(\beta)
\bigl(X_i- \mu_X(\beta)\bigr)\bigl(X_i- \mu_X(\beta)\bigr)^\top. \label{eqn:x_sigma}
\end{align}
Using this notation, note that $\Sigma_X^k = \Sigma_X(\beta_k^\star)$.

Several bounds below require the pooled population targets to retain an
interior margin. Set
\begin{equation}
C_{\mathrm{het}}
:=
\frac{2BM}{m}\,e^{2BD_\Theta},
\qquad
R_\Theta(\delta)
:=
2\delta+C_{\mathrm{het}}e^{12B\delta}\delta^2,
\label{eqn:pooled_interior_radius}
\end{equation}
and let $\Delta_0$ be the unique positive solution of
$R_\Theta(\delta)=\rho_\Theta/2$. This solution exists because
$R_\Theta$ is continuous and strictly increasing, with
$R_\Theta(0)=0$ and $R_\Theta(\delta)\to\infty$. Lemma~\ref{lem:app_pooled_bias}
shows that, for $\Delta\le\Delta_0$, each pooled population target is
at least $\rho_\Theta/2$ from the boundary of $\Theta$. The radius
$\Delta_0$ depends only on
$(B,m,M,D_\Theta,\rho_\Theta)$, and continuous $\Delta$-dependent
constants are therefore bounded uniformly on $[0,\Delta_0]$.

We also use the third-order directional derivative notation
\[
\nabla^3 f(x)[a,b,c]
:=
\sum_{i,j,\ell=1}^d
\frac{\partial^3 f(x)}{\partial x_i\partial x_j\partial x_\ell}
\,a_i b_j c_\ell.
\]
For twice- and thrice-differentiable functions, the identity
\[
\frac{d}{dt}\Big|_{t=0} a^\top \nabla^2 f(x+t c)a
=
\nabla^3 f(x)[a,a,c]
\]
will be used repeatedly.

\subsection{Hessian representation and the SCL property}

\begin{lemma}
\label{lem:app_hessian_representation}
For any store $k$, any parameter $\beta\in\mathbb R^d$, and any realized dataset $Y_k$,
\[
\nabla^2 L_k(\beta;Y_k)=n_k \Sigma_X(\beta),
\]
where $\Sigma_X(\cdot)$ is defined in \eqref{eqn:x_sigma}. 
In particular, the Hessian of the empirical loss $L_k$ is deterministic.
\end{lemma}

\begin{proof}{Proof of Lemma~\ref{lem:app_hessian_representation}.}
Write
\[
g(\beta):=\log\!\Bigl(1+\sum_{j=1}^N e^{\beta^\top X_j}\Bigr)
=\log\!\Bigl(\sum_{j=0}^N e^{\beta^\top X_j}\Bigr).
\]
The negative log-likelihood can be written as
\[
L_k(\beta;Y_k)
=-\sum_{i=1}^N c_{ik}\,\beta^\top X_i+n_k g(\beta),
\]
so the linear count term has zero Hessian and
\[
\nabla^2 L_k(\beta;Y_k)=n_k\nabla^2 g(\beta).
\]
Differentiating $g$ once yields
\[
\nabla g(\beta)=\sum_{i=0}^N P_i(\beta)X_i = \mu_X(\beta), 
\]
where $\mu_X(\cdot)$ is defined in \eqref{eqn:x_mu}.  Differentiating again gives
\[
\nabla^2 g(\beta)
=
\sum_{i=0}^N P_i(\beta)X_iX_i^\top- \mu_X(\beta)\mu_X(\beta)^\top
=
\Sigma_X(\beta),
\]
which proves the claim.
\hfill \Halmos
\end{proof}

\begin{corollary}
\label{cor:app_population_hessian}
For any store \(k\) and any parameter \(\beta\in\mathbb R^d\),
\[
\nabla^2 L_k^\star(\beta)=n_k \Sigma_X(\beta),
\]
where $\Sigma_X(\cdot)$ is defined in \eqref{eqn:x_sigma}.  In particular,
\[
\nabla^2 L_k^\star(\beta_k^\star)
=
n_k\Sigma_X(\beta_k^\star)
=
n_k\Sigma_X^k.
\]
\end{corollary}

\begin{proof}{Proof.}
By Lemma~\ref{lem:app_hessian_representation},
\[
\nabla^2 L_k(\beta;Y_k)=n_k\Sigma_X(\beta)
\]
for every realized dataset \(Y_k\). Since the right-hand side does not
depend on the realized counts, we may interchange expectation and
differentiation to obtain
\[
\nabla^2 L_k^\star(\beta)
=
\nabla^2 \mathbb E[L_k(\beta;Y_k)]
=
\mathbb E[\nabla^2 L_k(\beta;Y_k)]
=
n_k\Sigma_X(\beta)
\]
Evaluating at \(\beta=\beta_k^\star\) gives the claim. \hfill \Halmos
\end{proof}

\begin{definition}
A convex and thrice continuously differentiable function $f:\mathbb R^d\to\mathbb R$ is called $R$-self-concordant-like if for all \(x,a,b\in\mathbb R^d\),
\begin{equation}
\label{eq:app_scl_definition}
\bigl|\nabla^3 f(x)[a,a,b]\bigr|
\le
R\|b\|_2\,a^\top\nabla^2 f(x)a.
\end{equation}
\end{definition}

\begin{theorem}
\label{thm:app_mnl_scl}
Let $g(\beta):=\log\bigl(1+\sum_{j=1}^N e^{\beta^\top X_j}\bigr)$. Under Assumption~\ref{ass:bounded}, every function of the form
$\beta\mapsto a^\top\beta+n\,g(\beta)$, with $a\in\mathbb R^d$ and $n>0$, is
$2B$-self-concordant-like. In particular the empirical loss $L_k(\beta;Y_k)$, the population loss $L_k^\star(\beta)$, and their pooled and leave-one-out counterparts $L_0(\beta;Y)$, $L_0^{(-k)}(\beta;Y)$, and $L_0^{(-k),\star}(\beta)$ are all $2B$-self-concordant-like.
\end{theorem}

\begin{proof}{Proof of Theorem~\ref{thm:app_mnl_scl}.}
Consider any function $a^\top \beta + n g(\beta)$. The linear term has zero second and third derivatives, and \eqref{eq:app_scl_definition} is invariant under multiplication by $n > 0$ because both sides scale linearly, so it is enough to prove that $g$ is $2B$-self-concordant-like. Differentiating $g$ once yields $\nabla g(\beta)= \mu_X(\beta)$, where $\mu_X(\cdot)$ is defined in \eqref{eqn:x_mu}. Differentiating again gives the covariance representation
\[
\nabla^2 g(\beta)
=
\Sigma_X(\beta)
\]

Recall $X^\beta$ is the random feature vector with $\mathbb P(X^\beta = X_i) = P_i(\beta)$ and whose mean and covariance are $\mu_X(\beta)$ and $\Sigma_X(\beta)$, respectively. Fix directions $u,v\in\mathbb R^d$ and define random variables
\[
a:=u^\top X^\beta,
\qquad
b:=v^\top X^\beta.
\]
Note that $a$ and $b$ can only take the following values
\[
a_i:=u^\top X_i,
\qquad
b_i:=v^\top X_i,
\qquad i=0,\dots,N.
\]

Then the covariance identity implies
\begin{equation}
\label{eq:app_scl_hess_repr}
u^\top \nabla^2 g(\beta)u
=
\mathbb E\bigl[(a-\mathbb E[a])^2\bigr].
\end{equation}
To compute the mixed third derivative, consider
\[
\psi(t):=u^\top \nabla^2 g(\beta+t v)u.
\]
Writing $P_i(t):=P_i(\beta+t v)$ and $\bar a(t):=\sum_j P_j(t)a_j$, we have
\[
\psi(t)
=
\sum_{i=0}^N P_i(t)\bigl(a_i-\bar a(t)\bigr)^2.
\]
Differentiating at $t=0$ gives
\[
\psi'(0)
=
\sum_{i=0}^N \Bigl(\partial_v P_i(\beta)\Bigr)(a_i-\bar a)^2,
\qquad
\bar a:=\mathbb E[a],
\]
because the derivative of the mean term cancels after summation. The quotient rule yields
\[
\partial_v P_i(\beta)
=
P_i(\beta)\bigl(b_i-\bar b\bigr),
\qquad
\bar b:=\mathbb E[b].
\]
Hence
\[
\psi'(0)
=
\sum_{i=0}^N P_i(\beta)\,(b_i-\bar b)\,(a_i-\bar a)^2
=
\mathbb E\bigl[(a-\mathbb E[a])^2(b-\mathbb E[b])\bigr].
\]
Since $\psi'(0)=\nabla^3 g(\beta)[u,u,v]$, we obtain
\begin{equation}
\label{eq:app_scl_third_repr}
\nabla^3 g(\beta)[u,u,v]
=
\mathbb E\bigl[(a-\mathbb E[a])^2(b-\mathbb E[b])\bigr].
\end{equation}
Under Assumption~\ref{ass:bounded}, $\|X_i\|_2\le B$ for every $i$, so
\[
|b_i|=|v^\top X_i|\le B\|v\|_2,
\qquad i=0,\dots,N.
\]
Because $\mathbb E[b]$ is a convex combination of the $b_i$'s, it also satisfies
\[
|\mathbb E[b]|\le B\|v\|_2,
\]
and therefore
\[
|b_i-\mathbb E[b]|
\le |b_i|+|\mathbb E[b]|
\le 2B\|v\|_2.
\]
Using \eqref{eq:app_scl_third_repr} and then \eqref{eq:app_scl_hess_repr},
\begin{align*}
\bigl|\nabla^3 g(\beta)[u,u,v]\bigr|
&=
\Bigl|\mathbb E\bigl[(a-\mathbb E[a])^2(b-\mathbb E[b])\bigr]\Bigr| \\
&\le
\sup_i |b_i-\mathbb E[b]|\,\mathbb E\bigl[(a-\mathbb E[a])^2\bigr] \\
&\le 2B\|v\|_2\,u^\top \nabla^2 g(\beta)u,
\end{align*}
which is exactly the mixed self-concordant-like inequality with constant $2B$. 

Each of the stated losses has this form. By \eqref{eq:sec3_store_loss}, $L_k(\beta;Y_k) = -\sum_i c_{ik} X_i^\top \beta + n_k g(\beta)$, and taking expectations in \eqref{eq:sec3_pop_loss} gives $L_k^\star(\beta) = - n_k\sum_i P_i(\beta_k^\star) X_i^\top \beta + n_k g(\beta)$. The pooled and leave-one-out losses are sums of such functions, hence of the same form.\hfill \Halmos
\end{proof}

\begin{lemma}
\label{lem:app_hessian_stability}
Let $f$ be $R$-self-concordant-like. For any $x,y\in\mathbb R^d$,
\[
e^{-R\|y-x\|_2}\,\nabla^2 f(x)
\preceq
\nabla^2 f(y)
\preceq
e^{R\|y-x\|_2}\,\nabla^2 f(x).
\]
\end{lemma}

\begin{proof}{Proof of Lemma~\ref{lem:app_hessian_stability}.}
Fix $u\in\mathbb R^d$ and define
\[
\psi(t):=u^\top \nabla^2 f\bigl(x+t(y-x)\bigr)u,
\qquad t\in[0,1].
\]
Then
\[
\psi'(t)=\nabla^3 f\bigl(x+t(y-x)\bigr)[u,u,y-x].
\]
By the SCL property,
\[
|\psi'(t)|\le R\|y-x\|_2\,\psi(t).
\]
Whenever $\psi(t)>0$, this implies
\[
-R\|y-x\|_2\le \frac{\psi'(t)}{\psi(t)}\le R\|y-x\|_2.
\]
Integrating over $t\in[0,1]$ gives
\[
e^{-R\|y-x\|_2}\psi(0)
\le \psi(1) \le e^{R\|y-x\|_2}\psi(0),
\]
and since this holds for every $u$, the matrix inequality follows.
\hfill \Halmos
\end{proof}

\begin{lemma}
\label{lem:app_quadratic_approx}
Let $f$ be $R$-self-concordant-like, and fix $x,y\in\mathbb R^d$. Then
\[
f(y)-f(x)-\nabla f(x)^\top(y-x)
=
\frac12 (y-x)^\top \nabla^2 f(x)(y-x)+r_f(x,y),
\]
where
\[
|r_f(x,y)|
\le
\frac12\|y-x\|_{\nabla^2 f(x)}^2\,
\omega_{\mathrm{SCL}}\bigl(R\|y-x\|_2\bigr),
\]
with
\[
\omega_{\mathrm{SCL}}(a)
:=
\begin{cases}
\dfrac{2e^a-(a^2+2a+2)}{a^2}, & a>0,\\[4pt]
0, & a=0.
\end{cases}
\]
Moreover, $\omega_{\mathrm{SCL}}(a)=\frac{a}{3}+O(a^2)$ as $a\downarrow0$.
\end{lemma}

\begin{proof}{Proof of Lemma~\ref{lem:app_quadratic_approx}.}
Let $v:=y-x$ and define $\phi(t):=f(x+t v)$ for $t\in[0,1]$. Taylor's theorem with integral remainder gives
\[
\phi(1)=\phi(0)+\phi'(0)+\frac12\phi''(0)
+\int_0^1 \frac12(1-t)^2\phi'''(t)dt.
\]
Since
\[
\phi''(t)=v^\top \nabla^2 f(x+t v)v,
\qquad
\phi'''(t)=\nabla^3 f(x+t v)[v,v,v],
\]
the SCL property yields
\[
|\phi'''(t)|\le R\|v\|_2\,\phi''(t).
\]
By Lemma~\ref{lem:app_hessian_stability},
\[
\phi''(t)\le e^{R t\|v\|_2}\phi''(0).
\]
Substituting into the integral remainder and computing
\[
\int_0^1 (1-t)^2 e^{a t}dt
=
\frac{2e^a-(a^2+2a+2)}{a^3}
\qquad (a>0)
\]
gives the displayed bound, with $\phi''(0)=\|v\|_{\nabla^2 f(x)}^2$.
\hfill \Halmos
\end{proof}

\section{Finite-Sample Concentration of the MLE}
\label{app:mle}

This section collects the concentration results used to control the local and leave-one-out pooled MLEs.

\begin{lemma}
\label{lem:app_mle_implicit}
Let $\widehat\beta_k$ denote the constrained MLE in store $k$. On the
event $\{\widehat\beta_k\in\operatorname{int}(\Theta)\}$, define the
path-average covariance
\[
\bar\Sigma_X^k
:=
\int_0^1 \Sigma_X\bigl({\beta_k^\star+t(\widehat\beta_k-\beta_k^\star)}\bigr)dt,
\]
where $\Sigma_X(\cdot)$ is defined in \eqref{eqn:x_sigma}. Then
\[
\widehat\beta_k-\beta_k^\star
=
-\frac{1}{n_k}(\bar\Sigma_X^k)^{-1}s_k(\beta_k^\star),
\]
where $s_k(\beta):=\nabla L_k(\beta;Y_k)$.
\end{lemma}

\begin{proof}{Proof of Lemma~\ref{lem:app_mle_implicit}.}
On the stated event, the constraint is inactive, so first-order
optimality gives $s_k(\widehat\beta_k)=0$. Applying the fundamental
theorem of calculus to the score along the segment
$\beta_k^\star+t(\widehat\beta_k-\beta_k^\star)$ gives
\[
0-s_k(\beta_k^\star)
=
\int_0^1 \nabla^2 L_k\bigl(\beta_k^\star+t(\widehat\beta_k-\beta_k^\star)\bigr)dt\,(\widehat\beta_k-\beta_k^\star).
\]
By Lemma~\ref{lem:app_hessian_representation}, the integral equals $n_k\bar\Sigma_X^k$, which yields the claim.
\hfill \Halmos
\end{proof}

\begin{lemma}
\label{lem:app_score_unbiased}
Under Assumptions~\ref{ass:bounded} and~\ref{ass:indep}, the store-level score at the truth is centered:
\[
\mathbb E\bigl[s_k(\beta_k^\star)\bigr]=0.
\]
\end{lemma}

\begin{proof}{Proof of Lemma~\ref{lem:app_score_unbiased}.}
For a single observation \(y\in\{0,\dots,N\}\) drawn from the MNL model with parameter \(\beta_k^\star\), the score contribution is
\[
-\nabla_\beta \log P_y(\beta_k^\star)
=
-\;X_y+\sum_{j=0}^N P_j(\beta_k^\star)X_j.
\]
Taking expectation under \(y\sim P(\beta_k^\star)\) yields zero. Since \(s_k(\beta_k^\star)\) is the sum of \(n_k\) independent single-observation scores, its expectation is also zero. \hfill \Halmos
\end{proof}

\begin{lemma}
\label{lem:app_whitened_score_subg}
Under Assumptions~\ref{ass:bounded}, \ref{ass:fisher},
and~\ref{ass:indep}, the whitened score
\[
Z_k:=(n_k\Sigma_X^k)^{-1/2}s_k(\beta_k^\star)
\]
is isotropic and sub-Gaussian. More precisely, there exists a universal constant \(C_\psi>0\) such that
\[
\|Z_k\|_{\psi_2}\le C_\psi \frac{B}{\sqrt m}.
\]
\end{lemma}

\begin{proof}{Proof of Lemma~\ref{lem:app_whitened_score_subg}.}
Let \(H_k^\star:=n_k\Sigma_X^k\). By Lemma~\ref{lem:app_score_unbiased}, \(\mathbb E[s_k(\beta_k^\star)]=0\). Since the Fisher information for the MNL model is \(H_k^\star=n_k\Sigma_X^k\), the covariance of the score at the truth is exactly \(H_k^\star\). Therefore
\[
\operatorname{Cov}(Z_k)
=
H_k^{*-1/2}\operatorname{Cov}(s_k(\beta_k^\star))H_k^{*-1/2}
=
I_d,
\]
so \(Z_k\) is isotropic.

To prove sub-Gaussianity, write the score as
\[
s_k(\beta_k^\star)=\sum_{t=1}^{n_k}\xi_{t,k},
\qquad
\xi_{t,k}:=\mathbb E[X^{\beta_k^\star}]-X_{y_{t,k}},
\]
where \(X_{y_{t,k}}\) is the feature vector of the $t^\text{th}$ chosen alternative $y_t$ at store \(k\). Note that $(X_{y_t,k})_{t=1}^{n_k}$ are i.i.d.~and have the same distribution as $X^{\beta_k^\star}$. For any unit vector \(u\in\mathbb S^{d-1}\),
\[
u^\top H_k^{*-1/2}\xi_{t,k}
\]
is centered and bounded in absolute value by
\[
\|u\|_2\,\|H_k^{*-1/2}\|_{\mathrm{op}}\,
\|\mathbb E[X^{{\beta_k^\star}}]-X_{y_{t,k}}\|_2
\le
\frac{2B}{\sqrt{n_k m}},
\]
because \(\Sigma_X^k\succeq mI_d\) and \(\|X_i\|_2\le B\). Hoeffding's inequality for sums of independent bounded random variables then gives, for all \(t\ge0\),
\[
\Pr\bigl(|u^\top Z_k|\ge t\bigr)
\le
2\exp\!\left(-\frac{m t^2}{8B^2}\right).
\]
Thus every one-dimensional projection of \(Z_k\) is sub-Gaussian with parameter of order \(B^2/m\), and the displayed \(\psi_2\)-bound follows. \hfill \Halmos
\end{proof}

\begin{lemma}
\label{lem:app_score_concentration}
Under Assumptions~\ref{ass:bounded}, \ref{ass:fisher},
and~\ref{ass:indep}, there exist constants \(C_s,c_s>0\) depending
only on \((B,m)\) such that for every store \(k\) and every \(t\ge0\),
\[
\Pr\!\left(\,
\|s_k(\beta_k^\star)\|_{(H_k^\star)^{-1}}
\ge
C_s\bigl(\sqrt d+t\bigr)
\right)
\le 2e^{-c_s t^2}.
\]
\end{lemma}

\begin{proof}{Proof of Lemma~\ref{lem:app_score_concentration}.}
By Lemma~\ref{lem:app_whitened_score_subg},
\[
\|s_k(\beta_k^\star)\|_{(H_k^\star)^{-1}}
=
\|Z_k\|_2,
\qquad
Z_k:=(n_k\Sigma_X^k)^{-1/2}s_k(\beta_k^\star),
\]
and \(Z_k\) is isotropic sub-Gaussian with \(\|Z_k\|_{\psi_2}\lesssim B/\sqrt m\). Let \(\mathcal N\subset \mathbb S^{d-1}\) be a \(1/2\)-net of the sphere with \(|\mathcal N|\le 5^d\). For every \(z\in\mathbb R^d\),
\[
\|z\|_2\le 2\max_{v\in\mathcal N}|v^\top z|.
\]
Therefore,
\[
\Pr(\|Z_k\|_2\ge 2s)
\le
\sum_{v\in\mathcal N}\Pr(|v^\top Z_k|\ge s)
\le
2\cdot 5^d \exp\!\left(-c\,\frac{m s^2}{B^2}\right)
\]
for a universal constant \(c>0\). Choosing
\[
s=\frac{B}{\sqrt m}\bigl(\alpha\sqrt d+t\bigr)
\]
with \(\alpha\) large enough to absorb the \(5^d\) factor yields
\[
\Pr\!\left(\|Z_k\|_2\ge C_s(\sqrt d+t)\right)\le 2e^{-c_s t^2}
\]
for suitable constants \(C_s,c_s>0\) depending only on \((B,m)\). Rewriting in terms of the score gives the claim. \hfill \Halmos
\end{proof}

\begin{lemma}
\label{lem:app_constrained_basic}
Let
\(\ell(\beta)=a^\top\beta+n g(\beta)\), where
\(g(\beta)=\log(1+\sum_{j=1}^N e^{\beta^\top X_j})\), and let
\(\widehat\beta\in\argmin_{\beta\in\Theta}\ell(\beta)\). For any
\(\beta^\circ\in\Theta\) such that
\(H^\circ:=\nabla^2\ell(\beta^\circ)\succ0\),
\[
\|\widehat\beta-\beta^\circ\|_{H^\circ}
\le
2e^{2BD_\Theta}
\|\nabla\ell(\beta^\circ)\|_{(H^\circ)^{-1}}.
\]
\end{lemma}

\begin{proof}{Proof of Lemma~\ref{lem:app_constrained_basic}.}
Set $u:=\widehat\beta-\beta^\circ$. Because both points belong to
$\Theta$, $\|u\|_2\le D_\Theta$. Optimality of $\widehat\beta$ and
Taylor's theorem with integral remainder give
\begin{align*}
0
&\ge
\ell(\beta^\circ+u)-\ell(\beta^\circ)\\
&=
\nabla\ell(\beta^\circ)^\top u
+
\int_0^1(1-t)
u^\top\nabla^2\ell(\beta^\circ+tu)u\,dt.
\end{align*}
By Theorem~\ref{thm:app_mnl_scl} and
Lemma~\ref{lem:app_hessian_stability},
\[
\nabla^2\ell(\beta^\circ+tu)
\succeq
e^{-2Bt\|u\|_2}H^\circ
\succeq
e^{-2BD_\Theta}H^\circ.
\]
Therefore, with
$R:=\|u\|_{H^\circ}$ and
$G:=\|\nabla\ell(\beta^\circ)\|_{(H^\circ)^{-1}}$,
\[
0\ge -GR+\frac12e^{-2BD_\Theta}R^2.
\]
The claim follows after dividing by $R$ when $R>0$; it is immediate
when $R=0$. \hfill \Halmos
\end{proof}

\begin{theorem}
\label{thm:app_mle_concentration}
Under Assumptions~\ref{ass:bounded}--\ref{ass:indep}, there exist
constants \(C_\beta,c_\beta>0\), depending only on
\((B,m,D_\Theta)\), such that, for every store~\(k\) and every
\(t\ge0\),
\begin{equation}
\label{eq:app_local_mle_tail}
\Pr\!\left(
\|\widehat\beta_k-\beta_k^\star\|_{k}
\ge
C_\beta\bigl(\sqrt d+t\bigr)
\right)
\le
2e^{-c_\beta t^2}.
\end{equation}
In particular, taking \(t=\sqrt d\),
\[
\Pr\!\left(
\|\widehat\beta_k-\beta_k^\star\|_{k}
\le
2C_\beta\sqrt d
\right)
\ge
1-2e^{-c_\beta d}.
\]
On the same events, \(\Sigma_X^k\succeq mI_d\) gives the Euclidean bound
\[
\|\widehat\beta_k-\beta_k^\star\|_2
\le
\frac{C_\beta}{\sqrt m}\,\frac{\sqrt d+t}{\sqrt{n_k}} .
\]
\end{theorem}

\begin{proof}{Proof of Theorem~\ref{thm:app_mle_concentration}.}
Apply Lemma~\ref{lem:app_constrained_basic} to
$\ell(\beta)=L_k(\beta;Y_k)$ and $\beta^\circ=\beta_k^\star$. By
Lemma~\ref{lem:app_hessian_representation},
$H^\circ=H_k^\star=n_k\Sigma_X^k$, and therefore
\[
\|\widehat\beta_k-\beta_k^\star\|_k
\le
2e^{2BD_\Theta}
\|s_k(\beta_k^\star)\|_{(H_k^\star)^{-1}}.
\]
The full-range tail bound now follows directly from
Lemma~\ref{lem:app_score_concentration}, with
$C_\beta=2e^{2BD_\Theta}C_s$ and $c_\beta=c_s$. The case
$t=\sqrt d$ is immediate. Finally,
$\Sigma_X^k\succeq mI_d$ gives
$\|v\|_2\le\|v\|_k/\sqrt{n_km}$, which yields the Euclidean bound.
\hfill \Halmos
\end{proof}

\begin{lemma}
\label{lem:app_mle_bias}
Under Assumptions~\ref{ass:bounded}--\ref{ass:sample}, there exists a
constant \(C_{\mathrm{bias}}>0\), depending only on
\((B,m,M,C_{\max},D_\Theta,\rho_\Theta)\), such that
\[
\|\mathbb E e_k\|_k
\le
C_{\mathrm{bias}}\frac{d}{\sqrt{n_k}},
\qquad
e_k:=\widehat\beta_k-\beta_k^\star.
\]
\end{lemma}

\begin{proof}{Proof.}
Let
\[
H_k^\star:=\nabla^2 L_k^\star(\beta_k^\star)=n_k\Sigma_X^k.
\]
Write
\[
e_k:=\widehat\beta_k-\beta_k^\star,
\qquad
s_k:=\nabla L_k(\beta_k^\star;Y_k),
\qquad
G_k:=\|s_k\|_{(H_k^\star)^{-1}},
\]
and introduce the linearized error
\[
e_k^0:=-(H_k^\star)^{-1}s_k.
\]
By score unbiasedness and the information identity,
\[
\mathbb E e_k^0=0,
\qquad
\mathbb E G_k^2=d.
\]

Consider the interior event
\[
\mathcal I_k
:=
\{\|e_k\|_2<\rho_\Theta/2\}.
\]
The Euclidean form of Theorem~\ref{thm:app_mle_concentration}, together
with $n_k\ge C_0d$, implies that $C_0$ can be chosen so that
\begin{equation}
\label{eq:app_mle_interior_tail}
\Pr(\mathcal I_k^c)\le 2e^{-c_I n_k}
\end{equation}
for a constant $c_I>0$ depending only on
$(B,m,D_\Theta,\rho_\Theta)$. Indeed, take
$t=a_I\sqrt{n_k}$ in \eqref{eq:app_local_mle_tail}, with $a_I>0$
small enough, and then take $C_0$ large enough that the resulting
Euclidean radius is below $\rho_\Theta/2$.

On $\mathcal I_k$, the estimator is in the interior of $\Theta$, so
Lemma~\ref{lem:app_mle_implicit} applies. With
\[
\widetilde H_k
:=
\int_0^1\nabla^2L_k(\beta_k^\star+te_k;Y_k)\,dt,
\]
we have $e_k=-\widetilde H_k^{-1}s_k$. Let
\[
M_k
:=
(H_k^\star)^{1/2}
\bigl(\widetilde H_k^{-1}-(H_k^\star)^{-1}\bigr)
(H_k^\star)^{1/2}.
\]
Set $a_k:=2B\|e_k\|_2$. The SCL Hessian comparison along the segment
from $\beta_k^\star$ to $\widehat\beta_k$, followed by integration,
gives
\[
\frac{1-e^{-a_k}}{a_k}H_k^\star
\preceq
\widetilde H_k
\preceq
\frac{e^{a_k}-1}{a_k}H_k^\star.
\]
After inversion, the elementary inequalities
$a/(1-e^{-a})-1\le a$ and
$1-a/(e^a-1)\le a$ for $a\ge0$ imply
\[
\|M_k\|_{\mathrm{op}}
\le a_k
=2B\|e_k\|_2
\le
\frac{2B}{\sqrt{m n_k}}\|e_k\|_k.
\]
Consequently, on $\mathcal I_k$,
\[
\|e_k-e_k^0\|_k
\le
\frac{2B}{\sqrt{m n_k}}\|e_k\|_kG_k.
\]

Since $\mathbb E e_k^0=0$, splitting over $\mathcal I_k$ yields
\[
\|\mathbb E e_k\|_k
\le
\mathbb E[\|e_k-e_k^0\|_k\mathbf 1_{\mathcal I_k}]
+
\mathbb E[(\|e_k\|_k+G_k)\mathbf 1_{\mathcal I_k^c}].
\]
For the first term, Cauchy--Schwarz, the case $q=2$ of
Lemma~\ref{lem:app_Vind}, and $\mathbb EG_k^2=d$ give
\[
\mathbb E[\|e_k-e_k^0\|_k\mathbf 1_{\mathcal I_k}]
\le
\frac{2B}{\sqrt{m n_k}}
(\mathbb E\|e_k\|_k^2)^{1/2}(\mathbb EG_k^2)^{1/2}
\le
C\frac{d}{\sqrt{n_k}}.
\]
For the second term, another application of Cauchy--Schwarz and
\eqref{eq:app_mle_interior_tail} gives
\[
\mathbb E[(\|e_k\|_k+G_k)\mathbf 1_{\mathcal I_k^c}]
\le
C\sqrt d\,e^{-c_I n_k/2}
\le
C\frac{d}{\sqrt{n_k}},
\]
where the final inequality follows uniformly for $n_k\ge C_0d$ and
$d\ge1$. Combining the two bounds proves the claim. \hfill \Halmos
\end{proof}

\section{Auxiliary Results for the Shrinkage Analysis}
\label{app:shrinkage_aux}

\begin{lemma}
\label{lem:app_independence}
Under Assumption~\ref{ass:indep}, \(\widehat\beta_k\) depends only on \(Y_k\), whereas \(\widehat\beta_0^{(-k)}\) depends only on \(\{Y_\ell:\ell\neq k\}\). Hence
\[
\widehat\beta_k \perp\!\!\!\perp \widehat\beta_0^{(-k)}.
\]
\end{lemma}

\begin{proof}{Proof of Lemma~\ref{lem:app_independence}.}
This is immediate from the independence of the store-level samples and the fact that the two estimators are measurable with respect to disjoint sigma-fields. \hfill \Halmos
\end{proof}

\begin{proof}{Proof of Lemma~\ref{lem:theory_local_expansion}.}
Apply Lemma~\ref{lem:app_quadratic_approx} to the population loss \(L_k^\star\) at \(x=\beta_k^\star\). Since \(\nabla L_k^\star(\beta_k^\star)=0\) and \(\nabla^2 L_k^\star(\beta_k^\star)=H_k^\star=n_k\Sigma_X^k\), we obtain
\[
L_k^\star(\beta)-L_k^\star(\beta_k^\star)
=
\frac12\|\beta-\beta_k^\star\|_k^2+r_k(\beta),
\]
with
\[
|r_k(\beta)|
\le
\frac12\|\beta-\beta_k^\star\|_k^2\,
\omega_{\mathrm{SCL}}\!\bigl(2B\|\beta-\beta_k^\star\|_2\bigr).
\]
Assumption~\ref{ass:fisher} implies \(\Sigma_X^k\succeq mI_d\), hence
\[
\|v\|_2\le \frac{\|v\|_k}{\sqrt{n_k m}}.
\]
Substituting this into the argument of \(\omega_{\mathrm{SCL}}\) yields exactly the bound stated in the lemma. \hfill \Halmos
\end{proof}

\begin{proof}{Proof of Lemma~\ref{lem:theory_quadratic_identity}.}
Expanding the square in the \(k\)-local norm gives
\[
\|(1-u)e_k+u g_k\|_k^2
=
(1-u)^2\|e_k\|_k^2
+
u^2\|g_k\|_k^2
+
2u(1-u)\langle e_k,g_k\rangle_k.
\]
Subtracting \(\|e_k\|_k^2\) and taking expectations yields
\begin{align*}
\mathbb E\!\left[
\|(1-u)e_k+u g_k\|_k^2-\|e_k\|_k^2
\right]
&=
\bigl((1-u)^2-1\bigr)V_k^{\mathrm{loc}}
+
u^2V_k^{\mathrm{pool}}
+
2u(1-u)C_k  \\
&=
(u^2-2u)V_k^{\mathrm{loc}}
+
u^2V_k^{\mathrm{pool}}
+
2u(1-u)C_k  \\
&=
\bigl(V_k^{\mathrm{loc}}+V_k^{\mathrm{pool}}-2C_k\bigr)u^2
-
2\bigl(V_k^{\mathrm{loc}}-C_k\bigr)u \\
&=
\alpha_k u^2-2\zeta_k u,
\end{align*}
where the last equality uses the definitions
\[
\alpha_k
:=
V_k^{\mathrm{loc}}
+
V_k^{\mathrm{pool}}
-
2C_k,
\qquad
\zeta_k
:=
V_k^{\mathrm{loc}}
-
C_k.
\]
This proves the identity. \hfill \Halmos
\end{proof}

\begin{proof}{Proof of Lemma~\ref{lem:theory_quadratic_benchmark}.}
For \(\alpha_k>0\),
\[
Q_k(u)= \frac{1}{2} \alpha_k\left(u-\frac{\zeta_k}{\alpha_k}\right)^2- \frac{1}{2}\frac{\zeta_k^2}{\alpha_k}.
\]
Hence the unconstrained minimizer is \(\zeta_k/\alpha_k\), and the constrained minimizer over \([0,1]\) is its Euclidean projection onto that interval. \hfill \Halmos
\end{proof}

\begin{lemma}
\label{lem:app_Vind}
Under Assumptions~\ref{ass:bounded}--\ref{ass:sample}, there are
constants $C_q<\infty$, $q=2,3,4$, and $c_{\mathrm{loc}}>0$,
depending only on
$(B,m,M,C_{\max},D_\Theta,\rho_\Theta)$, such that
\begin{equation}
\label{eq:app_individual_moments}
\mathbb E\|e_k\|_k^q
\le
C_qd^{q/2},
\qquad q=2,3,4.
\end{equation}
In particular, with $C_{\mathrm{loc}}:=C_2$,
\[
c_{\mathrm{loc}} d
\le
V_k^{\mathrm{loc}}
\le
C_{\mathrm{loc}} d
\qquad\text{for all }k.
\]
\end{lemma}

\begin{proof}{Proof of Lemma~\ref{lem:app_Vind}.}
\emph{Upper moments.} Let $R_k:=\|e_k\|_k$. The full-range tail in
Theorem~\ref{thm:app_mle_concentration} implies, for
$q\in\{2,3,4\}$,
\begin{align*}
\mathbb ER_k^q
&\le
C_\beta^q d^{q/2}
+2qC_\beta^q
\int_0^\infty(\sqrt d+t)^{q-1}e^{-c_\beta t^2}\,dt\\
&\le C_qd^{q/2},
\end{align*}
where the last inequality uses $d\ge1$. This proves
\eqref{eq:app_individual_moments}, including the upper bound on
$V_k^{\mathrm{loc}}$.

\emph{Lower second moment.} Set
\[
S_k
:=
s_k(\beta_k^\star)^\top (H_k^\star)^{-1}s_k(\beta_k^\star)
=
\frac1{n_k}\,s_k(\beta_k^\star)^\top(\Sigma_X^k)^{-1}s_k(\beta_k^\star).
\]
Since the score at the truth is centered
(Lemma~\ref{lem:app_score_unbiased}) and its covariance equals the
Fisher information \(H_k^\star\),
\begin{equation}
\label{eq:app_Vind_trace}
\mathbb E S_k
=
\operatorname{tr}\!\left((H_k^\star)^{-1}\Var\bigl(s_k(\beta_k^\star)\bigr)\right)
=
\operatorname{tr}(I_d)
=
d .
\end{equation}
Moreover \(S_k=\|Z_k\|_2^2\) for the whitened score \(Z_k\) of
Lemma~\ref{lem:app_whitened_score_subg}, so integrating the tail bound
of Lemma~\ref{lem:app_score_concentration}, which holds for every
\(t\ge0\), gives a constant \(C_{S,4}>0\) depending only on \((B,m)\) with
\begin{equation}
\label{eq:app_Vind_score_4th}
\mathbb E S_k^2 \le C_{S,4}\,d^2 .
\end{equation}

Fix $\theta>0$, chosen below, and define the score event
\[
\mathcal E_k
:=
\left\{
\|s_k(\beta_k^\star)\|_{(H_k^\star)^{-1}}
\le C_s(2\sqrt d+\theta)
\right\}.
\]
Lemma~\ref{lem:app_score_concentration}, at
$t=\sqrt d+\theta$, gives
$\Pr(\mathcal E_k^c)\le2e^{-c_s\theta^2}$. On $\mathcal E_k$,
Lemma~\ref{lem:app_constrained_basic} gives
\[
\|e_k\|_k
\le
2e^{2BD_\Theta}C_s(2\sqrt d+\theta).
\]
Because $n_k\ge C_0d$ and $d\ge1$, $C_0$ may be chosen so that this
event implies both
\(\|e_k\|_2<\rho_\Theta/2\) and
\[
a_k:=2B\|e_k\|_2
\le\log2.
\]
The first implication places $\widehat\beta_k$ in the interior of
$\Theta$, so Lemma~\ref{lem:app_mle_implicit} gives
\[
e_k=-\frac1{n_k}(\bar\Sigma_X^k)^{-1}s_k(\beta_k^\star).
\]
The SCL Hessian comparison along the same segment yields
\[
\frac{1-e^{-a_k}}{a_k}\,\Sigma_X^k
\preceq
\bar\Sigma_X^k
\preceq
\frac{e^{a_k}-1}{a_k}\,\Sigma_X^k ,
\]
so that on \(\mathcal E_k\), where \(a_k\le\log2\),
\[
\|e_k\|_k^2
=
\frac1{n_k}s_k(\beta_k^\star)^\top(\bar\Sigma_X^k)^{-1}
\Sigma_X^k(\bar\Sigma_X^k)^{-1}s_k(\beta_k^\star)
\ge
c_M\,S_k,
\qquad
c_M:=\log^2 2 .
\]
Discarding \(\mathcal E_k^c\) and using \eqref{eq:app_Vind_trace},
\eqref{eq:app_Vind_score_4th}, and Cauchy--Schwarz,
\[
V_k^{\mathrm{loc}}
\ge
c_M\,\mathbb E\bigl[S_k\mathbf 1_{\mathcal E_k}\bigr]
=
c_M\Bigl(d-\mathbb E\bigl[S_k\mathbf 1_{\mathcal E_k^c}\bigr]\Bigr)
\ge
c_M\Bigl(d-\sqrt{2C_{S,4}\,e^{-c_s\theta^2}}\;d\Bigr).
\]
Choosing \(\theta\) so that
\(\sqrt{2C_{S,4} e^{-c_s\theta^2}}\le\tfrac12\),
which is possible since the left-hand side tends to \(0\), yields
\(V_k^{\mathrm{loc}}\ge \tfrac12 c_M d\) for every \(d\ge1\). Taking
\(c_{\mathrm{loc}}:=\tfrac12\log^22\) completes the proof.
\hfill \Halmos
\end{proof}

\begin{lemma}
\label{lem:app_pooled_bias}
Define
\[
L_0^{(-k),\star}(\beta):=\sum_{\ell\neq k}L_\ell^\star(\beta),
\qquad
\beta_0^{(-k),\star}:=\arg\min_{\beta\in\Theta} L_0^{(-k),\star}(\beta),
\]
and the leave-one-out weighted barycenter
\[
\bar\beta_{-k}:=\frac{1}{S}\sum_{\ell\neq k} n_\ell \beta_\ell^\star,
\qquad
S:=\sum_{\ell\neq k}n_\ell.
\]
Under Assumptions~\ref{ass:bounded}, \ref{ass:parameter},
\ref{ass:fisher}, and~\ref{ass:hetero}, and whenever
$\Delta\in[0,\Delta_0]$, with $\Delta_0$ defined in
Appendix~\ref{sec:preliminaries},
\[
\|\beta_0^{(-k),\star}-\beta_k^\star\|_2
\le
R_\Theta(\Delta)
=
2\Delta + C_{\mathrm{het}} e^{12B\Delta}\Delta^2,
\]
where $C_{\mathrm{het}}$ and $R_\Theta$ are defined in
\eqref{eqn:pooled_interior_radius}. Moreover,
\[
\operatorname{dist}(\beta_0^{(-k),\star},\partial\Theta)
\ge \rho_\Theta/2.
\]
Consequently,
\[
\|\beta_0^{(-k),\star}-\beta_k^\star\|_k
\le
\sqrt{n_k M}\left(2\Delta + C_{\mathrm{het}} e^{12B\Delta}\Delta^2\right).
\]
The same Euclidean bound and interior-margin conclusion hold for the
full pooled population target $\beta_0^\star$.
\end{lemma}

\begin{proof}{Proof of Lemma~\ref{lem:app_pooled_bias}.}
Because $\Theta$ is convex, the weighted barycenter
$\bar\beta_{-k}$ belongs to $\Theta$. Define
\[
d_\ell:=\beta_\ell^\star-\bar\beta_{-k},
\qquad \ell\neq k.
\]
By construction,
\[
\sum_{\ell\neq k} n_\ell d_\ell=0,
\qquad
\frac1S\sum_{\ell\neq k}n_\ell\|d_\ell\|_2^2\le \Delta^2,
\qquad
\|d_\ell\|_2\le 2\Delta.
\]

Expanding each $\mu_X(\beta_\ell^\star)$ around
$\bar\beta_{-k}$ gives
\[
\mu_X(\beta_\ell^\star)-\mu_X(\bar\beta_{-k})
=
\Sigma_X({\bar\beta_{-k}})d_\ell+r_\ell,
\]
where $\Sigma_X(\cdot)$ is defined in \eqref{eqn:x_sigma} and
\[
r_\ell
=
\int_0^1
\bigl(\Sigma_X({\bar\beta_{-k}+t d_\ell})-\Sigma_X({\bar\beta_{-k}})\bigr)d_\ell\,dt.
\]
The first-order terms cancel after weighted summation. Since
$\nabla L_\ell^\star(\beta)
=n_\ell(\mu_X(\beta)-\mu_X(\beta_\ell^\star))$, the pooled gradient at
the barycenter is therefore purely second order:
\[
\nabla L_0^{(-k),\star}(\bar\beta_{-k})=-\sum_{\ell\neq k}n_\ell r_\ell.
\]

By Theorem~\ref{thm:app_mnl_scl} and Lemma~\ref{lem:app_hessian_stability}, the covariance map \(\Sigma_X(\cdot)\) is Lipschitz up to an exponential factor along bounded paths. Using \(\|d_\ell\|_2\le 2\Delta\) and \(\Sigma_X({\bar\beta_{-k}})\preceq M e^{4B\Delta}I_d\), one obtains
\[
\|r_\ell\|_2
\le
BM e^{8B\Delta}\|d_\ell\|_2^2.
\]
Therefore
\[
\|\nabla L_0^{(-k),\star}(\bar\beta_{-k})\|_2
\le
SBM e^{8B\Delta}\Delta^2.
\]

Let $h_k:=\beta_0^{(-k),\star}-\bar\beta_{-k}$ and set
$H_{-k}:=\nabla^2L_0^{(-k),\star}(\bar\beta_{-k})$. Hessian stability
between $\bar\beta_{-k}$ and any $\beta_\ell^\star$ gives
\[
H_{-k}
=S\Sigma_X(\bar\beta_{-k})
\succeq
Sm e^{-4B\Delta}I_d.
\]
Apply Lemma~\ref{lem:app_constrained_basic} to the deterministic
function $L_0^{(-k),\star}$, with reference point
$\bar\beta_{-k}$ and constrained minimizer
$\beta_0^{(-k),\star}$. Using the preceding curvature bound and the
gradient bound gives
\begin{align*}
\|h_k\|_2
&\le
\frac{2e^{2BD_\Theta}}{Sm e^{-4B\Delta}}
\|\nabla L_0^{(-k),\star}(\bar\beta_{-k})\|_2\\
&\le
\frac{2BM}{m}e^{2BD_\Theta+12B\Delta}\Delta^2
=
C_{\mathrm{het}}e^{12B\Delta}\Delta^2.
\end{align*}
It follows that
\[
\|\beta_0^{(-k),\star}-\beta_k^\star\|_2
\le
\|\beta_0^{(-k),\star}-\bar\beta_{-k}\|_2+\|\bar\beta_{-k}-\beta_k^\star\|_2
\le
C_{\mathrm{het}}e^{12B\Delta}\Delta^2+2\Delta
=R_\Theta(\Delta).
\]
For $\Delta\le\Delta_0$, the definition of $\Delta_0$ gives
$R_\Theta(\Delta)\le\rho_\Theta/2$. Since
$\operatorname{dist}(\beta_k^\star,\partial\Theta)\ge\rho_\Theta$,
the distance function is $1$-Lipschitz, and hence
\[
\operatorname{dist}(\beta_0^{(-k),\star},\partial\Theta)
\ge
\rho_\Theta-R_\Theta(\Delta)
\ge
\rho_\Theta/2.
\]
The local-norm bound follows from $\Sigma_X^k\preceq MI_d$. Replacing
the leave-one-out index set by all stores gives the same argument for
$\beta_0^\star$. \hfill \Halmos
\end{proof}

\begin{lemma}
\label{lem:app_pooled_variance}
Under Assumptions~\ref{ass:bounded}--\ref{ass:sample}, and whenever
$\Delta \in [0,\Delta_0]$ for $\Delta_0$ defined in
Section~\ref{sec:preliminaries}, there exist functions
\(C_{\mathrm{pool}}^{(q)}(\Delta)>0\) for \(q=2,3,4\), depending only on
\((B,m,M,C_{\max},D_\Theta,\rho_\Theta,\Delta)\) and bounded uniformly
on \([0,\Delta_0]\),
such that
\[
\mathbb E\bigl\|\widehat\beta_0^{(-k)}-\beta_0^{(-k),\star}\bigr\|_k^{\,q}
\le
\frac{C_{\mathrm{pool}}^{(q)}(\Delta)}{(K-1)^{q/2}}\;d^{\,q/2},
\qquad q=2,3,4 .
\]
As a consequence, writing
\(C_{\mathrm{pool}}(\Delta):=C_{\mathrm{pool}}^{(2)}(\Delta)\),
\[
V_k^{\mathrm{pool}}
\le
C_{\mathrm{pool},1}(\Delta)\,d + C_{\mathrm{pool},2}(\Delta)\,n_k\Delta^2
\]
for suitable functions
\(C_{\mathrm{pool},1}(\Delta),C_{\mathrm{pool},2}(\Delta)>0\).
\end{lemma}

\begin{proof}{Proof of Lemma~\ref{lem:app_pooled_variance}.}
Write
\[
\widehat\beta_0:=\widehat\beta_0^{(-k)},
\qquad
\beta_0^\star:=\beta_0^{(-k),\star},
\qquad
S:=\sum_{\ell\neq k} n_\ell,
\]
and define the pooled population loss
\[
L_0^\star(\beta):=L_0^{(-k),\star}(\beta)=\sum_{\ell\neq k} L_\ell^\star(\beta).
\]
By Corollary~\ref{cor:app_population_hessian}, its Hessian at the pooled
population minimizer is
\[
H_0^\star:=\nabla^2 L_0^\star(\beta_0^\star)
=
S\,\Sigma_X(\beta_0^\star),
\]
where $\Sigma_X(\cdot)$ is defined in \eqref{eqn:x_sigma}.

\emph{Step 1: curvature of the pooled loss.} The proof of
Lemma~\ref{lem:app_pooled_bias} shows that the pooled population
minimizer is close to the weighted barycenter
\(\bar\beta_{-k}=S^{-1}\sum_{\ell\neq k}n_\ell\beta_\ell^\star\), namely
\(\|\beta_0^\star-\bar\beta_{-k}\|_2\le C_{\mathrm{het}}e^{12B\Delta}\Delta^2\).
Since \(\bar\beta_{-k}\) is a convex combination of the
\(\beta_\ell^\star\), Assumption~\ref{ass:hetero} gives
\(\|\bar\beta_{-k}-\beta_\ell^\star\|_2\le 2\Delta\) for every store
\(\ell\), so
\[
\|\beta_0^\star-\beta_\ell^\star\|_2
\le
2\Delta + C_{\mathrm{het}}e^{12B\Delta}\Delta^2
=:R_\Delta,
\qquad \ell\neq k .
\]
Each store-level population loss is \(2B\)-self-concordant-like by
Theorem~\ref{thm:app_mnl_scl}, so Lemma~\ref{lem:app_hessian_stability}
applied to \(L_\ell^\star\) and cancellation of \(n_\ell\) give
\[
e^{-2BR_\Delta}\Sigma_X^\ell
\preceq
\Sigma_X(\beta_0^\star)
\preceq
e^{2BR_\Delta}\Sigma_X^\ell,
\qquad \ell\neq k,
\]
where $\Sigma_X^\ell = \Sigma_X(\beta_\ell^\star)$. With
\(mI_d\preceq \Sigma_X^\ell \preceq MI_d\) from
Assumption~\ref{ass:fisher},
\[
c_-(\Delta)\,S\,I_d
\preceq
H_0^\star
\preceq
c_+(\Delta)\,S\,I_d,
\qquad
c_-(\Delta):=m e^{-2BR_\Delta},
\quad
c_+(\Delta):=M e^{2BR_\Delta}.
\]

\emph{Step 2: concentration of the pooled score.} The pooled score is
\[
s_0(\beta;Y):=\nabla L_0^{(-k)}(\beta;Y)=\sum_{\ell\neq k}\nabla L_\ell(\beta;Y_\ell),
\]
and Lemma~\ref{lem:app_pooled_bias} places $\beta_0^\star$ in the
interior of $\Theta$. First-order optimality therefore gives
\(\mathbb E\,s_0(\beta_0^\star;Y)=\nabla L_0^\star(\beta_0^\star)=0\).
For one observation from store \(\ell\), the centered score contribution
at \(\beta_0^\star\) is
\[
\bigl(\mu_X(\beta_0^\star)-X^{\beta_\ell^\star}\bigr)
-\mathbb E\bigl[\mu_X(\beta_0^\star)-X^{\beta_\ell^\star}\bigr]
=
\mu_X(\beta_\ell^\star)-X^{\beta_\ell^\star},
\]
which has Euclidean norm at most \(2B\) under
Assumption~\ref{ass:bounded}. Unlike the single-store case, the whitened
pooled score \((H_0^\star)^{-1/2}s_0(\beta_0^\star;Y)\) is not isotropic,
because \(\Var(s_0(\beta_0^\star;Y))=\sum_{\ell\neq k}n_\ell\Sigma_X^\ell\)
differs from \(H_0^\star=S\Sigma_X(\beta_0^\star)\). The argument of
Lemma~\ref{lem:app_score_concentration} does not use isotropy, however,
only a bound on the increments together with
\(\|(H_0^\star)^{-1/2}\|_{\mathrm{op}}\le (c_-(\Delta)S)^{-1/2}\).
Repeating it verbatim with that whitening therefore yields constants
\(C_s(\Delta),c_s(\Delta)>0\), depending only on
\((B,m,M,D_\Theta,\Delta)\), with
\[
\Pr\!\left(
\|s_0(\beta_0^\star;Y)\|_{(H_0^\star)^{-1}}
\ge
C_s(\Delta)\bigl(\sqrt d+t\bigr)
\right)
\le
2e^{-c_s(\Delta)t^2},
\qquad t\ge 0 .
\]

\emph{Step 3: global concentration of the pooled estimator.} Apply
Lemma~\ref{lem:app_constrained_basic} to
$L_0^{(-k)}(\cdot;Y)$ with reference point $\beta_0^\star$. Its Hessian
at the reference point is $H_0^\star$, so
\[
\|\widehat\beta_0-\beta_0^\star\|_{H_0^\star}
\le
2e^{2BD_\Theta}
\|s_0(\beta_0^\star;Y)\|_{(H_0^\star)^{-1}}.
\]
Combining this deterministic inequality with the pooled-score tail in
Step~2 gives constants
$C_\beta(\Delta),c_\beta(\Delta)>0$, depending only on
$(B,m,M,D_\Theta,\Delta)$, such that
\begin{equation}
\label{eq:app_pooled_tail}
\Pr\!\left(
\|\widehat\beta_0-\beta_0^\star\|_{H_0^\star}
\ge
C_\beta(\Delta)\bigl(\sqrt d+t\bigr)
\right)
\le
2e^{-c_\beta(\Delta)t^2}
\qquad\text{for every } t\ge0.
\end{equation}

\emph{Step 4: from tail to moments.} Fix \(q\in\{2,3,4\}\) and let
\(R_0:=\|\widehat\beta_0-\beta_0^\star\|_{H_0^\star}\). Integrating
the full tail in \eqref{eq:app_pooled_tail} gives
\begin{align*}
\mathbb ER_0^q
&\le
C_\beta(\Delta)^q d^{q/2}
+2qC_\beta(\Delta)^q
\int_0^\infty(\sqrt d+t)^{q-1}
e^{-c_\beta(\Delta)t^2}\,dt\\
&\le
C_{\mathrm{tail}}^{(q)}(\Delta)d^{q/2},
\end{align*}
where $C_{\mathrm{tail}}^{(q)}(\Delta)$ is bounded uniformly over
$\Delta\in[0,\Delta_0]$.

\emph{Step 5: transfer to the \(k\)-local norm.} By
Assumption~\ref{ass:fisher}, \(H_k^\star=n_k\Sigma_X^k \preceq n_k M I_d\),
while Assumption~\ref{ass:ratio} gives
\(S\ge \frac{K-1}{C_{\max}}n_k\), so the lower bound of Step 1 yields
\[
H_0^\star
\;\succeq\;
\frac{c_-(\Delta)}{C_{\max}}(K-1)n_k\,I_d .
\]
Hence for every \(v\in\mathbb R^d\),
\[
\|v\|_k^2
\le
n_k M\|v\|_2^2
\le
\frac{M C_{\max}}{c_-(\Delta)(K-1)}\,\|v\|_{H_0^\star}^2 .
\]
Raising to the power \(q/2\), applying this with
\(v=\widehat\beta_0-\beta_0^\star\), and taking expectations gives the
first claim with
\[
C_{\mathrm{pool}}^{(q)}(\Delta)
:=
\left(\frac{M C_{\max}}{c_-(\Delta)}\right)^{q/2}
C_{\mathrm{tail}}^{(q)}(\Delta).
\]
Each of these is continuous in \(\Delta\) and hence bounded uniformly on
\([0,\Delta_0]\).

\emph{Step 6: the bound on \(V_k^{\mathrm{pool}}\).} Write
\[
g_k
=
\widehat\beta_0^{(-k)}-\beta_k^\star
=
(\widehat\beta_0-\beta_0^\star)+(\beta_0^\star-\beta_k^\star).
\]
Using \(\|a+b\|_k^2\le 2\|a\|_k^2+2\|b\|_k^2\),
\[
V_k^{\mathrm{pool}}
=
\mathbb E\|g_k\|_k^2
\le
2\,\mathbb E\|\widehat\beta_0-\beta_0^\star\|_k^2
+
2\,\|\beta_0^\star-\beta_k^\star\|_k^2 .
\]
The first term is bounded by the case \(q=2\) already proved. For the
second, Lemma~\ref{lem:app_pooled_bias} gives
\[
\|\beta_0^\star-\beta_k^\star\|_2
\le
\Delta\Bigl(2+C_{\mathrm{het}}e^{12B\Delta}\Delta\Bigr),
\]
so, using \(H_k^\star\preceq n_k M I_d\),
\[
\|\beta_0^\star-\beta_k^\star\|_k^2
\le
n_k M \Delta^2
\Bigl(2+C_{\mathrm{het}}e^{12B\Delta}\Delta\Bigr)^2 .
\]
Absorbing constants into
\[
C_{\mathrm{pool},1}(\Delta):=2C_{\mathrm{pool}}(\Delta),
\qquad
C_{\mathrm{pool},2}(\Delta):=
2M\Bigl(2+C_{\mathrm{het}}e^{12B\Delta}\Delta\Bigr)^2,
\]
we conclude that
\[
V_k^{\mathrm{pool}}
\le
C_{\mathrm{pool},1}(\Delta)d
+
C_{\mathrm{pool},2}(\Delta)n_k\Delta^2 .
\]
\hfill \Halmos
\end{proof}

\begin{lemma}
\label{lem:app_cross_control}
Let
\[
C_k:=\mathbb E\langle e_k,g_k\rangle_k,
\qquad
\zeta_k:=V_k^{\mathrm{loc}}-C_k,
\qquad
\zeta_k^+:=\max\{\zeta_k,0\}.
\]
Under Assumptions~\ref{ass:bounded}--\ref{ass:sample}, and whenever
$\Delta\in[0,\Delta_0]$ for $\Delta_0$ defined in
Appendix~\ref{sec:preliminaries}, there exist constants
\(c_{\mathrm{cross},0},c_{\mathrm{cross},1},c_\Delta>0\) depending only
on \((B,m,M,C_{\max},D_\Theta,\rho_\Theta)\) such that
\[
|C_k|
\le
d\Bigl(c_{\mathrm{cross},0}+c_{\mathrm{cross},1}\Delta e^{c_\Delta\Delta}\Bigr).
\]
Consequently,
\[
\zeta_k^+
\ge
c_\zeta(\Delta)d,
\qquad
c_\zeta(\Delta):=
\max\Bigl\{c_{\mathrm{loc}}-c_{\mathrm{cross},0}-c_{\mathrm{cross},1}\Delta e^{c_\Delta\Delta},0\Bigr\}.
\]
\end{lemma}

\begin{proof}{Proof of Lemma~\ref{lem:app_cross_control}.}
By Lemma~\ref{lem:app_independence},
\[
C_k
=
\mathbb E\langle e_k,g_k\rangle_k
=
\langle \mathbb E e_k,\mathbb E g_k\rangle_k.
\]
Hence, by Cauchy--Schwarz in the \(k\)-local norm,
\begin{equation}
\label{eq:app_cross_cs}
|C_k|
\le
\|\mathbb E e_k\|_k\,\|\mathbb E g_k\|_k.
\end{equation}

Lemma~\ref{lem:app_mle_bias} gives
\[
\|\mathbb E e_k\|_k
\le
C_{\mathrm{bias}}\frac{d}{\sqrt{n_k}}.
\]
For the pooled term, write
\[
g_k
=
\widehat\beta_0^{(-k)}-\beta_k^\star
=
\bigl(\widehat\beta_0^{(-k)}-\beta_0^{(-k),\star}\bigr)
+
\delta_k,
\qquad
\delta_k:=\beta_0^{(-k),\star}-\beta_k^\star.
\]
Therefore
\[
\|\mathbb E g_k\|_k
\le
\Bigl\|\mathbb E\bigl[\widehat\beta_0^{(-k)}-\beta_0^{(-k),\star}\bigr]\Bigr\|_k
+
\|\delta_k\|_k
\le
\Bigl(\mathbb E\|\widehat\beta_0^{(-k)}-\beta_0^{(-k),\star}\|_k^2\Bigr)^{1/2}
+
\|\delta_k\|_k.
\]
By Lemma~\ref{lem:app_pooled_variance},
\[
\Bigl(\mathbb E\|\widehat\beta_0^{(-k)}-\beta_0^{(-k),\star}\|_k^2\Bigr)^{1/2}
\le
\left(\frac{C_{\mathrm{pool}}(\Delta)}{K-1}d\right)^{1/2}.
\]
By Lemma~\ref{lem:app_pooled_bias},
\[
\|\delta_k\|_k
\le
\sqrt{n_k M}\left(2\Delta + C_{\mathrm{het}}e^{12B\Delta}\Delta^2\right).
\]
Combining the last three displays,
\begin{equation}
\label{eq:app_cross_gk_mean}
\|\mathbb E g_k\|_k
\le
\left(\frac{C_{\mathrm{pool}}(\Delta)}{K-1}d\right)^{1/2}
+
\sqrt{n_k M}\left(2\Delta + C_{\mathrm{het}}e^{12B\Delta}\Delta^2\right).
\end{equation}

Substituting \eqref{eq:app_cross_gk_mean} into \eqref{eq:app_cross_cs} yields
\[
|C_k|
\le
C_{\mathrm{bias}}\frac{d}{\sqrt{n_k}}
\left[
\left(\frac{C_{\mathrm{pool}}(\Delta)}{K-1}d\right)^{1/2}
+
\sqrt{n_k M}\left(2\Delta + C_{\mathrm{het}}e^{12B\Delta}\Delta^2\right)
\right].
\]
We bound the two terms separately. Since \(K\ge2\), \((K-1)^{-1/2}\le1\). Moreover, when \(\Delta\in[0,\Delta_0]\), the function \(C_{\mathrm{pool}}(\Delta)\) is bounded above by
\[
\bar C_{\mathrm{pool}}
:=
\sup_{0\le u\le \Delta_0} C_{\mathrm{pool}}(u)
<\infty,
\]
where \(\bar C_{\mathrm{pool}}\) depends only on
\((B,m,M,C_{\max},D_\Theta,\rho_\Theta)\) because \(\Delta_0\) does.
Using also Assumption~\ref{ass:sample}, \(n_k\ge C_0 d\), we obtain
\[
C_{\mathrm{bias}}\frac{d}{\sqrt{n_k}}
\left(\frac{C_{\mathrm{pool}}(\Delta)}{K-1}d\right)^{1/2}
\le
C_{\mathrm{bias}}\sqrt{\frac{\bar C_{\mathrm{pool}}}{C_0}}\,d.
\]
For the structural term, we have
\[
C_{\mathrm{bias}}\frac{d}{\sqrt{n_k}}\,
\sqrt{n_k M}\left(2\Delta + C_{\mathrm{het}}e^{12B\Delta}\Delta^2\right)
=
C_{\mathrm{bias}}\sqrt M\,d\left(2\Delta + C_{\mathrm{het}}e^{12B\Delta}\Delta^2\right).
\]
Now fix
\[
c_\Delta:=12B+1.
\]
Since \(e^{c_\Delta\Delta}\ge1\) for \(\Delta\ge0\),
\[
2\Delta \le 2\Delta e^{c_\Delta\Delta},
\]
and since \(\Delta\le e^\Delta\) for \(\Delta\ge0\),
\[
e^{12B\Delta}\Delta^2
\le
\Delta e^{(12B+1)\Delta}
=
\Delta e^{c_\Delta\Delta}.
\]
Hence
\[
2\Delta + C_{\mathrm{het}}e^{12B\Delta}\Delta^2
\le
(2+C_{\mathrm{het}})\,\Delta e^{c_\Delta\Delta}.
\]
Putting everything together,
\[
|C_k|
\le
C_{\mathrm{bias}}\sqrt{\frac{\bar C_{\mathrm{pool}}}{C_0}}\,d
+
C_{\mathrm{bias}}\sqrt M\,(2+C_{\mathrm{het}})\,d\,\Delta e^{c_\Delta\Delta}.
\]
Therefore the claimed bound holds with
\[
c_{\mathrm{cross},0}
:=
C_{\mathrm{bias}}\sqrt{\frac{\bar C_{\mathrm{pool}}}{C_0}},
\qquad
c_{\mathrm{cross},1}
:=
C_{\mathrm{bias}}\sqrt M\,(2+C_{\mathrm{het}}),
\]
which depend only on
\((B,m,M,C_{\max},D_\Theta,\rho_\Theta)\).

For the second claim, by definition,
\[
\zeta_k^+
=
\max\{V_k^{\mathrm{loc}}-C_k,0\}
\ge
\max\{V_k^{\mathrm{loc}}-|C_k|,0\}.
\]
Lemma~\ref{lem:app_Vind} gives
\[
V_k^{\mathrm{loc}} \ge c_{\mathrm{loc}}\,d.
\]
Thus
\[
\zeta_k^+
\ge
\max\Bigl\{
c_{\mathrm{loc}}d
-
d\bigl(c_{\mathrm{cross},0}+c_{\mathrm{cross},1}\Delta e^{c_\Delta\Delta}\bigr),
\,0
\Bigr\}
=
c_\zeta(\Delta)\,d,
\]
where
\[
c_\zeta(\Delta)
:=
\max\Bigl\{
c_{\mathrm{loc}}-c_{\mathrm{cross},0}-c_{\mathrm{cross},1}\Delta e^{c_\Delta\Delta},
\,0
\Bigr\}.
\]
This proves the lemma. \hfill \Halmos
\end{proof}

\begin{lemma}
\label{lem:app_weight_certificate}
Let
\[
\alpha_k:=V_k^{\mathrm{loc}}+V_k^{\mathrm{pool}}-2C_k,
\qquad
\zeta_k:=V_k^{\mathrm{loc}}-C_k,
\qquad
\bar\delta_k:=\mathbb E g_k,
\]
and define
\[
w_k^\star=
\begin{cases}
0, & \zeta_k\le0,\\[4pt]
\dfrac{\zeta_k}{\alpha_k+\zeta_k}, & \zeta_k>0.
\end{cases}
\]
Then:
\begin{enumerate}[label=(\roman*),leftmargin=*]
\item
\[
Q_k(w_k^\star)
\le
- \frac{1}{2} \frac{(\zeta_k^+)^2}{\alpha_k+\zeta_k^+}.
\]

\item
\[
w_k^\star\|\bar\delta_k\|_k
\le
\sqrt{V_k^{\mathrm{loc}}}.
\]

\item
If
\[
\delta_k:=\beta_0^{(-k),\star}-\beta_k^\star,
\qquad
b_k^{\mathrm{pool}}
:=
\mathbb E[
\widehat\beta_0^{(-k)}-\beta_0^{(-k),\star}
],
\]
then
\[
w_k^\star\|\delta_k\|_k
\le
\sqrt{V_k^{\mathrm{loc}}}
+
\|b_k^{\mathrm{pool}}\|_k.
\]
\end{enumerate}
\end{lemma}

\begin{proof}{Proof of Lemma~\ref{lem:app_weight_certificate}.}
For part (i), substitute \(w_k^\star\) into the quadratic identity from
Lemma~\ref{lem:theory_quadratic_identity}. If \(\zeta_k\le0\), then
\(w_k^\star=0\) and both sides are zero. If \(\zeta_k>0\), then
\[
\frac12 \left( \alpha_k(w_k^\star)^2-2\zeta_k w_k^\star\right)
=
- \frac12 \frac{\zeta_k^2(\alpha_k+2\zeta_k)}{(\alpha_k+\zeta_k)^2}
\le
- \frac12 \frac{\zeta_k^2}{\alpha_k+\zeta_k},
\]
which proves the stated quadratic-improvement bound.

For part (ii), use the centered decomposition
\[
g_k=\xi_k+\bar\delta_k,
\qquad
\bar\delta_k:=\mathbb E g_k,
\qquad
\xi_k:=g_k-\mathbb E g_k,
\qquad
\mathbb E\xi_k=0.
\]
By the leave-one-out construction, \(e_k\) is independent of \(g_k\),
and therefore independent of \(\xi_k\). Hence
\[
C_k
=
\mathbb E\langle e_k,g_k\rangle_k
=
\mathbb E\langle e_k,\xi_k+\bar\delta_k\rangle_k
=
\langle \mathbb E e_k,\bar\delta_k\rangle_k .
\]
By Cauchy--Schwarz,
\[
|C_k|
\le
\|\mathbb E e_k\|_k\,\|\bar\delta_k\|_k
\le
\sqrt{V_k^{\mathrm{loc}}}\,\|\bar\delta_k\|_k .
\]
Also, since \(\mathbb E\xi_k=0\),
\[
V_k^{\mathrm{pool}}
=
\mathbb E\|g_k\|_k^2
=
\mathbb E\|\xi_k+\bar\delta_k\|_k^2
=
\mathbb E\|\xi_k\|_k^2+\|\bar\delta_k\|_k^2 .
\]

Set
\[
V:=V_k^{\mathrm{loc}},
\qquad
U:=\mathbb E\|\xi_k\|_k^2,
\qquad
A:=\|\bar\delta_k\|_k,
\qquad
D:=C_k .
\]
Then
\[
\zeta_k=V-D,
\qquad
\alpha_k=V+U+A^2-2D.
\]
If \(V\le D\), then \(\zeta_k\le0\), so \(w_k^\star=0\) and there is
nothing to prove. Suppose instead that \(V>D\). Then \(V>0\), and
\[
w_k^\star A
=
A\,\frac{V-D}{\alpha_k+\zeta_k}
=
A\,\frac{V-D}{2V+U+A^2-3D}.
\]
By the bound \(|D|\le \sqrt V A\), writing
\[
x:=\frac{A}{\sqrt V},
\qquad
y:=\frac{D}{V},
\]
gives \(|y|\le x\) and \(y<1\). In particular,
\[
2+x^2-3y>0.
\]
Indeed, if \(0\le x\le1\), then
\[
2+x^2-3y \ge 2+x^2-3x=(1-x)(2-x)>0,
\]
while if \(x\ge1\), then \(y<1\) implies
\[
2+x^2-3y> x^2-1\ge0,
\]
with strict positivity also at \(x=1\) because \(y<1\). Hence
\(2V+A^2-3D>0\), and since \(U\ge0\),
\[
w_k^\star A
\le
A\,\frac{V-D}{2V+A^2-3D}.
\]
The desired inequality \(w_k^\star A\le \sqrt V\) is therefore implied by
\[
\frac{x(1-y)}{2+x^2-3y}\le 1.
\]
After rearrangement, this becomes
\[
2+x^2-x+y(x-3)\ge0.
\]
If \(x\ge3\), then \(x-3\ge0\) and \(y\ge -x\), so the left-hand side is
at least
\[
2+x^2-x-x(x-3)=2+2x\ge0.
\]
If \(0\le x\le3\), then \(x-3\le0\) and \(y\le x\), so the left-hand
side is at least
\[
2+x^2-x+x(x-3)=2(x-1)^2\ge0.
\]
Thus
\[
w_k^\star\|\bar\delta_k\|_k
\le
\sqrt{V_k^{\mathrm{loc}}},
\]
which proves part (ii).

For part (iii), note that
\[
\bar\delta_k
=
\mathbb E g_k
=
\mathbb E\widehat\beta_0^{(-k)}-\beta_k^\star
=
\delta_k
+
b_k^{\mathrm{pool}}.
\]
Therefore,
\[
\|\delta_k\|_k
\le
\|\bar\delta_k\|_k+\|b_k^{\mathrm{pool}}\|_k.
\]
Multiplying by \(w_k^\star\le1\) and applying part (ii) gives
\[
w_k^\star\|\delta_k\|_k
\le
\sqrt{V_k^{\mathrm{loc}}}
+
\|b_k^{\mathrm{pool}}\|_k.
\]
This proves part (iii). \hfill \Halmos
\end{proof}

\begin{lemma}
\label{lem:app_remainder_control}
Under Assumptions~\ref{ass:bounded}--\ref{ass:sample}, and whenever $\Delta \in [0,\Delta_0]$ for $\Delta_0$ defined in Appendix~\ref{sec:preliminaries}, there
exists a constant \(C_r>0\), depending only on
\((B,m,M,C_{\max},D_\Theta,\rho_\Theta)\), such that
\[
\mathbb E|r_k(\widehat\beta_k^{\mathrm{shr}})|
\le C_r\frac{d^{3/2}}{\sqrt{n_k}},
\qquad
\mathbb E|r_k(\widehat\beta_k)|
\le C_r\frac{d^{3/2}}{\sqrt{n_k}}.
\]
\end{lemma}

\begin{proof}{Proof of Lemma~\ref{lem:app_remainder_control}.}
Let
\[
z_k
:=
\widehat\beta_k^{\mathrm{shr}}(w_k^\star)-\beta_k^\star
=
(1-w_k^\star)e_k+w_k^\star g_k
\]
denote the displacement of the shrinkage estimator from
\(\beta_k^\star\). Split the leave-one-out pooled error into its
centered and systematic parts,
\[
g_k=\xi_k+\bar\delta_k,
\qquad
\bar\delta_k:=\mathbb E g_k,
\qquad
\xi_k:=g_k-\mathbb E g_k,
\]
so that
\[
z_k
=
\underbrace{(1-w_k^\star)e_k+w_k^\star \xi_k}_{=:N_k}
+
\underbrace{w_k^\star\bar\delta_k}_{=:b_k}.
\]
The weight $w_k^\star$ is deterministic because it is defined through
population moments. Lemmas~\ref{lem:app_Vind} and
\ref{lem:app_pooled_variance} give third moments of order $d^{3/2}$
for $e_k$ and for
$\varepsilon_k^{\mathrm{pool}}
:=\widehat\beta_0^{(-k)}-\beta_0^{(-k),\star}$. Centering changes the
corresponding $L^3$ norm by at most a factor of two, so the same order
holds for $\xi_k$. Minkowski's inequality therefore gives
\[
\bigl(\mathbb E\|z_k\|_k^3\bigr)^{1/3}
\le
(1-w_k^\star)\bigl(\mathbb E\|e_k\|_k^3\bigr)^{1/3}
+
w_k^\star\bigl(\mathbb E\|\xi_k\|_k^3\bigr)^{1/3}
+
\|b_k\|_k .
\]
Lemma~\ref{lem:app_weight_certificate}(ii) and
Lemma~\ref{lem:app_Vind} bound the last term as
\[
\|b_k\|_k
=
w_k^\star\|\bar\delta_k\|_k
\le
\sqrt{V_k^{\mathrm{loc}}}
\le
\sqrt{C_{\mathrm{loc}}\,d}.
\]
Consequently, for a constant $C_3$ uniform over
$\Delta\in[0,\Delta_0]$,
\begin{equation}
\label{eq:app_rem_moments}
\mathbb E\|z_k\|_k^3\le C_3\,d^{3/2}.
\end{equation}

It remains to use the bounded parameter region in the pointwise SCL
remainder bound. Define
\[
C_{\Theta,\omega}
:=
\sup_{0<a\le2BD_\Theta}
\frac{\omega_{\mathrm{SCL}}(a)}{a},
\]
which is finite because
$\omega_{\mathrm{SCL}}(a)/a\to1/3$ as $a\downarrow0$. The proof of
Lemma~\ref{lem:theory_local_expansion}, before replacing the Euclidean
displacement by the local norm, gives
\[
|r_k(\widehat\beta_k^{\mathrm{shr}})|
\le
\frac12\|z_k\|_k^2
\omega_{\mathrm{SCL}}(2B\|z_k\|_2).
\]
Both $\widehat\beta_k^{\mathrm{shr}}$ and $\beta_k^\star$ belong to
$\Theta$, so $2B\|z_k\|_2\le2BD_\Theta$. Moreover,
$\|z_k\|_2\le\|z_k\|_k/\sqrt{n_km}$. It follows that
\[
|r_k(\widehat\beta_k^{\mathrm{shr}})|
\le
\frac{BC_{\Theta,\omega}}{\sqrt m}
\frac{\|z_k\|_k^3}{\sqrt{n_k}}.
\]
Taking expectations and using \eqref{eq:app_rem_moments} proves the
first claim. The same pointwise argument with $z_k$ replaced by $e_k$,
together with the case $q=3$ of Lemma~\ref{lem:app_Vind}, proves the
second. \hfill \Halmos
\end{proof}

\section{Conditions on heterogeneity and $C_0$}
\label{app:barDelta}

The reference radius $\Delta_0>0$ fixed in
Appendix~\ref{sec:preliminaries} ensures that all pooled population
targets retain a fixed interior margin and that the $\Delta$-dependent
constants in Lemmas~\ref{lem:app_pooled_bias}--
\ref{lem:app_cross_control} are uniform on $[0,\Delta_0]$. The proof of
Theorem~\ref{thm:theory_main} requires the additional condition
$c_\zeta(\Delta)>0$, which holds when
\begin{equation*}%
c_{\mathrm{cross},0} + c_{\mathrm{cross},1}\,\delta\,e^{c_\Delta\delta}
\;<\; c_{\mathrm{loc}}.
\end{equation*}
Therefore we set the restricted radius
\[
\bar\Delta
:=
\sup\Bigl\{
\delta\in[0,\Delta_0]:
c_{\mathrm{cross},0}
+ c_{\mathrm{cross},1}\delta e^{c_\Delta\delta}
\le \tfrac12 c_{\mathrm{loc}}
\Bigr\},
\]
which ensures $c_\zeta(\Delta)\ge\tfrac12c_{\mathrm{loc}}>0$ for
$\Delta\in[0,\bar\Delta]$. The left-hand side is continuous and
increasing in $\delta$ and equals $c_{\mathrm{cross},0}$ at zero. Thus
$\bar\Delta>0$ whenever
$c_{\mathrm{cross},0}<\tfrac12c_{\mathrm{loc}}$. Since
\[
c_{\mathrm{cross},0}
=
C_{\mathrm{bias}}\sqrt{\bar C_{\mathrm{pool}}/C_0},
\qquad
\bar C_{\mathrm{pool}}
:=
\sup_{0\le\delta\le\Delta_0}C_{\mathrm{pool}}(\delta),
\]
this strict inequality follows by taking the normalized sample-size
floor $C_0$ sufficiently large.

The same floor also ensures that the score event used for the lower
bound in Lemma~\ref{lem:app_Vind} places the constrained estimator in
the interior of $\Theta$, and it yields the exponentially small
boundary probability used in Lemma~\ref{lem:app_mle_bias}. Throughout,
Assumption~\ref{ass:sample} is understood with a value of $C_0$
admissible in the sense of the next lemma.

\begin{lemma}
\label{lem:app_constants}
There exist $C_0\ge1$ and $\bar\Delta>0$, depending only on
$(B,m,M,C_{\max},D_\Theta,\rho_\Theta)$, that satisfy the requirements
above simultaneously.
\end{lemma}
\begin{proof}{Proof of Lemma~\ref{lem:app_constants}.}
Let $C_0^{\mathrm{int}}$ be large enough for the two interior-event
requirements in Lemmas~\ref{lem:app_Vind} and
\ref{lem:app_mle_bias}. It depends only on
$(B,m,D_\Theta,\rho_\Theta)$. Once
$C_0\ge C_0^{\mathrm{int}}$, the constants
$c_{\mathrm{loc}}$, $C_{\mathrm{bias}}$, and
$\bar C_{\mathrm{pool}}$ may be fixed independently of any further
increase in $C_0$, whereas
$c_{\mathrm{cross},0}=C_{\mathrm{bias}}\sqrt{\bar C_{\mathrm{pool}}/C_0}$
is decreasing in $C_0$. Taking
\[
C_0
:=
\max\!\left\{
C_0^{\mathrm{int}},\;
\frac{16C_{\mathrm{bias}}^2\bar C_{\mathrm{pool}}}{c_{\mathrm{loc}}^2},\;
1
\right\}
\]
gives $c_{\mathrm{cross},0}\le\tfrac14c_{\mathrm{loc}}$, and hence the
required strict inequality. Therefore
$\bar\Delta>0$ follows since the left-hand side of the defining
inequality is continuous and increasing in $\delta$. \hfill \Halmos
\end{proof}

\section{Proofs of the Main Finite-Sample Results}
\label{app:main_results}

\begin{proof}{Proof of Lemma~\ref{lem:theory_weight_properties}.}
Part~(i) is Lemma~\ref{lem:app_weight_certificate}(i), and part~(ii) is
Lemma~\ref{lem:app_weight_certificate}(ii).
\hfill \Halmos
\end{proof}

\begin{proof}{Proof of Theorem~\ref{thm:theory_main}.}
Let
\[
\widehat\beta_k^{\mathrm{shr}}=(1-w_k^\star)\widehat\beta_k+w_k^\star\widehat\beta_0^{(-k)}.
\]
Applying Lemma~\ref{lem:theory_local_expansion} to both $\widehat\beta_k^{\mathrm{shr}}$ and $\widehat\beta_k$ gives
\begin{align*}
\mathbb E\Big[L_k^\star(\widehat\beta_k^{\mathrm{shr}})-L_k^\star(\widehat\beta_k)\Big]
&=
\frac12\mathbb E\Big[\|(1-w_k^\star)e_k+w_k^\star g_k\|_k^2-\|e_k\|_k^2\Big]\\
&\quad +\mathbb E r_k(\widehat\beta_k^{\mathrm{shr}})-\mathbb E r_k(\widehat\beta_k).
\end{align*}
By Lemma~\ref{lem:theory_weight_properties}(i),
\[
\frac12\mathbb E\Big[\|(1-w_k^\star)e_k+w_k^\star g_k\|_k^2-\|e_k\|_k^2\Big]
\le
-\frac12\frac{(\zeta_k^+)^2}{\alpha_k+\zeta_k^+}.
\]
By Lemma~\ref{lem:app_remainder_control}, the difference of the two remainder terms is bounded in absolute value by $C_r d^{3/2}/\sqrt{n_k}$.

It therefore remains to control the ratio $(\zeta_k^+)^2/(\alpha_k+\zeta_k^+)$.  Lemma~\ref{lem:app_cross_control} gives the lower bound
\[
\zeta_k^+\ge c_\zeta(\Delta)d,
\]
throughout the moderate-heterogeneity region.  On the other hand,
\[
\alpha_k+\zeta_k^+=2V_k^{\mathrm{loc}}+V_k^{\mathrm{pool}}-3C_k,
\]
and Lemmas~\ref{lem:app_Vind}, \ref{lem:app_pooled_variance}, and~\ref{lem:app_cross_control} imply
\[
\alpha_k+\zeta_k^+
\le C_1(\Delta)d+C_2(\Delta)n_k\Delta^2
\]
for suitable positive functions $C_1(\Delta)$ and $C_2(\Delta)$
depending only on
$(B,m,M,C_{\max},D_\Theta,\rho_\Theta,\Delta)$. Hence
\[
\frac{(\zeta_k^+)^2}{\alpha_k+\zeta_k^+}
\ge
\frac{c_\zeta(\Delta)^2 d^2}{C_1(\Delta)d+C_2(\Delta)n_k\Delta^2}.
\]
Combining the displays above gives
\[
\mathbb E\Big[L_k^\star(\widehat\beta_k^{\mathrm{shr}})-L_k^\star(\widehat\beta_k)\Big]
\le
-\frac{c_\zeta(\Delta)^2 d^2}{C_1(\Delta)d+C_2(\Delta)n_k\Delta^2}
+c_6\frac{d^{3/2}}{\sqrt{n_k}},
\]
for a constant $c_6>0$ depending only on
$(B,m,M,C_{\max},D_\Theta,\rho_\Theta)$. Renaming the coefficient
functions and restricting to the certified moderate-heterogeneity
region $\Delta\in[0,\bar\Delta]$ yields the form stated in
Theorem~\ref{thm:theory_main}. Summing over $k$ proves the global bound.
\hfill \Halmos
\end{proof}

\begin{proof}{Proof of Corollary~\ref{cor:theory_weight_scale}.}
In the certified moderate-heterogeneity region,
Lemma~\ref{lem:app_cross_control} implies
\[
\zeta_k^+\ge c_\zeta(\Delta)d,
\]
with \(c_\zeta(\Delta)>0\). Hence \(\zeta_k>0\) and
\[
w_k^\star=\frac{\zeta_k}{\alpha_k+\zeta_k}.
\]

We first prove the lower bound on \(w_k^\star\). By
Lemma~\ref{lem:app_Vind} and Lemma~\ref{lem:app_cross_control},
\[
V_k^{\mathrm{loc}}+|C_k|\le Cd.
\]
Write
\[
g_k
=
\widehat\beta_0^{(-k)}-\beta_k^\star
=
\bigl(\widehat\beta_0^{(-k)}-\beta_0^{(-k),\star}\bigr)
+
\bigl(\beta_0^{(-k),\star}-\beta_k^\star\bigr)
=
\varepsilon_k^{\mathrm{pool}}+\delta_k,
\]
where
\[
\delta_k:=\beta_0^{(-k),\star}-\beta_k^\star,
\qquad
\Delta_k=\|\delta_k\|_2 .
\]
By the pooled-estimation moment bound and the balanced-sample
assumption,
\[
\mathbb E\|\varepsilon_k^{\mathrm{pool}}\|_k^2\le Cd.
\]
Therefore,
\[
V_k^{\mathrm{pool}}
=
\mathbb E\|\varepsilon_k^{\mathrm{pool}}+\delta_k\|_k^2
\le
2\mathbb E\|\varepsilon_k^{\mathrm{pool}}\|_k^2
+
2\|\delta_k\|_k^2
\le
Cd+Cn_k\Delta_k^2,
\]
where the last inequality uses \(H_k^\star\preceq n_k M I_d\). Hence
\[
\alpha_k+\zeta_k
=
2V_k^{\mathrm{loc}}+V_k^{\mathrm{pool}}-3C_k
\le
C(d+n_k\Delta_k^2).
\]
Taking the infimum of \(c_\zeta(\Delta)\) over
\(\Delta\in[0,\bar\Delta]\) gives
\[
w_k^\star
=
\frac{\zeta_k}{\alpha_k+\zeta_k}
\ge
c_-\,\frac{d}{d+n_k\Delta_k^2}.
\]

We now prove the upper bound. First,
\[
\zeta_k
\le
V_k^{\mathrm{loc}}+|C_k|
\le
Cd.
\]
It remains to lower bound the denominator. In the certified
moderate-heterogeneity region, Lemma~\ref{lem:app_Vind} gives
\[
V_k^{\mathrm{loc}}\ge c_{\mathrm{loc}}d,
\]
while Lemma~\ref{lem:app_cross_control} and the definition of
\(\bar\Delta\) give
\[
|C_k|\le \frac12 c_{\mathrm{loc}}d.
\]
Therefore,
\[
2V_k^{\mathrm{loc}}-3C_k
\ge
2V_k^{\mathrm{loc}}-3|C_k|
\ge
\frac12 c_{\mathrm{loc}}d.
\]
Thus
\[
\alpha_k+\zeta_k
=
2V_k^{\mathrm{loc}}+V_k^{\mathrm{pool}}-3C_k
\ge
cd+V_k^{\mathrm{pool}}.
\]

We next lower bound \(V_k^{\mathrm{pool}}\). Since
\(g_k=\varepsilon_k^{\mathrm{pool}}+\delta_k\),
\[
V_k^{\mathrm{pool}}
=
\mathbb E\|\varepsilon_k^{\mathrm{pool}}+\delta_k\|_k^2.
\]
Using the inequality
\[
\|a+b\|^2\ge \frac12\|b\|^2-\|a\|^2,
\]
we obtain
\[
V_k^{\mathrm{pool}}
\ge
\frac12\|\delta_k\|_k^2
-
\mathbb E\|\varepsilon_k^{\mathrm{pool}}\|_k^2.
\]
The pooled-estimation moment bound gives
\[
\mathbb E\|\varepsilon_k^{\mathrm{pool}}\|_k^2\le Cd,
\]
and Assumption~\ref{ass:fisher} gives
\[
\|\delta_k\|_k^2
=
\delta_k^\top H_k^\star \delta_k
\ge
m n_k\|\delta_k\|_2^2
=
m n_k\Delta_k^2.
\]

We consider two cases. If
\[
n_k\Delta_k^2\le A d
\]
for a sufficiently large constant \(A\), then
\[
d+n_k\Delta_k^2\le C d,
\]
and the denominator lower bound
\[
\alpha_k+\zeta_k\ge cd
\]
implies
\[
\alpha_k+\zeta_k\ge c'(d+n_k\Delta_k^2).
\]

If instead
\[
n_k\Delta_k^2>A d,
\]
with \(A\) chosen large enough, then
\[
V_k^{\mathrm{pool}}
\ge
\frac12 m n_k\Delta_k^2-Cd
\ge
c n_k\Delta_k^2.
\]
Combining this with
\[
\alpha_k+\zeta_k\ge cd+V_k^{\mathrm{pool}},
\]
we again obtain
\[
\alpha_k+\zeta_k
\ge
c'(d+n_k\Delta_k^2).
\]
Thus, in all cases,
\[
\alpha_k+\zeta_k
\ge
c'(d+n_k\Delta_k^2).
\]
Together with \(\zeta_k\le Cd\), this gives
\[
w_k^\star
=
\frac{\zeta_k}{\alpha_k+\zeta_k}
\le
c_+\,\frac{d}{d+n_k\Delta_k^2}.
\]
This proves the two-sided store-specific order bound.
\hfill \Halmos
\end{proof}

\begin{proof}{Proof of Lemma~\ref{lem:theory_plugin_behavior}.}
Recall that
\[
w_k^\circ(\lambda)
=
\frac{d}{d+\lambda n_k\Delta_k^2},
\qquad
\widehat w_k(\lambda)
=
\frac{d}{d+\lambda n_k\widehat\Delta_k^2}.
\]

First suppose \(\Delta_k>0\). Since
\(\widehat\Delta_k\xrightarrow{p}\Delta_k\), we have
\(\widehat\Delta_k^2\xrightarrow{p}\Delta_k^2>0\). Hence
\[
\frac{n_k}{d}\widehat\Delta_k^2
\xrightarrow{p}
\infty,
\]
because \(n_k/d\to\infty\). Therefore
\[
\widehat w_k(\lambda)
=
\frac{1}{1+\lambda(n_k/d)\widehat\Delta_k^2}
\xrightarrow{p}0.
\]
Moreover,
\[
\frac{\widehat w_k(\lambda)}{w_k^\circ(\lambda)}
=
\frac{d+\lambda n_k\Delta_k^2}
     {d+\lambda n_k\widehat\Delta_k^2}
=
\frac{d/n_k+\lambda\Delta_k^2}
     {d/n_k+\lambda\widehat\Delta_k^2}
\xrightarrow{p}
\frac{\lambda\Delta_k^2}{\lambda\Delta_k^2}
=
1,
\]
where the last step uses \(d/n_k\to0\) and
\(\widehat\Delta_k^2\xrightarrow{p}\Delta_k^2>0\).

Now suppose \(\Delta_k=0\). Then
\[
w_k^\circ(\lambda)=1,
\qquad
\widehat w_k(\lambda)
=
\frac{1}{1+\lambda(n_k/d)\widehat\Delta_k^2}.
\]
Thus
\[
\widehat w_k(\lambda)\xrightarrow{p}1
\]
whenever
\[
\frac{n_k}{d}\widehat\Delta_k^2\xrightarrow{p}0.
\]
This condition is equivalent to
\[
\widehat\Delta_k=o_p\!\left(\sqrt{\frac{d}{n_k}}\right).
\]

Finally, consistency of \(\widehat\Delta_k\) alone is not sufficient in
the boundary case \(\Delta_k=0\). For example, if
\(\widehat\Delta_k^2=d/n_k\), then
\(\widehat\Delta_k\to0\), but
\[
\widehat w_k(\lambda)
=
\frac{1}{1+\lambda},
\]
which does not converge to \(1\). This proves the claim.
\hfill \Halmos
\end{proof}

\section{Proofs of the Operational Regret Bounds}
\label{app:operational}

\begin{lemma}
\label{lem:app_assortment_lipschitz}
Under Assumption~\ref{ass:bounded}, for every $S\subseteq\mathcal N$ and all $\beta,\beta'\in\mathbb R^d$,
\[
|R(S;\beta)-R(S;\beta')|
\le 2p_{\max}B\,\|\beta-\beta'\|_2.
\]
\end{lemma}

\begin{proof}{Proof of Lemma~\ref{lem:app_assortment_lipschitz}.}
For $j\in S$,
\[
\nabla_\beta P_j(S;\beta)
=
P_j(S;\beta)\left(X_j-\sum_{i\in S}P_i(S;\beta)X_i\right).
\]
Hence $\|\nabla_\beta P_j(S;\beta)\|_2\le 2B P_j(S;\beta)$ by Assumption~\ref{ass:bounded}.  Summing with weights $p_j\le p_{\max}$ gives
\[
\|\nabla_\beta R(S;\beta)\|_2
\le
2p_{\max}B,
\]
and the mean value theorem yields the claim.
\hfill \Halmos
\end{proof}

\begin{proof}{Proof of Proposition~\ref{prop:assortment_regret_hp}.}
Let $\widehat S_k:=\arg\max_S R(S;\widehat\beta_k^{\mathrm{shr}})$.  Then
\begin{align*}
\mathrm{Regret}_k^{\mathrm{assort}}(w_k)
&=R(S_k^\star;\beta_k^\star)-R(\widehat S_k;\beta_k^\star)\\
&=\bigl(R(S_k^\star;\beta_k^\star)-R(S_k^\star;\widehat\beta_k^{\mathrm{shr}})\bigr)
+\bigl(R(S_k^\star;\widehat\beta_k^{\mathrm{shr}})-R(\widehat S_k;\widehat\beta_k^{\mathrm{shr}})\bigr)
+\bigl(R(\widehat S_k;\widehat\beta_k^{\mathrm{shr}})-R(\widehat S_k;\beta_k^\star)\bigr)\\
&\le 2\sup_{S\subseteq\mathcal N}|R(S;\widehat\beta_k^{\mathrm{shr}})-R(S;\beta_k^\star)|\\
&\le 4p_{\max}B\,\|\widehat\beta_k^{\mathrm{shr}}-\beta_k^\star\|_2,
\end{align*}
where the last inequality uses Lemma~\ref{lem:app_assortment_lipschitz}.

It therefore remains to control the Euclidean estimation error.  Using
\[
\widehat\beta_k^{\mathrm{shr}}=(1-w_k)\widehat\beta_k+w_k\widehat\beta_0^{(-k)},
\]
and the triangle inequality,
\[
\|\widehat\beta_k^{\mathrm{shr}}-\beta_k^\star\|_2
\le
(1-w_k)\|\widehat\beta_k-\beta_k^\star\|_2
+w_k\|\widehat\beta_0^{(-k)}-\beta_0^{(-k),\star}\|_2
+w_k\|\beta_0^{(-k),\star}-\beta_k^\star\|_2.
\]
The first two terms are controlled by the individual and leave-one-out pooled concentration bounds, which yield the radii $r_{\mathrm{loc}}(\delta)$ and $r_{\mathrm{pool}}(\delta)$.  The third term is exactly the pooled structural bias, which is bounded by Lemma~\ref{lem:app_pooled_bias}.  Substituting these bounds into the display above and then into the regret reduction inequality proves both the generic high-probability bound and the explicit moderate-heterogeneity version.
\hfill \Halmos
\end{proof}

\begin{lemma}
\label{lem:app_pricing_lipschitz}
Fix an offered set $S\subseteq\mathcal N$. Under
Assumptions~\ref{ass:bounded}--\ref{ass:fisher}, and under the bounded
price set in Section~\ref{subsec:operational_pricing}, there exists a
constant $C_p>0$ depending only on $S$, $p_{\max}$, $B$, and the
primitive regularity quantities such that
\[
|\Pi_k(p; \beta)-\Pi_k(p; \beta')|
\le C_p\|\beta-\beta'\|_2
\qquad
\text{for all }p\in\mathcal P,\quad \beta,\beta'\in\Theta.
\]
\end{lemma}

\begin{proof}{Proof of Lemma~\ref{lem:app_pricing_lipschitz}.}
For fixed $S$ and $p$, the mapping $\beta\mapsto \Pi_k(p;\beta)$ is continuously differentiable on a compact parameter region and a compact feasible price region.  The MNL probabilities are smooth, and the bounded-feature and bounded-price assumptions imply that the gradient of $\Pi_k$ with respect to $\beta$ is uniformly bounded over $p\in\mathcal P$.  Taking
\[
C_p:=\sup_{p\in\mathcal P,\,\beta\in\Theta}\|\nabla_\beta \Pi_k(p;\beta)\|_2
\]
therefore yields the displayed Lipschitz bound by the mean value theorem.
\hfill \Halmos
\end{proof}

\begin{proof}{Proof of Proposition~\ref{prop:pricing_regret_hp}.}
Let \(p_k^\star\) denote the oracle price vector and \(\widehat p_k\)
the plug-in price vector based on
\(
\widehat\beta_k^{\mathrm{shr}}(w_k)
=
(1-w_k)\widehat\beta_k+w_k\widehat\beta_0^{(-k)}.
\) Then
\begin{align*}
\mathrm{Regret}_k^{\mathrm{price}}(w_k)
&=\Pi_k(p_k^\star;\beta_k^\star)-\Pi_k(\widehat p_k;\beta_k^\star)\\
&\le 2\sup_{p\in\mathcal P}|\Pi_k(p;\widehat\beta_k^{\mathrm{shr}}(w_k))-\Pi_k(p;\beta_k^\star)|\\
&\le 2C_p\|\widehat\beta_k^{\mathrm{shr}}(w_k)-\beta_k^\star\|_2,
\end{align*}
where the second step follows exactly as in the assortment proof and the third uses Lemma~\ref{lem:app_pricing_lipschitz}.

The Euclidean parameter error is bounded by the same decomposition used in the assortment proof:
\[
\|\widehat\beta_k^{\mathrm{shr}}(w_k)-\beta_k^\star\|_2
\le
(1-w_k)\|\widehat\beta_k-\beta_k^\star\|_2
+w_k\|\widehat\beta_0^{(-k)}-\beta_0^{(-k),\star}\|_2
+w_k\|\beta_0^{(-k),\star}-\beta_k^\star\|_2.
\]
The individual and leave-one-out pooled terms are controlled by the concentration radii $r_{\mathrm{loc}}(\delta)$ and $r_{\mathrm{pool}}(\delta)$, while the structural mismatch term is bounded by the same pooled-bias argument used above.  Substituting these bounds yields the displayed high-probability inequality.  The explicit moderate-heterogeneity version follows by inserting the bound
\[
\|\beta_0^{(-k),\star}-\beta_k^\star\|_2
\le C_\Delta(\Delta+\Delta^2).
\]
\hfill \Halmos
\end{proof}

\section{Additional Details for Computational Experiments}
\label{app:computational_details}

This appendix collects implementation details for the simulation study and the retailer experiment. The main text focuses on the comparative statics in cross-store heterogeneity \(\Delta\) and local sample size \(n\), together with the real-data featured-assortment exercise. The purpose of the present appendix is to document the underlying computational design in enough detail to make the numerical evidence easy to interpret and replicate.

\subsection{Synthetic data-generating process}

All synthetic experiments are based on the same multi-store MNL environment. Each replication generates a network of \(K\) stores, a common catalog of \(N\) products, and \(d\)-dimensional product-feature vectors. These vectors contain an intercept term, \(d-2\) nonprice features, and a negative price coordinate, so each store's ground-truth coefficient vector \(\beta_k^\star\in\mathbb R^d\) contains the corresponding intercept, nonprice coefficients, and price-sensitivity parameter. The outside option is included in the MNL choice set.

The baseline finite-sample regime is intentionally nontrivial. Unless otherwise noted, we use \(K=10\) stores, \(N=15\) products, parameter dimension \(d=7\) with five nonprice features, and \(n_{\mathrm{train}}=50\) local observations per store. For the synthetic assortment problem, the observed training data are generated from offered sets of size \(m_{\mathrm{offer}}=5\), while the exact downstream assortment optimization uses a small capacity parameter \(m_{\max}=3\). The baseline price-sensitivity scale is set to \(\alpha=1.0\).

Cross-store heterogeneity is introduced directly at the parameter level. For each store $k \in \{1, \dots, K\}$, the true local parameter $\beta_k^\star$ is generated by applying a controlled perturbation to the global base:
\[
\beta_k^\star \;=\; \beta_c^\star \;+\; \nu_k,
\]
where the perturbation direction $\nu_k \in \mathbb{R}^d$ is drawn uniformly from the unit sphere in $\mathbb R^d$, and then scaled such that its magnitude equals a prescribed heterogeneity radius, $\|\nu_k\|_2 = \Delta$.

Within each replication, we compare five estimators: the store-specific MLE,
the fully pooled MLE, the realization-specific oracle shrinkage benchmark
described below, the analytic plug-in shrinkage estimator, and a
cross-validated shrinkage estimator.

\subsection{Realization-specific oracle shrinkage}
\label{app:oracle_shrinkage}

The oracle benchmark follows the damped construction in
\eqref{eq:theory_analytic_weight}, using the known simulation truth to form
realized error terms. For store \(k\) in replication \(r\), let
\[
e_{k,r}:=\widehat\beta_{k,r}-\beta_{k,r}^\star,
\qquad
g_{k,r}:=\widehat\beta_{0,r}^{(-k)}-\beta_{k,r}^\star,
\]
where \(\widehat\beta_{0,r}^{(-k)}\) is fitted using all stores except
\(k\). We compute the Euclidean quadratic terms
\begin{equation}
\widetilde\alpha_{k,r}
:=\|e_{k,r}-g_{k,r}\|_2^2,
\qquad
\widetilde\zeta_{k,r}
:=e_{k,r}^\top(e_{k,r}-g_{k,r}).
\label{eq:app_oracle_coefficients}
\end{equation}
These have the same gain-to-cost decomposition as
\eqref{eq:theory_alpha_zeta}: \(\widetilde\alpha_{k,r}
=\|e_{k,r}\|_2^2+\|g_{k,r}\|_2^2-2e_{k,r}^\top g_{k,r}\)
and \(\widetilde\zeta_{k,r}
=\|e_{k,r}\|_2^2-e_{k,r}^\top g_{k,r}\).
Applying the same damping rule gives
\begin{equation}
w_{k,r}^{\mathrm{or}}
:=
\begin{cases}
0, & \widetilde\zeta_{k,r}\le0,\\[4pt]
\dfrac{\widetilde\zeta_{k,r}}
{\widetilde\alpha_{k,r}+\widetilde\zeta_{k,r}},
& \widetilde\zeta_{k,r}>0.
\end{cases}
\label{eq:app_oracle_weight}
\end{equation}
The resulting estimate is
\(\widehat\beta_{k,r}^{\mathrm{or}}
=(1-w_{k,r}^{\mathrm{or}})\widehat\beta_{k,r}
+w_{k,r}^{\mathrm{or}}\widehat\beta_{0,r}^{(-k)}\).

We form each oracle estimate before averaging performance over stores
and replications. No averaging of error terms across replications is
used to construct the weights. The simulation figures label this
realization-specific benchmark Shrinkage (\(w^\star\)).

\subsection{Calibration of the plug-in weight}
\label{app:plugin_calibration}

The plug-in rule uses the weight scale in
Corollary~\ref{cor:theory_weight_scale}, with its variance scale calibrated
from the training data. In replication \(r\), we evaluate the Hessian of
the full pooled empirical loss at the mean of the leave-one-out estimates:
\[
\widetilde\beta_r
:=\frac1K\sum_{\ell=1}^K\widehat\beta_{0,r}^{(-\ell)},
\qquad
\widehat H_{\mathrm{pool},r}
:=\nabla^2L_0(\widetilde\beta_r;Y_r),
\]
where \(Y_r\) contains all stores' training observations. We use
\(\widehat v_r:=K\operatorname{tr}(\widehat H_{\mathrm{pool},r}^{\dagger})\)
as a proxy for local estimation variance, where \(\dagger\) denotes the
Moore--Penrose pseudoinverse. The factor \(K\) adjusts the pooled inverse
information to the local sample size in the balanced simulations,
where \(n_k=n\). The resulting weight is
\begin{equation}
\widehat w_{k,r}^{\mathrm{PI}}
:=\frac{\widehat v_r}{\widehat v_r+\widehat\Delta_{k,r}^2}
=\frac{d}{d+\widehat\lambda_r n\widehat\Delta_{k,r}^2},
\qquad
\widehat\lambda_r:=\frac{d}{n\widehat v_r},
\label{eq:app_plugin_calibration}
\end{equation}
where \(\widehat\Delta_{k,r}
:=\|\widehat\beta_{0,r}^{(-k)}-\widehat\beta_{k,r}\|_2\).
Thus, \(\widehat\lambda_r\) provides a data-dependent calibration of the
order-based rule: it is common across stores within a replication and
recomputed in each replication. No validation data are used for this
calibration; validation tuning is reserved for the separate CV estimator.

\subsection{Main simulation sweeps}

\paragraph{Heterogeneity sweep.}
The main heterogeneity experiment holds the baseline configuration fixed at \(K=10\), \(N=15\), \(d=7\), and \(n_{\mathrm{train}}=50\), and varies the heterogeneity radius over
\[
\Delta \in \{0.00, 0.10, 0.20, 0.35, 0.50, 0.75, 1.00, 1.25, 1.50, 1.75, 2.00, 2.25, 2.50\}.
\]
This is the experiment reported in Figure~\ref{fig:exp1_regret}. The purpose of the sweep is to trace the phase transition between variance reduction and pooling bias: when \(\Delta\) is small, pooling suppresses local estimation noise; when \(\Delta\) is large, the pooled target becomes systematically misaligned with the store-specific truth.

\paragraph{Sample-size sweep.}
The main sample-size experiment holds the same baseline configuration fixed at \(K=10\), \(N=15\), and \(d=7\), fixes heterogeneity at \(\Delta=1\), and varies the local training size over
\[
n_{\mathrm{train}} \in \{30, 50, 75, 100, 150, 200, 300, 500\}.
\]
This is the experiment reported in Figure~\ref{fig:exp2_regret}. The purpose of the sweep is to show how the value of borrowing declines as the local store-specific history becomes richer.

\paragraph{Monte Carlo design.}
Both main sweeps are evaluated over 200 independent Monte Carlo replications. Each point in the simulation figures is the average over these replications and shares represent the standard error.

\subsection{Simulation metrics}

For each synthetic instance and each estimator, we compute downstream decision quality rather than only coefficient error. Specifically, the simulation study reports assortment regret and pricing regret under the true store-level model. The assortment metric compares the expected objective value under the oracle assortment chosen using the true parameter to the value induced by the assortment chosen from the estimated parameter. The pricing metric is defined analogously for the pricing decision. In the figures shown in the paper, these regret measures are normalized and reported as percentages to make comparisons across calibrations easier to read.

\subsection{Retail-data implementation}
\label{app:retail_data_implementation}

This subsection records the implementation details for the retailer
experiment in Section~\ref{sec:empirical}.

The analysis is conducted separately within each of the five product
categories. We use a
rolling-window design with 8 training weeks, 2 validation weeks, and 2
test weeks. The window is then advanced through the panel to generate
the out-of-sample weeks. In each window, Pooled MLE, Store
MLE, and Shrinkage CV are estimated using only data available before the
test period.
As in the main estimator definitions, every likelihood fit in the
training and refitting steps is defined over the common parameter region
\(\Theta\).

All methods are evaluated on the same common-support candidate set. For
each validation or test store-category-week, the candidate set consists
of products that were stocked in that store-category-week and were also
observed in the corresponding training window. Predicted choice
probabilities are renormalized over this candidate set before products
are ranked. This ensures that differences across methods reflect
differences in estimated rankings, rather than differences in product
coverage.

Shrinkage CV uses the validation window to select the shrinkage weight
from a finite grid \(\mathcal W\). For each candidate weight, products in
the validation candidate set are ranked by predicted revenue,
\(p_{ikct}\widehat P_{ikct}\), and the selected weight maximizes
validation revenue capture. The same selected weight is then used to
rank products in the test window. The main empirical results use a
display budget of \(m=20\%\), meaning that the model selects the top
\(\lceil 0.20|\mathcal A_{kct}|\rceil\) products from the relevant
candidate set.

For the volume-bucket analysis, we compute each store-week's local training
volume as the total number of units sold in the corresponding training
window:
\[
V_{kct}^{\mathrm{train}}
=
\sum_{\tau\in\mathcal T_t^{\mathrm{train}}}
\sum_{i\in\mathcal A_{kc\tau}} q_{ikc\tau}.
\]
Within each category, store-weeks are split into bottom,
middle, and top terciles of \(V_{kct}^{\mathrm{train}}\). We refer to
these groups as low-, medium-, and high-volume store-weeks. These buckets
measure recent local sales information; they are not fixed physical
store-size classifications.

For inference, we compute paired revenue-capture lifts at the
store-category-week level. Confidence intervals are obtained using a
clustered bootstrap over stores, preserving all category-week
observations within each resampled store. Win rates report the fraction
of store-weeks in which Shrinkage CV achieves higher revenue
capture than the corresponding benchmark.

\subsection{Additional retail-data robustness checks}
\label{app:retail_robustness}

We report two additional robustness checks for the retailer experiment.
First, we repeat the featured-assortment revenue-capture exercise at
display budgets \(m=10\%\) and \(m=30\%\). The main text uses
\(m=20\%\), so these additional budgets test whether the empirical
ordering is driven by a particular display capacity. A smaller budget
places more emphasis on the top of the product ranking, while a larger
budget evaluates whether the gains persist when the selected set covers
a broader part of the stocked assortment. The results are given in Figure~\ref{fig:app_revenue_capture_10} and Figure~\ref{fig:app_revenue_capture_30}.

Second, we report diagnostics for the validation-selected shrinkage
weights in Figure~\ref{fig:app_weight_diagnostics}. These diagnostics examine how much the estimator borrows from
the network across categories and local-information buckets. The
theoretical analysis predicts that borrowing should be valuable when
local estimation is noisy relative to the complexity of the choice
problem, but that excessive pooling can be harmful when local demand
structure differs across stores. The weight diagnostics therefore
provide a mechanism check: they show whether the validation weight adjusts
the amount of borrowing in a way that is consistent with the
finite-sample bias--variance logic of the paper.

\begin{figure}[t]
\centering
\begin{tabular}{ccc}
\includegraphics[width=0.32\textwidth]{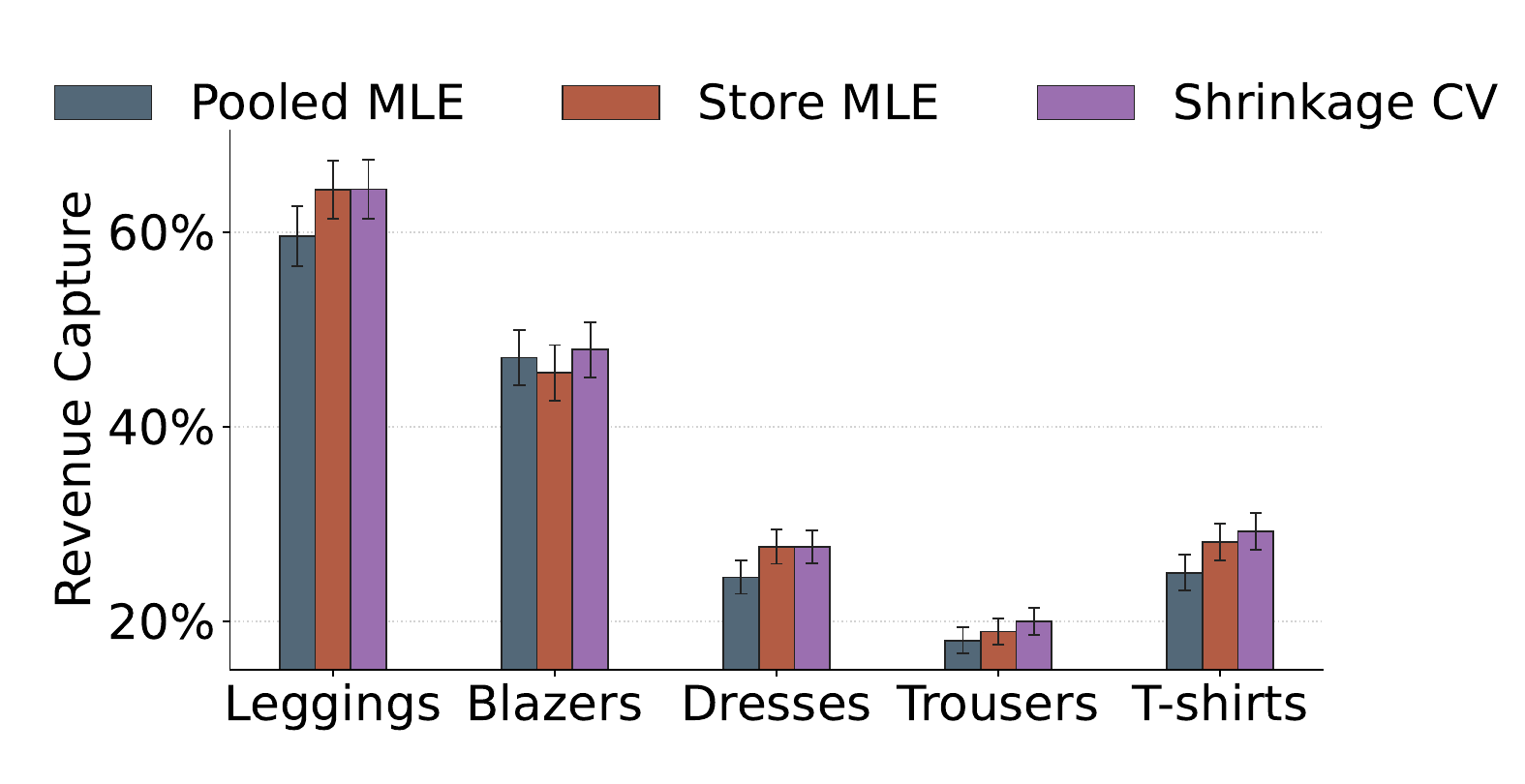} &
\includegraphics[width=0.32\textwidth]{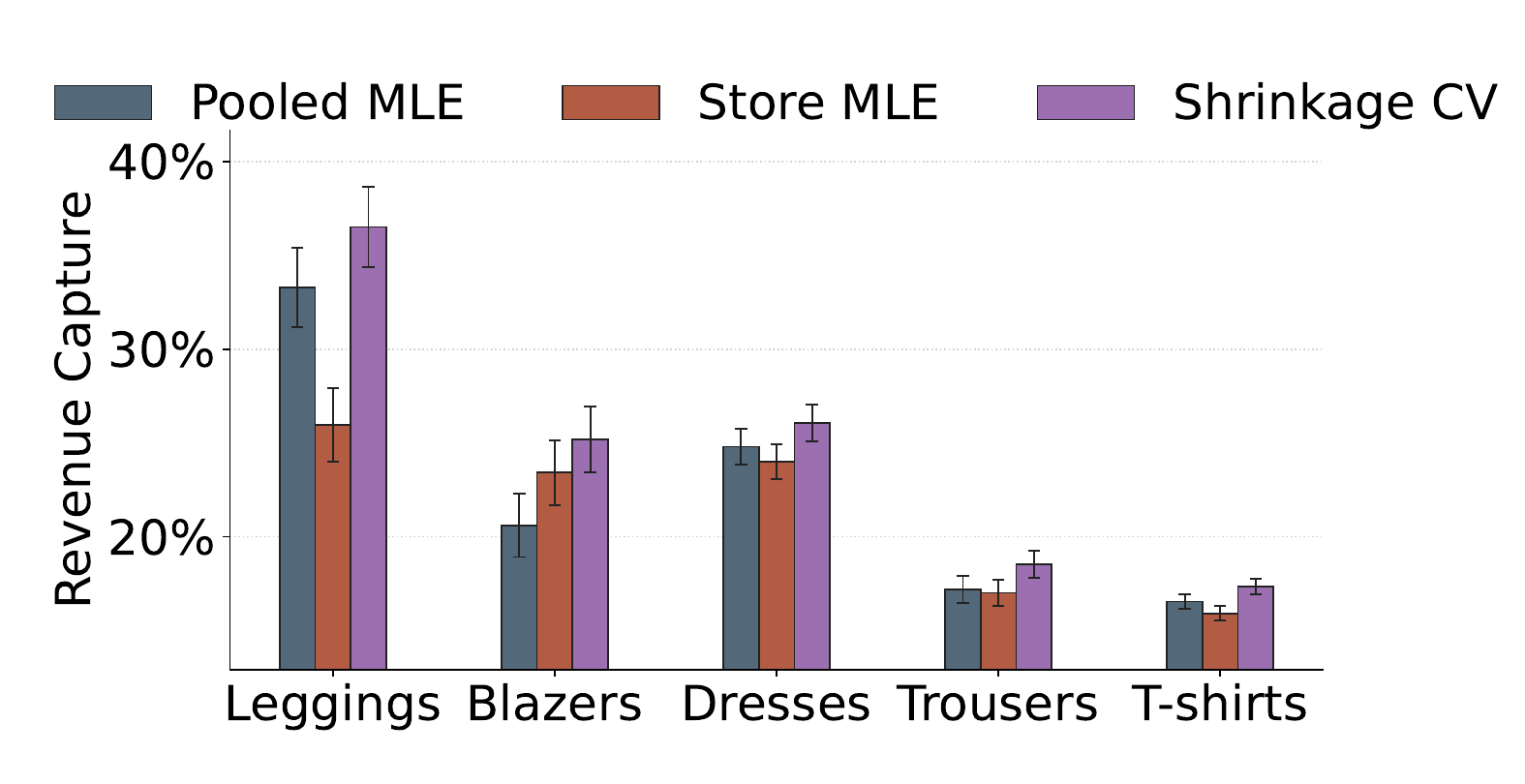} &
\includegraphics[width=0.32\textwidth]{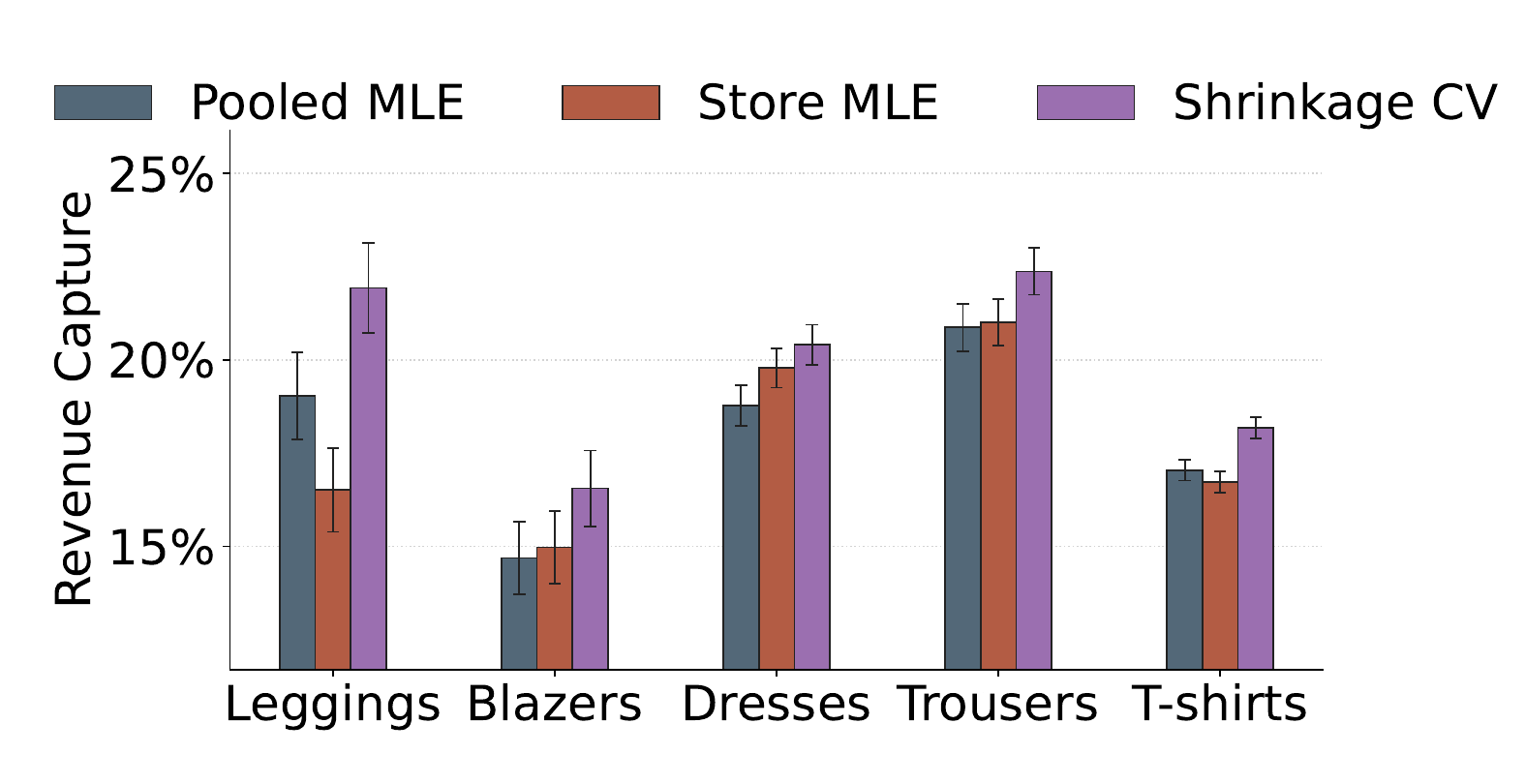}
\end{tabular}
\caption{Revenue capture by product category and local-information
tercile at display budget \(m=10\%\).}
\label{fig:app_revenue_capture_10}
\end{figure}

\begin{figure}[t]
\centering
\begin{tabular}{ccc}
\includegraphics[width=0.32\textwidth]{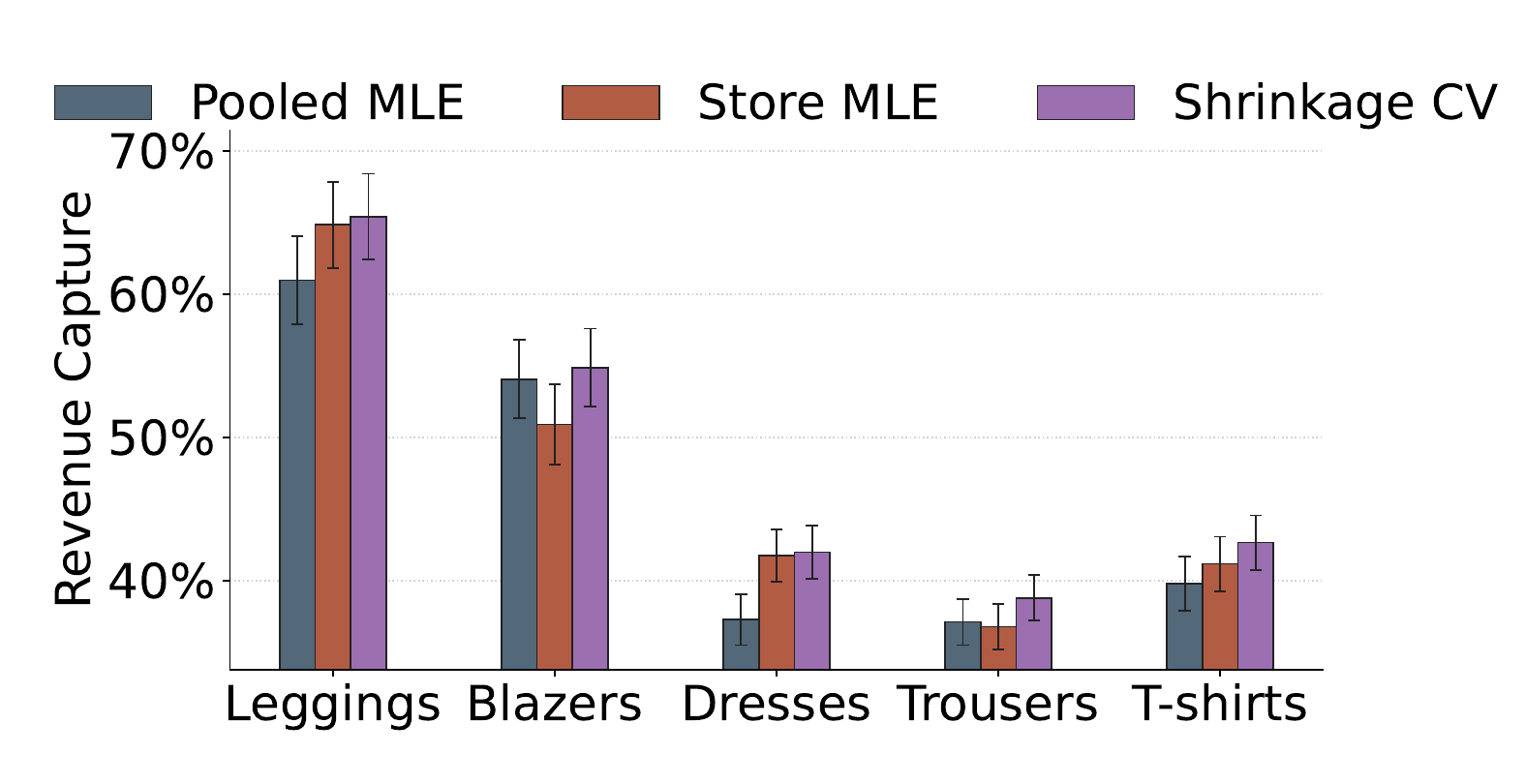} &
\includegraphics[width=0.32\textwidth]{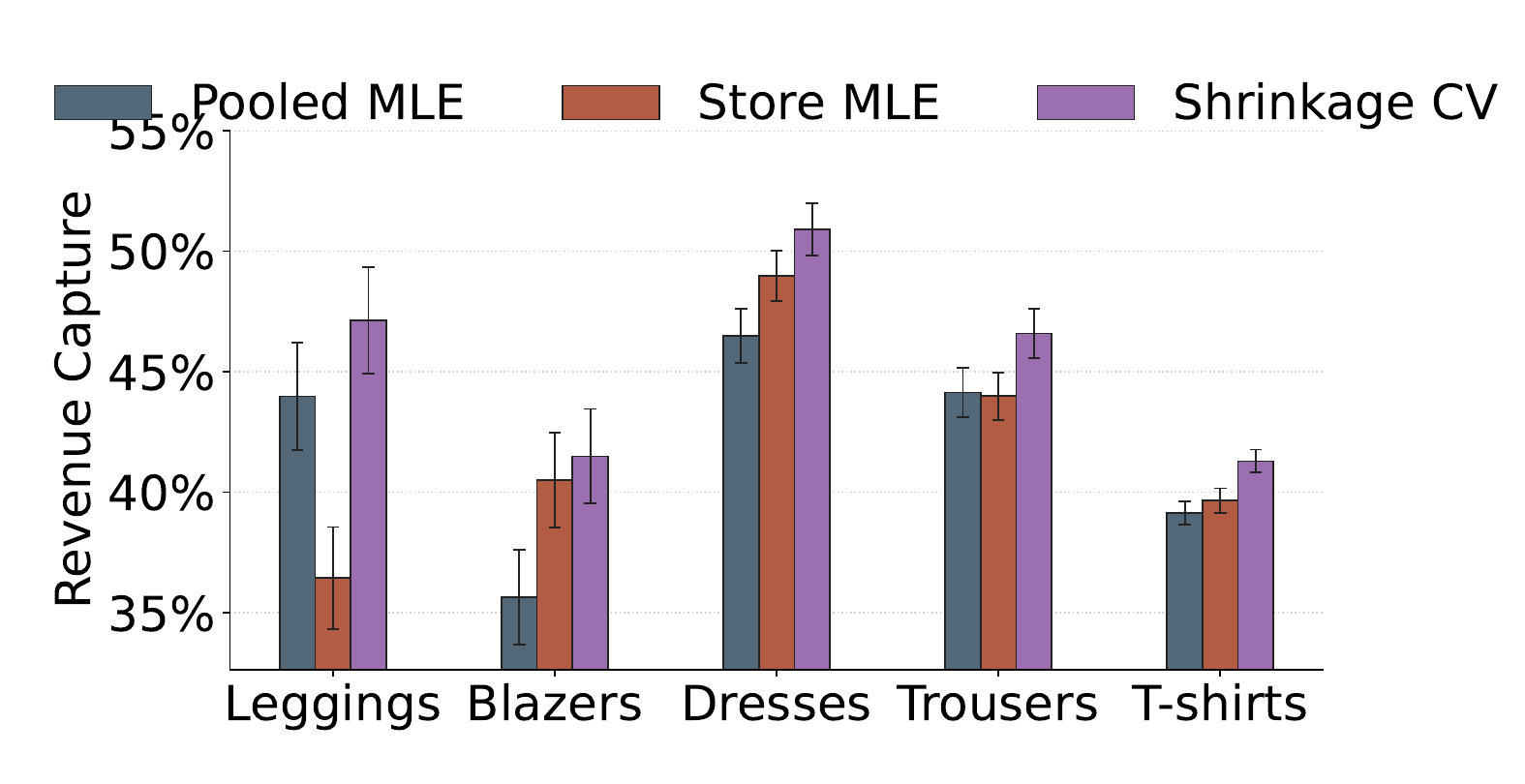} &
\includegraphics[width=0.32\textwidth]{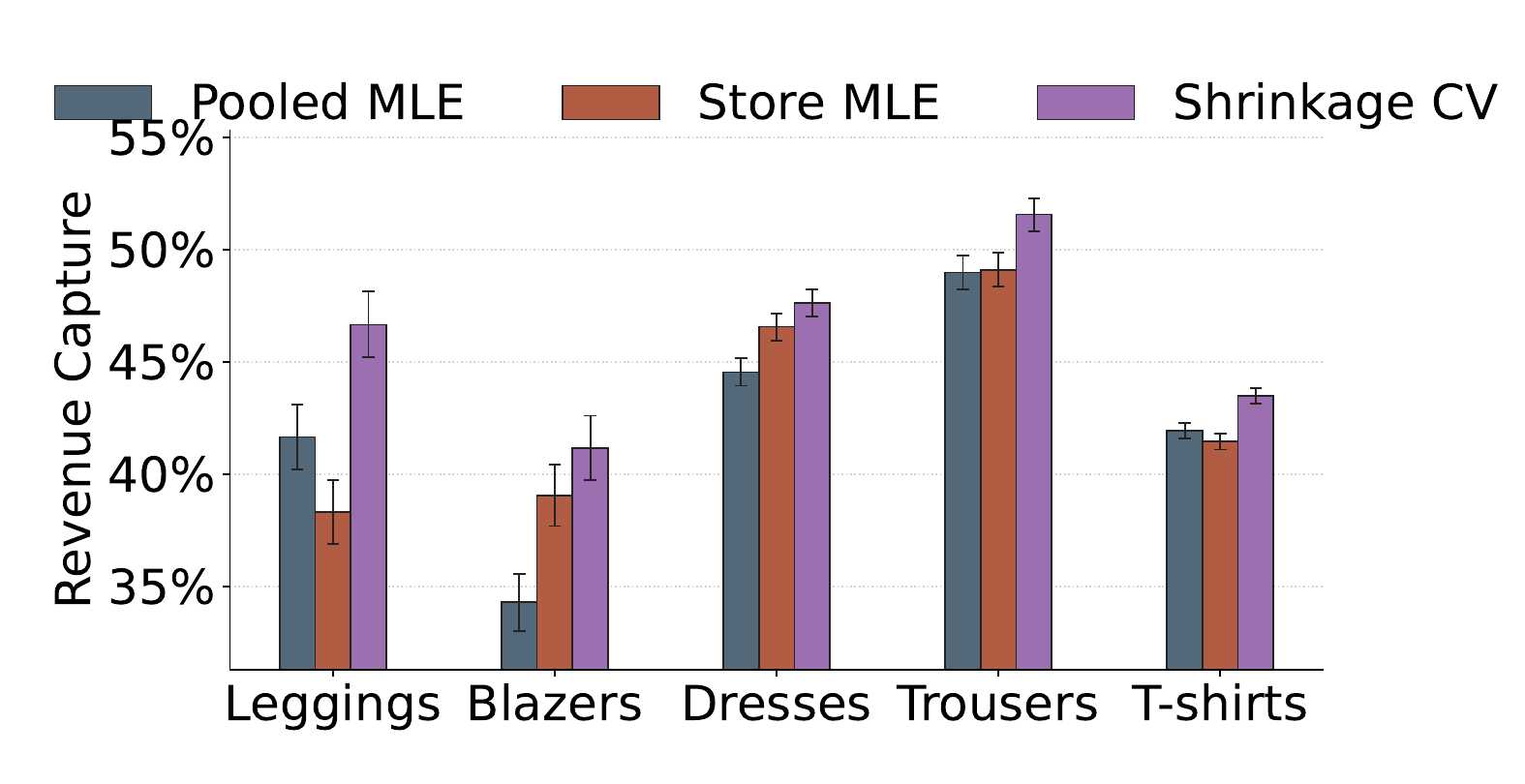}
\end{tabular}
\caption{Revenue capture by product category and local-information
tercile at display budget \(m=30\%\).}
\label{fig:app_revenue_capture_30}
\end{figure}

\begin{figure}[t]
\centering
\includegraphics[width=0.8\textwidth]{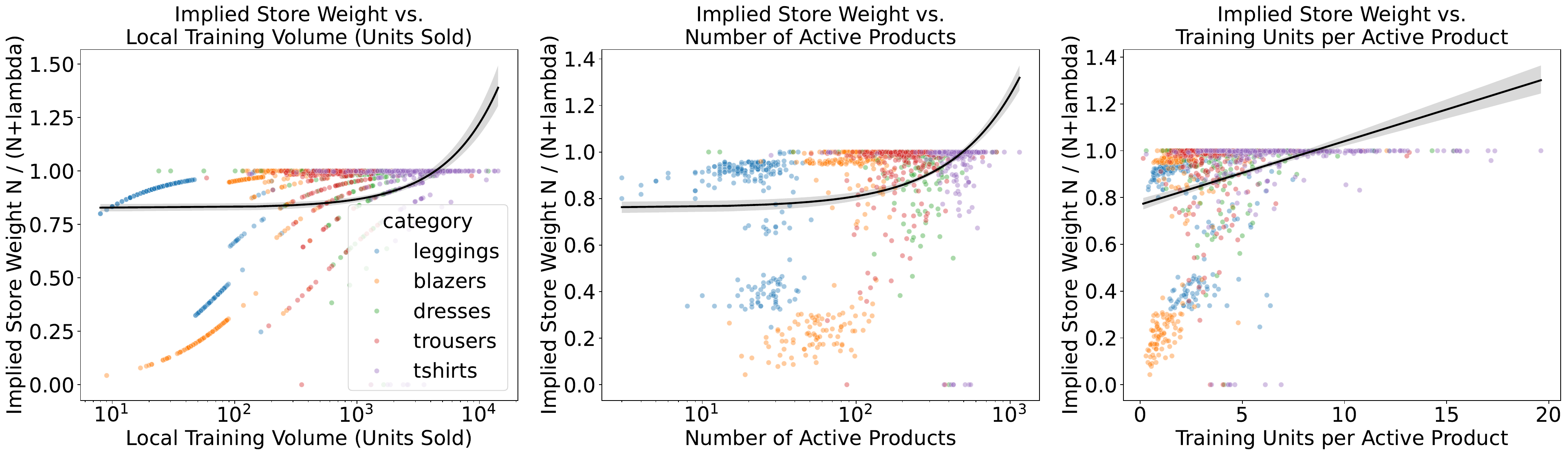}
\caption{Validation-selected shrinkage weights by category and
local-information bucket.}
\label{fig:app_weight_diagnostics}
\end{figure}

\subsection{Sensitivity to Parameter Dimension}
\label{app:dimensionality_sweep}

As an additional robustness check, we vary the parameter dimension \(d\) by changing the number of nonprice features \(d-2\), while holding the local sample size and the overall cross-store heterogeneity radius fixed. To make this sweep interpretable, the synthetic generator rescales the nonprice features by \(1/\sqrt{d-2}\), so that the variance of the linear predictor \(X_i^\top \beta_k^\star\) remains approximately stable as \(d\) changes. This prevents the dimensional sweep from degenerating into a trivial utility-explosion experiment in which probabilities collapse toward \(0\) or \(1\) purely for numerical reasons.

Unlike the heterogeneity and sample-size sweeps in the main text, however, the dimensionality sweep should be interpreted as a robustness exercise rather than as a sharp comparative-static test of the theory. The reason is that increasing \(d\) changes not only the statistical burden of estimation, but also the way a fixed amount of cross-store heterogeneity is distributed across coordinates. In our calibration, the total heterogeneity radius \(\Delta\) is held fixed as \(d\) grows. As a result, the same overall mismatch between stores is spread over a larger number of coefficients, which can reduce the effective per-coordinate bias of the pooled estimator.

Figure~\ref{fig:exp6_regret} reports assortment and pricing regret as \(d\) varies. The assortment results are especially informative. The local estimator deteriorates only mildly with dimension, but the pooled estimator improves steadily as \(d\) rises. This pattern is consistent with the heterogeneity-dilution effect described above: when the same total cross-store mismatch is spread across more coordinates, the pooled estimator can become less biased even as the parameter vector grows. The shrinkage estimators remain close to the lower envelope throughout, with the cross-validated weight performing best across the full range of \(d\).

The pricing results display a different but still interpretable pattern. The pooled estimator exhibits a pronounced U-shape: its regret falls as \(d\) rises from very low levels, but increases again at higher dimensions. The local estimator, by contrast, deteriorates sharply as dimensionality increases. The initial decline in pooled pricing regret is again consistent with the fact that fixed total heterogeneity becomes less concentrated in any given coordinate, reducing pooled bias. The subsequent increase reflects the fact that higher dimensionality still makes the pricing problem statistically harder, so the pooled estimator eventually begins to pay a variance or misspecification cost even after the initial bias reduction. In this setting, the shrinkage estimators remain competitive across the full range of \(d\), although the relative ranking among the practical and infeasible shrinkage benchmarks is more calibration-sensitive than in the main \(\Delta\)- and \(n\)-experiments.

Overall, the dimensionality sweep should be read as a robustness check rather than as a direct confirmation of the main theory. The cleanest theoretical comparative statics remain those in heterogeneity \(\Delta\) and local sample size \(n\). The \(d\)-pass nevertheless provides a useful additional perspective: even after stabilizing the utility scale, higher-dimensional choice estimation can interact with fixed total heterogeneity in nontrivial ways, and the shrinkage estimators continue to provide a robust compromise between noisy local estimation and rigid full pooling.

\begin{figure}[t]
\centering
\begin{tabular}{cc}
\includegraphics[trim=0bp 33bp 0bp 0bp,clip,width=0.47\textwidth]{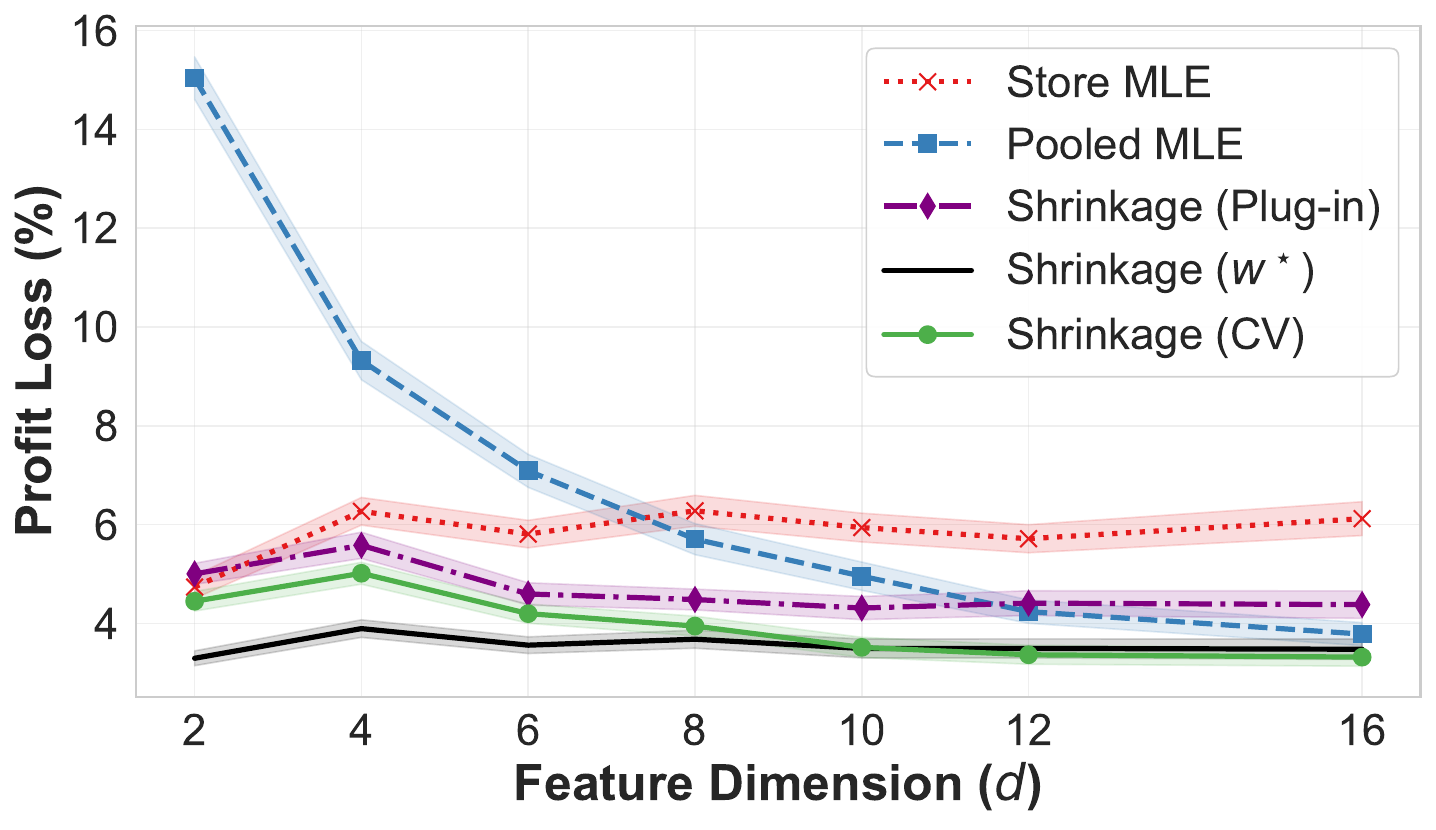} &
\includegraphics[trim=0bp 33bp 0bp 0bp,clip,width=0.47\textwidth]{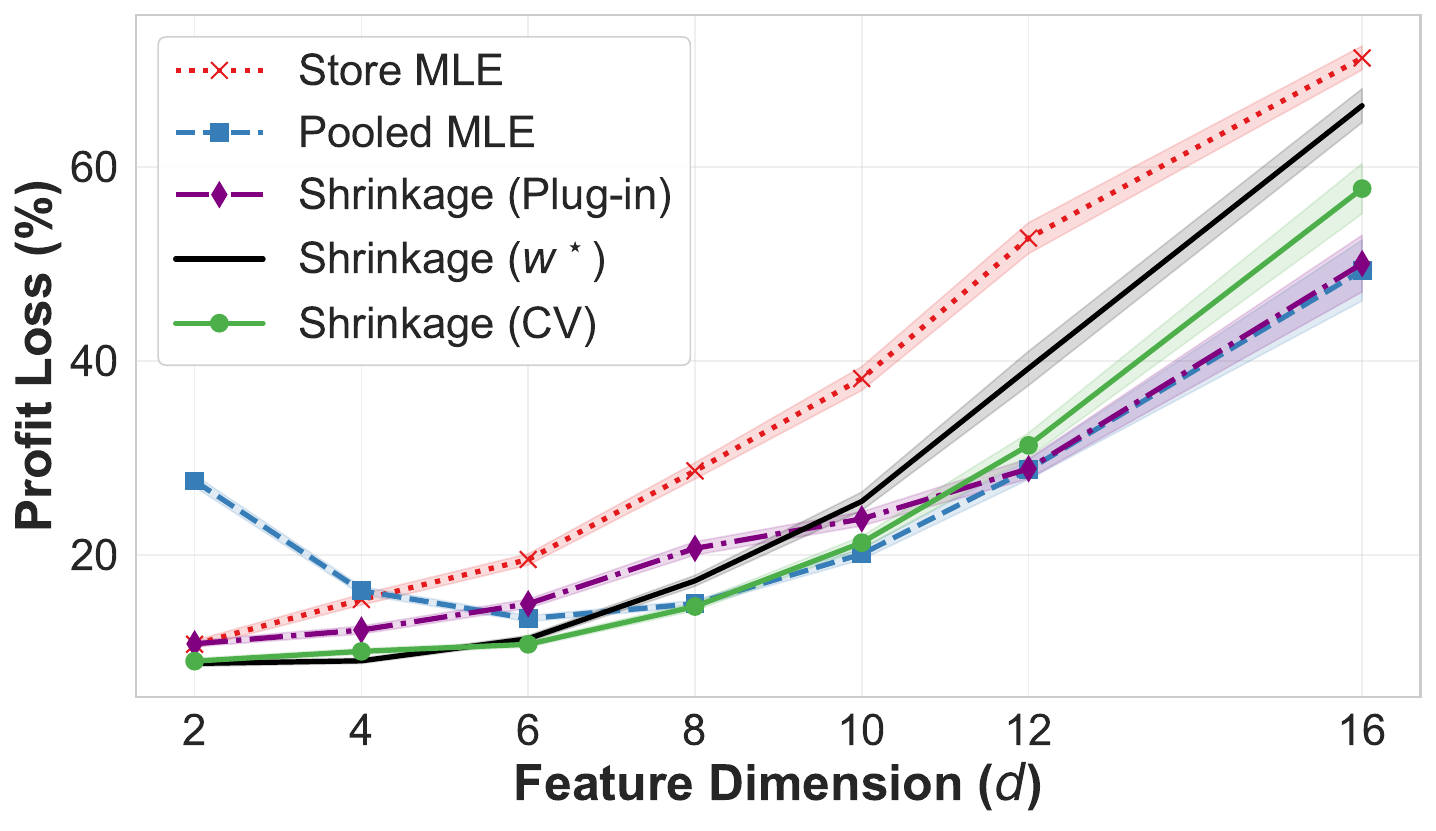}
\end{tabular}
\par
{\small\sffamily\bfseries Nonprice features (\(d-2\))\par}
\caption{Sensitivity of assortment (left) and pricing (right) regret to the number of nonprice features \(d-2\). In this sweep, nonprice features are rescaled by \(1/\sqrt{d-2}\) so that the variance of the linear predictor remains approximately constant as dimension grows. The heterogeneity radius \(\Delta\) is held fixed, so increasing \(d\) spreads a fixed amount of cross-store mismatch over more coordinates.}
\label{fig:exp6_regret}
\end{figure}

\end{document}